\documentclass[twoside]{article}
\usepackage[margin=1in]{geometry}
 
\usepackage[]{amsmath,amssymb,epsfig}
\usepackage{amsthm}
\usepackage{amsmath,mathtools}
\usepackage{amssymb}
\usepackage{graphicx}
\usepackage{epstopdf}
\usepackage{comment}
\usepackage{array}
\usepackage{algorithm}
\usepackage{url}
\usepackage{enumitem}
\usepackage{booktabs}
\usepackage[toc,page,header]{appendix}
\usepackage{etoc}
\usepackage{calligra}
\usepackage{lscape}
\usepackage{algpseudocode}
\usepackage{setspace}
\usepackage{multicol}
\usepackage{multirow}
\usepackage{color}
\usepackage{colortbl}
\usepackage{xcolor}
\usepackage{hyperref}
\usepackage{placeins}
\usepackage[%
    font={small,sf},
    labelfont=bf,
    format=hang,    
    format=plain,
    margin=0pt,
    width=0.8\textwidth,
]{caption}
\usepackage[list=true]{subcaption}
\usepackage[round]{natbib}
\usepackage{caption}
\usepackage{dsfont}
\usepackage{titletoc}
\makeatletter
\let\oldaddcontentsline\addcontentsline 
\makeatother

\numberwithin{equation}{section}

\newtheorem{theorem}{Theorem}

\newtheorem{assumption}{Assumption}

\newtheorem{definition}{Definition}

\newtheorem{lemma}{Lemma}

\newtheorem{remark}{Remark}

\numberwithin{equation}{section}

\DeclareMathOperator{\Tr}{Tr}

\DeclareMathOperator{\supp}{supp}
\DeclareMathOperator{\diag}{diag}

\DeclareMathOperator{\vect}{vec}

\newcommand{\calA}{\ensuremath{\mathcal{A}}}
\newcommand{\calB}{\ensuremath{\mathcal{B}}}
\newcommand{\calC}{\ensuremath{\mathcal{C}}}
\newcommand{\calD}{\ensuremath{\mathcal{D}}}

\newcommand{\calF}{\ensuremath{\mathcal{F}}}

\newcommand{\calI}{\ensuremath{\mathcal{I}}}
\newcommand{\calP}{\ensuremath{\mathcal{P}}}

\newcommand{\calS}{\ensuremath{\mathcal{S}}}
\newcommand{\calM}{\ensuremath{\mathcal{M}}}
\newcommand{\calN}{\ensuremath{\mathcal{N}}}

\newcommand{\calK}{\ensuremath{\mathcal{K}}}
\newcommand{\calX}{\ensuremath{\mathcal{X}}}

\newcommand{\calY}{\ensuremath{\mathcal{Y}}}
\newcommand{\calE}{\ensuremath{\mathcal{E}}}

\newcommand{\norm}[1]{\big\|{#1}\big\|}
\newcommand{\abs}[1]{\big|{#1}\big|}

\newcommand{\set}[1]{\left\{{#1}\right\}}
\newcommand{\dotprod}[2]{\Big\langle#1,#2\Big\rangle}
\newcommand{\est}[1]{\widehat{#1}}
\newcommand{\expec}{\ensuremath{\mathbb{E}}}
\newcommand{\matR}{\ensuremath{\mathbb{R}}}

\newcommand{\argmax}[1]{\underset{#1}{\operatorname{argmax}}}
\newcommand{\argmin}[1]{\underset{#1}{\operatorname{argmin}}}

\newcommand{\prob}{\ensuremath{\mathbb{P}}}

\newcommand{\frobsphere}{\mathbb{S}_n}
\newcommand{\frobball}{\calB_n}

\newcommand{\gammafn}[1]{\ensuremath{\gamma_{#1}}}

\newcommand{\eventp}[1]{\ensuremath{\calE_{#1}}}
\newcommand{\matrixenc}[1]{\ensuremath{\begin{bmatrix} #1 \end{bmatrix}}}
\newcommand{\absenc}[1]{\ensuremath{\left | #1 \right |}}

\newcommand{\uStar}{\ensuremath{u^*}}
\newcommand{\UStar}{\ensuremath{U^*}}
\newcommand{\thtStar}{\ensuremath{\theta^*}}
\newcommand{\tildex}{\ensuremath{\tilde{x}}}
\newcommand{\tilcalS}{\ensuremath{\widetilde{\calS}}}
\newcommand{\tilcalSC}{\ensuremath{\widetilde{\calS}^c}}

\newcommand{\indic}{\ensuremath{\mathds{1}}} 

\newcommand{\vecOne}{\mathbf{1}}

\newcommand{\nrml}{\mathcal{N}}

\newcounter{pcount}
\newenvironment{palign}{%
    \edef\tempEqCount{\the\value{equation}}%
    \setcounter{equation}{\value{pcount}}%
    \align
}{%
    \endalign
    \setcounter{pcount}{\value{equation}}%
    \setcounter{equation}{\tempEqCount}%
    \ignorespacesafterend
}
    
\usepackage[preprint]{aistats2027}
\begin{document}

%

%

\twocolumn[

\aistatstitle{RobustLDS: Learning linear dynamical systems under adversarial corruptions}

\aistatsauthor{Aravinda Kanchana Ruwanpathirana \And Hemant Tyagi}
\vspace{0.25cm}
\aistatsaddress{Division of Mathematical Sciences, SPMS, NTU Singapore 637371 \\ \texttt{\{aravinda.r, hemant.tyagi\}@ntu.edu.sg} } ]

\begin{abstract}
We consider the problem of learning linear dynamical systems under adversarial contamination from a single trajectory of length $T$. While identification of linear dynamical systems itself is well-studied, the problem of robust system identification under adversarial contamination is relatively less explored. In this work, we study the setting where a fraction of the $T$ observations are contaminated by adversarial outliers. We propose different estimators based on relaxations of least-trimmed squares along with an alternating minimization algorithm. Furthermore, we also propose two estimators which exploit the group-sparsity (through penalization/hard-constraints) of the outliers. For the estimator with group-sparse penalty, we derive non-asymptotic error bounds which establish its robustness to outliers. We also show empirically that the proposed estimators work well in practice.
\end{abstract}

%

\section{Introduction} \label{sec:intro}
Regression is a fundamental problem in machine learning where there exists a set of covariates $\calX \subset \matR^{d}$, set of labels $\calY \subset \matR$ and a distribution $\calD$ on $\calX \times \calY$, from which a learner gets $(x_i, y_i)_{i=1}^m \stackrel{i.i.d}{\sim} \calD$ and attempts to learn the underlying model parameters. For instance in the case of linear regression where $y_i = x_i^\top \beta^*+\varepsilon_i$ ( $\varepsilon_i$ is noise), minimax error rates for estimating $\beta^*$ are now well known even in the high-dimensional regime ($m \ll d$) \citep{Ye2010Rate,raskutti2011Minimax,lounici2011Oracle}.

However, when a fraction of the data $(x_i, y_i)_{i=1}^m$ is contaminated by an adversary, estimating $\beta^*$ becomes a more challenging task. This is known as robust regression under adversarial contamination and has been extensively studied in recent years under different models of contamination, e.g., inlier-outlier model \citep{gaiffas2022, merad2023robust}, Huber contamination \citep{du2017computational} and strong contamination \citep{liu2019high, liu2019highM}. A wide range of methods have been explored, e.g., robust gradient-descent \citep{gaiffas2022, merad2023robust, liu2019high, liu2019highM, cherapanamjeri2020} with gradient estimators using Median-of-Means, Trimmed Mean, robust PCA and SDP formulations; least-trimmed squares estimators \citep{bhatia2015robust,Bhatia2017Consistent} through hard-thresholding; penalized Huber M-estimators \citep{pensia2021robust, dalalyan2019outlier, sasai2020robust, sasai2022robust, sasai2024outlier, RobustLin2025}; other approaches such as least-absolute-deviation \citep{karmalkar2019, pensia2021robust} and Dantzig selectors \citep{chen13Robust}. 

While the aforementioned works provide strong theoretical results when the clean data is i.i.d, whether they can be extended to the setting where the clean data exhibits dependencies is relatively less understood.
%
%
%
\paragraph{Learning linear dynamical systems (LDS).} Let $A^* \in \matR^{n \times n}$ be an unknown matrix, and consider the LDS\footnote{Also referred to as VAR (vector autoregressive) model of order $1$ in the literature.} (with $x_0 = 0$)
\begin{align} \label{eq:dyn_process}
x_{t+1} = A^* x_t + \eta_{t+1},  \quad t = 0,\dots,T-1,
\end{align}
where $\eta_t \stackrel{i.i.d.}{\sim} \calN(0,I_n)$ refers to the process noise (the ``excitation'' of the system). Given the trajectory $(x_t)_{t=1}^T$, estimating $A^*$ is by now a well-studied problem, with many applications arising, e.g., in control systems and reinforcement learning. A key challenge in provably estimating $A^*$ is the dependencies between the observations $(x_t)_{t=1}^T$. It was recently shown \citep{Jedra20,Sarkar19,Simchowitz18a,FaraUnstable18} that the ordinary least squares (OLS) estimator recovers $A^*$ with non-asymptotic error bounds (up to log factors) of $O(1/\sqrt{T})$ holding with high probability (w.h.p) when $T$ is suitably large w.r.t $n$. The results hold under different sets of stability assumptions on $A^*$, captured through its spectral radius $\rho(A^*)$, such as: strict stability ($\rho(A^*) < 1$) \citep{Jedra20}, the more challenging marginal stability regime $\rho(A^*) \leq 1$ \citep{Sarkar19, simchowitz19a}, and also for the unstable regime $\rho(A^*) > 1$ \citep{Sarkar19,FaraUnstable18}. Results under different structural assumptions on $A^*$ also exist, including sparsity \citep{Basu15,Kock2015,song2011large,Melnyk16}, low-rank plus sparsity  \citep{lv2021, Basu19}, and general convex constraints \citep{tyagi2024learning}. These works show that $A^*$ can be provably recovered under milder requirements on $T$ w.r.t $n$ as opposed to the unstructured case.

\paragraph{Our setting.} In this work, we focus on a setting in which an $\epsilon$-fraction of the states are adversarially contaminated. Formally, the adversary corrupts data residing on a subset $\calS \subset [T]$ of time-indices, referred to as the set of outliers, with $\abs{\calS} = s = \epsilon T$ (with $0 < \epsilon < 1$). The goal is to recover an estimate $\est{A}$ of $A^* \in \calA$ using the corrupted data $(\tildex_t)_{t=1}^T$, and ensure it is robust to the corruptions. Here, $\calA \subset \matR^{n \times n}$ is a closed, convex set that captures prior structural information about $A^*$. As discussed in Section \ref{sec:rel_work}, existing work in this setting for \eqref{eq:dyn_process} is scarce, and typically not focused on the adversarial setting. 
%
%
%
\subsection{Our contributions} \label{subsec:contributions}
Our contributions are threefold.
\begin{enumerate}
\item In Section \ref{sec:rob_estimators}, we introduce different estimators for jointly estimating $A^*$ and the set of outliers, each based on a particular continuous relaxation of least-trimmed squares \citep{Rousseeuw1984, Rousseeuw2006LTS}. We also provide a simple  alternating minimization based meta-algorithm (cf., Algorithm \ref{alg:robust-estimator}) which provides an efficient procedure for solving these relaxations.

\item In Section \ref{sec:other-app}, we introduce two estimators which exploit the group-sparse nature of the outliers, and amount to minimizing a constrained, weighted least-squares objective with either (i) a $\ell_1/\ell_2$ group-sparse penalty, or (ii) group-sparse hard constraints. For the former, we obtain non-asymptotic error rates in the Frobenius norm of (up to log-factors)
$$O\left(\sqrt{\frac{k}{T}} + \sqrt{\frac{ns}{T}} \right)$$
for joint estimation of $k$-sparse $A^*$ (with $\calA$ chosen to be a $\ell_1$ ball) and the contamination-levels. See Theorem \ref{thm:main-thm} for details.

\item We provide a through empirical evaluation of the proposed estimators on synthetic examples, which show that they work well in practice.
\end{enumerate}
Our theoretical results hold for an oblivious adversary (Assumption \ref{assump1}), analogous to the inlier-outlier model considered in robust regression: the choice of $\calS$ is independent of $(x_t)_{t=1}^T$. Once $\calS$ is chosen, the corruptions on $t \in \calS$ are allowed to be arbitrary -- they can even depend on the data in the inlier set $\calS^c$. Moreover, $A^*$ is required to be strictly stable (Assumption \ref{assump2}). 

Our error rates are likely near-optimal (up to log factors) if one is also interested in  estimating the outliers, but not for estimation rates solely for $A^*$; see discussion after Theorem \ref{thm:main-thm}. To our knowledge, these are the first theoretical results for provably robust estimation of LDSs, under adversarial contamination, from a single trajectory.

%
%
%
%
\subsection{Related work} \label{sec:rel_work}
Existing works for robust estimation of LDSs can be categorized as follows. See Appendix \ref{app-sec:robust-lds-lit-review} for more details.
%
%
%
%
%
\paragraph{Heavy-tailed covariates or process noise.}
There exists a substantial recent line of work that provide robust estimators for learning $A^*$ under \eqref{eq:dyn_process} (or higher-order versions thereof) from a single trajectory, with $(\eta_t)_{t \geq 1}$ assumed to be heavy-tailed \citep{qiu15robust, wong2019lasso, halder2022robust, wang2022rate, wang2024robust, dijk2025tail, Lu2025Robust}. Recently \cite{kanakeri2025outlier} studied this in the setting of multiple i.i.d trajectories of the form \eqref{eq:dyn_process}. These works typically provide non-asymptotic error bounds, with \cite{wang2022rate} providing minimax optimal rates.

\paragraph{Sparse process noise.}
There exists a line of work that considers learning under (slightly general forms of) \eqref{eq:dyn_process} from a single trajectory, where $\eta_t$ is either drawn from a sparse adversarial process \citep{kim2026necessity, kim2026huber} or involves an additive sparse outlier term \citep{kim2025macroeconomic,farahmand2011doubly}. The results of \cite{kim2026necessity, kim2026huber} involve non-asymptotic error rates for parameter-recovery while \cite{kim2025macroeconomic} establish asymptotic normality for their estimator.

%
%
%
\paragraph{Learning from corrupted states.}
\cite{Bhatia0KK16} consider a \emph{univariate} ($n=1$), higher-order version of \eqref{eq:dyn_process}, and assume that the learner observes $\tilde{x}_t = x_t+\theta_t$ where $\theta = \matrixenc{\theta_1 & \theta_2 & \dots & \theta_T}$ is sparse. Assuming the outliers are independent of noise\footnote{this is not stated explicitly but seems to be needed in the proof, e.g., within Appendix B.6} they propose a block-sparse hard thresholding algorithm and provide non-asymptotic rates.
The work of \cite{kanakeri2025outlier} discussed earlier, further shows (for the strong contamination adversarial model) that if an $\epsilon$-fraction of the trajectories are corrupted, then their estimator still recovers $A^*$ up to a small error (w.r.t operator norm). Their theoretical result for this adversarial setting, and even their algorithm, does not apply to the single trajectory setting considered in our work. 
\cite{wong2017lasso} considers a higher-order version \eqref{eq:dyn_process}. Assuming that the learner observes $\tilde{x}_t = x_t+w_t$ where $(w_t)_{t \geq 1}$ are i.i.d centered Gaussians (with known covariance), they establish non-asymptotic recovery guarantees for the penalized lasso. 
Finally, we note that a line of work has considered the setting where $\tilde{x}_t$ is formed by independently  sub-sampling the entries of $x_t$ (with missing entries set to zero), where $(x_t)_{t=1}^T$ is generated from \eqref{eq:dyn_process}. Non-asymptotic error bounds are known for different structral assumptions on $A^*$ (under strict stability), e.g., sparsity and low-rankness \citep{rao2016system, milind2017estimation, castro_missing25, JalaliWillett18, wong2017lasso}. \cite{Rao2017_compmeas} consider the setting where $\tilde{x}_t$ is formed by taking compressive linear measurements of $x_t$ (at each $t$), and study the limits of covariance-based estimation.

%
\subsection{Notation}\label{subsec:notations}
%
%
Let $x \in \matR^n$ be a vector and $B,A \in \matR^{n_1 \times n_2}$ be matrices. For $i,j \in [n_1] \times [n_2]$ let $A_{i,:} \in \matR^{1 \times n_2}$ be the $i$-th row, $A_{:,j} \in \matR^{n_1}$ be the $j$-th column and $A_{ij}$ be the $(i,j)$-th entry of $A$. $\norm{x}_p$ denotes the standard $\ell_p$ norm and $\norm{A}_2$ and $\norm{A}_F$ denotes the spectral norm and the Frobenius norm of $A$ respectively. For integers $p,q \geq 1$, recall the (mixed) entry-wise $L_{p,q}$ norm
\begin{align*}
    \norm{A}_{p,q} := \left(\sum_{j=1}^{n_2} \norm{A_{:,j}}_p^q \right)^{1/q}.
\end{align*}
Furthermore, we denote $\norm{A}_{2,0} := \sum_{j \in [n_2]} \indic \set{\norm{A_{:,j}}_2>0}$ where $\indic\set{\calE}$ denotes the indicator variable for event $\calE$. 
%
Note that $\norm{\cdot}_{2,0}$ is not a norm. The inner product of $B,A$ is defined as $\dotprod{B}{A} = \Tr\left(B^\top A\right)$, and $\vect(A) \in \matR^{n_1n_2}$ denotes the vector formed by stacking the columns of $A$. The all-ones vector in $\matR^n$ is denoted by $\vecOne_{n}$. For matrices in $\matR^{n \times n}$, we denote the unit Frobenius sphere by $\frobsphere$, the unit Frobenius ball by $\frobball$, and the identity matrix by $I_n$. The Kronecker product is denoted by $\otimes$.

As far as possible, calligraphic letters are reserved for sets. The normal distribution with mean $\mu$ and covariance $\Sigma$ is denoted by $\nrml \left(\mu,\Sigma\right)$. For $a,b > 0$, we write $a \lesssim b$ if there exists a constant $C > 0$ such that $a \le Cb$. If $a \lesssim b$ and $a \gtrsim b$, then we write $a \asymp b$. The values of symbols used for denoting constants (e.g., $c, C$ etc.) may change from line-to-line.
\section{Robust estimators based on least-trimmed squares} \label{sec:rob_estimators}
Consider the linear dynamical process given in equation~\eqref{eq:dyn_process} and the contamination process where $s$ observations are contaminated. Recall that $(\tildex_t)_{t=1}^T$ denotes the contaminated observations. Denoting 
\begin{align*}
   z_t^* = \indic\set{t \not\in \calS},
\end{align*}
consider the following $\calA$-constrained least squares problem for estimating $A^*$.
\begin{align*}
    \min_{A \in \calA} \ \sum_{t=1}^{T-1}z^*_t z^*_{t+1} \norm{\tildex_{t+1}-A\tildex_t}_2^2.
\end{align*}
The above objective function only considers summands where both $\tildex_{t+1}$ and $\tildex_t$ are uncontaminated, thus filtering out the effects of contaminated entries. Of course, $(z^*_t)_{t=1}^T$ is not known as $\calS$ is unknown, and what is known is that $\sum_{t=1}^T z^*_t = T-s$. Hence, a natural approach is to jointly estimate $(A^*, (z^*_t)_{t=1}^T)$ via the following combinatorial optimization problem. 
\begin{palign} \label{eq:comb-opt}
& \min_{\substack{A \in \calA \\ (z_t)_{t=1}^T}} \sum_{t=1}^{T-1} z_t  z_{t+1} \norm{\tildex_{t+1}-A\tildex_t}_2^2 \\
& \text{s.t.} \;\sum_{t=1}^T z_t = T-s \text{ and } z_t \in \set{0,1} \; \forall \; t\in [T]. \nonumber
\end{palign}
Note that $(A^*, (z^*_t)_{t=1}^T)$ is feasible for ~\eqref{eq:comb-opt}. 
\begin{remark} Note that Problem~\eqref{eq:comb-opt} is analogous to the \emph{least trimmed squares} estimator~\cite{Rousseeuw1984,Rousseeuw2006LTS} which is known to be robust to strong contamination~\cite{bhatia2015robust,Bhatia2017Consistent}, in i.i.d. regression. However, unlike the standard \emph{least trimmed squares},  problem~\eqref{eq:comb-opt} involves products $z_t  z_{t+1}$ and is thus formulated as a quadratic program w.r.t. $(z_t)_{t=1}^T$. 
\end{remark}
\begin{remark}
We assume $s$ is known a priori, following \cite{diakonikolas2019effi, diakonikolas2023near, diakonikolas2024robust, RobustLin2025, pensia2021robust}. 
\end{remark}

%
%
%

%
%
%
%
We can rewrite problem \eqref{eq:comb-opt} as follows. 
Given a matrix $A \in \matR^{n \times n}$, let $Q(A) \in \matR^{T \times T}$ be defined as follows: for all $1 \le t \le T-1$, $Q(A)_{t,t+1} = Q(A)_{t+1,t} = \frac{1}{2} \norm{\tildex_{t+1}-A\tildex_t}_2^2$ and $Q(A)_{t,t'}=0$ for all other $t,t' \in [T]$. Note that
\begin{align*}
&\sum_{t=1}^{T-1}  z_t  z_{t+1} \norm{\tildex_{t+1}-A\tildex_t}_2^2 \\
&= \sum_{t=1}^{T-1} z_t  z_{t+1} (Q(A)_{t,t+1}+Q(A)_{t+1,t})\\
&= \sum_{t=1}^{T} \sum_{\tilde{t}=1}^{T} z_t Q(A)_{t,\tilde{t}} z_{\tilde{t}} = z^\top Q(A) z,
\end{align*}
where $z = \matrixenc{z_1 & z_2 & \dots & z_T}^\top$. Therefore, problem~\eqref{eq:comb-opt} is equivalent to
\begin{palign}\label{eq:quad-form}
 \min_{ A \in \calA, \ z \in \set{0,1}^{T}} z^\top Q(A) z \ 
 \text{ s.t } \ 
 \vecOne_T^\top z = T-s. 
\end{palign}

\begin{algorithm}[htbp!]
\caption{Meta algorithm for jointly estimating $A^*$ and $(z^*_t)_{t=1}^T$}
\label{alg:robust-estimator}
\begin{algorithmic}[1]
\Require Initial guess $A^{(0)}$, convergence threshold $\epsilon$, maximum iterations $K$.
\Ensure Final estimates $(\widehat{A}, \widehat{z})$.
\State $k \gets 0$
\Repeat
    \State \textbf{Step 1: Update $z$}
    
    \State Given $A^{(k)}$ solve a relaxed formulation (cf., \eqref{eq:sphere},\eqref{eq:convex},\eqref{eq:SDP} and \eqref{eq:biconvex}) to obtain $\tilde{z}$.
    
    \State Round $\tilde{z}$ using \eqref{eq:rounding} to obtain  $z^{(k+1)} \in \set{0,1}^T$. 

    \State \textbf{Step 2: Update $A$} \label{stp:A-calc}
    
    \State Given $z^{(k+1)}$, solve \eqref{eq:quad-form} to find $A^{(k+1)}$.
    
    \State $k \gets k + 1$
\Until{$ \frac{\norm{A^{(k)}-A^{(k-1)}}_F}{\norm{A^{(k-1)}}_F} \le \epsilon$ or $k \geq K$}\\
\Return $\widehat{A} = A^{(k)}$, $\widehat{z} = z^{(k)}$
\end{algorithmic}
\end{algorithm}

\paragraph{Overview of relaxations of \eqref{eq:quad-form}} Note that    problem~\eqref{eq:quad-form} is a combinatorial problem w.r.t the binary variables $z_t$'s. In order to circumvent this, we next introduce several estimators all of which are essentially obtained by considering different continuous relaxations of $\set{0,1}^T$.   
\begin{enumerate}
\item \textbf{Sphere relaxation.} Using the fact that $\vecOne_T^\top z  = \norm{z}_2^2$ when $z \in \set{0,1}^T$, we can relax \eqref{eq:quad-form} as follows. See Appendix~\ref{sec:sphere} for details.
\begin{palign}\label{eq:sphere}
& \min_{A \in \calA,z \in \matR^T} z^\top Q(A) z \;  \\
&\text{s.t.} \; \norm{z}_2^2 = T-s, \  \vecOne_{T}^\top z = T-s. \nonumber
\end{palign}
Note that when $A$ is fixed, \eqref{eq:sphere} is a variant of the Trust Region Subproblem. Since $Q(A)$ contains only $O(T)$ non-zero values, we can use the Mor\'{e}-Sorensen algorithm \citep{sorenson1983} which takes $O(n^2T+TK)$ time -- here $K$ is the number of iterations needed by the algorithm. 

\item \textbf{Semi-definite program (SDP) relaxation.} 
Since $z^\top Q(A) z = \dotprod{Q(A)}{zz^\top}$ we can replace $zz^\top$ with the p.s.d matrix $Y$, and use the dual of Shor's relaxation \citep{Boyd1997} to derive the following relaxation of \eqref{eq:quad-form}. See Appendix~\ref{sec:SDP} for details.
\begin{palign}\label{eq:SDP}
& \min_{A \in \calA,Y \succeq 0, z \in [0,1]^T} \langle Q(A), Y \rangle  \\
& \text{s.t.}\; {\diag(Y) = z}, \vecOne_T^\top z = T-s \text{ and } \matrixenc{1 & z^\top \\ z & Y} \succeq 0.\nonumber
\end{palign}
Note that for a fixed $A$, \eqref{eq:SDP} is a SDP and can be solved, e.g. using interior point methods, in time polynomial in $T$ to obtain $Y, z$.  For efficiency, we implement an equivalent Burer-Monteiro method \citep{Burer2003} running in $O(n^2T+TK)$ time. While the iteration count $K$ lacks theoretical bounds, it remains small in practice.  

\item \textbf{Bi-convex relaxation.} Since the filtering in \eqref{eq:comb-opt} is on each $\norm{\tildex_{t+1}-A\tildex_t}_2^2$, instead of using $z_t z_{t+1}$, we could introduce a $b_t$ such that $b_t = 1$ if and only if $z_t  z_{t+1} = 1$. It can be shown that this can be represented using a set of linear constraints involving $b_t, z_t$ and $z_{t+1}$. Using this, and relaxing the binary constraints on $(b_t)_{t=1}^{T-1}, (z_t)_{t=1}^T$ we can derive the following bi-convex optimization problem (see Appendix~\ref{sec:biconvex} for details).
\begin{palign}\label{eq:biconvex}
& \min_{\substack{A \in \calA \\ z \in [0,1]^T \\ 
b \in [0,1]^{T-1}}} \sum_{t=1}^{T-1} b_t \norm{\tildex_{t+1}-A\tildex_t}_2^2\\
& \text{s.t.}\; {\sum_{t=1}^T z_t = T-s}  \nonumber\\
& z_t+z_{t+1}-1 \le b_t \le \min \set{z_t,z_{t+1}} \; \forall t\in [T-1]. \nonumber
\end{palign}
Note that the objective is bi-convex in $A, b$ while the constraints are convex w.r.t $A,b$ and $z$. In particular, when $A$ is fixed, \eqref{eq:biconvex} is a linear program w.r.t $b,z$, and hence can be solved efficiently in practice, e.g. using interior point methods, in time polynomial in $n,T$.

\item \textbf{Convexifying $Q(A)$ to estimate $z$.}  For any fixed $A$, define $\tilde{Q}(A) := Q(A)+\alpha(A) I_T$ where $\alpha(A)$ is chosen so that $\tilde{Q}(A)$ is p.s.d. It can be easily seen that in the setting where $z$ is binary, this does not change the objective. Now we can replace $Q(A)$ with $\tilde{Q}(A)$ in \eqref{eq:quad-form}, and relax the constraints on $z$ to $[0,1]^T$, to obtain (see Appendix~\ref{sec:convexQA} for details)
\begin{palign}\label{eq:convex}
& \min_{z \in [0,1]^{T}} z^\top \tilde{Q}(A) z \quad \text{s.t.} \quad {\vecOne_T^\top z = T-s}.  
\end{palign}
Note that \eqref{eq:convex} is a convex problem. Moreover, unlike the other relaxations described above, \eqref{eq:convex} requires an estimate $A$ of $A^*$ to be available. Similar to the case of \eqref{eq:biconvex}, when $A$ is fixed, \eqref{eq:convex}  can be solved in time polynomial in $n,T$.

\end{enumerate}
Note that while the formulations \eqref{eq:sphere}--\eqref{eq:biconvex} are non-convex, they share a common property: if we fix $A \in \calA$ then the resulting problem is tractable in the remaining variables, and we can solve for $z \in \set{0,1}$. On the other hand, if we fix $z \in \set{0,1}$, then we can  solve the ensuing $\calA$-constrained least-squares problem in \eqref{eq:quad-form} to estimate $A^*$. This motivates the alternating minimization based meta algorithm described in Algorithm~\ref{alg:robust-estimator} to find  estimates $(\est{A}, (\est{z}_t)_{t=1}^T)$. 

\begin{remark}
The convergence of Algorithm \ref{alg:robust-estimator} will depend on the initialization $A^{(0)}$. Since the relaxations considered are all non-convex, while a choice of $A^{(0)}$  within a small neighborhood of the ground-truth might lead to convergence, its unclear whether this holds for any random initialization. In practice, a potential choice of $A^{(0)}$ is to solve the $\calA$-constrained least squares problem on the contaminated data -- our experiments follow this strategy. 
\end{remark}
\begin{remark}\label{rem:solve-A}
Suppose we fix $z \in \set{0,1}^T$. Then \eqref{eq:sphere} is essentially just $\calA$-constrained weighted least squares which is the same as solving \eqref{eq:quad-form} (with fixed $z$). For \eqref{eq:biconvex}, we can replace $b_t$ with $z_t  z_{t+1}$ (which is what the envelope constraints would be doing) and therefore, it reduces to solving \eqref{eq:quad-form} with fixed $z$. For \eqref{eq:SDP} if we set $Y = zz^\top$, then $\dotprod{Q(A)}{Y} = z^\top Q(A) z$ which for fixed $z$ is the same as \eqref{eq:quad-form}. Solving \eqref{eq:quad-form} over $\calA$ takes $O(n^2T+n^3)$ time for $\calA = \matR^{n \times n}$, and $O(n^2T+n^3K)$ for $\calA = \set{A \in \matR^{n \times n}: \norm{A}_{1,1}\le \norm{A^*}_{1,1}}$ (with $K$ iterations of the FISTA algorithm \citep{beck2009fast}).  
\end{remark}

\paragraph{Rounding procedure.} For rounding  $\tilde{z} \in \matR^T$ in Algorithm~\ref{alg:robust-estimator}, we use the following procedure.
\begin{equation}\label{eq:rounding}
    \bar{z}_t =
    \begin{cases}
        1, & t \in \est{\mathcal{S}}, \\
        0, & \text{otherwise},
    \end{cases}
    \qquad
    \est{\mathcal{S}} \in \argmax{S' \subseteq [T],\ |S'| = T-s} \; \sum_{t \in S'} |\tilde{z}_t|,
\end{equation}
where $\bar{z} \in \set{0,1}^T$ is the rounded solution. Note that this can be computed in $O(T\log T)$ time per iteration.
%
\section{Robust learning using group-sparse lasso}\label{sec:other-app}
Consider the linear dynamical system defined in~\eqref{eq:dyn_process} and recall that $\calS \subset [T]$ is the set of outliers. For each $t \in \calS$, the adversary corrupts $x_t$ to obtain $\tildex_t$. Let $\thtStar_t = \tildex_t-x_t$ and note that $\thtStar_t = 0$ for all $t \in \calS^c$. Furthermore,
\begin{align*}
\tildex_{t+1} &= x_{t+1}+ \thtStar_{t+1}\\
&=A^* x_t+\eta_{t+1}+\thtStar_{t+1}\\
&=A^*(\tildex_t-\thtStar_t)+\eta_{t+1}+\thtStar_{t+1}\\
&=A^*\tildex_t+\eta_{t+1}+\left(\thtStar_{t+1}-A^*\thtStar_{t}\right).
\end{align*}
Now denoting $\uStar_t := \left(\thtStar_{t+1}-A^*\thtStar_{t}\right)$, we obtain $\tildex_{t+1}= A^*\tildex_t+\eta_{t+1}+\uStar_t$. Let us denote $\UStar = \matrixenc{\uStar_1 & \uStar_2 & \dots & \uStar_{T-1}}$, $\uStar := \vect\left(\UStar\right)$ and 
$$\calS(\uStar) = \{t \in [T-1]: \norm{\uStar_t}_2 \neq 0\}.$$ 
Then it is easy to see that $\absenc{\calS(\uStar)} \le 2\absenc{\calS} = 2s$ (recall that $\abs{\calS} = s$). 

Consider the matrices
\begin{align*}
& X = \matrixenc{x_1 & x_2 & \dots & x_{T-1}}, \; \bar{X} = \matrixenc{x_2 & x_3 & \dots & x_{T}},\\
& \Theta^* = \matrixenc{\thtStar_1 & \thtStar_2 & \dots & \thtStar_{T-1}}, \bar{\Theta}^* = \matrixenc{\thtStar_2 & \thtStar_3 & \dots & \thtStar_{T}}.
\end{align*}
where $(x_t)_{t \in [T]}$ are the uncontaminated data. For the contaminated data, we observe that 
\begin{align*}
Y &:= \matrixenc{\tildex_1 & \tildex_2 & \dots & \tildex_{T-1}} = X+\Theta^* \text{ and},\\
\bar{Y} &:= \matrixenc{\tildex_2 & \tildex_3 & \dots & \tildex_{T}} = \bar{X}+\bar{\Theta}^*.
\end{align*}
Furthermore, we also have $\UStar = \bar{\Theta}^*-A^*\Theta^*$ where as seen above, $\norm{\UStar}_{2,0} \leq 2s$. 

Note that since we do not have any assumptions on the magnitude of the outliers, they can be potentially unbounded. To handle this, we use a filtering procedure that removes a pair $(\tildex_{t}, \tildex_{t+1})$ if either $\norm{\tildex_{t}}_{\infty}$ or  $\norm{\tildex_{t+1}}_{\infty}$ is large. Formally, we introduce a diagonal matrix $D \in \matR^{(T-1) \times (T-1)}$ with $D_t$ denoting the $t$-th diagonal entry of $D$. $D_{t}$ represents the weight on the term $\norm{\tildex_{t+1}-A\tildex_t}_2^2$. In this work, we set $$D_{t} = \indic\set{\norm{\tildex_{t+1}}_{\infty} \le r \text{ and } \norm{\tildex_{t}}_{\infty} \le r}$$ for some parameter $r$ (which will be specified in the analysis). 
\begin{remark}
In this work, we consider one of the simplest forms of filtering (with $D$). In the i.i.d. robust regression setting, works such as \cite{sasai2024outlier, pensia2021robust} use more sophisticated approaches, including iterative filtering and weight learning. Extending these to the VAR setting is not straightforward as the dependence between data would break the concentration arguments underlying their analyses.
\end{remark}

Given this, we can now introduce our two estimators for jointly estimating $A^*$ and $U^*$, both based on exploiting the block-sparse\footnote{We will use ``block'' and ``sparse'' interchangeably.} nature of $U^*$.
\begin{enumerate}
\item \textbf{Least squares with $\ell_{1}/\ell_2$ block-sparse penalization.} Denoting $U = \matrixenc{u_1 & u_2 & \dots & u_{T-1}} \in \matR^{n \times (T-1)}$, the first estimator solves the block-sparse penalized least squares problem 
\begin{palign}
\left(\widehat{A},\widehat{U}\right) \in \Bigg\{\argmin{\substack{A \in \calA \\ U \in \matR^{n \times (T-1)}}} &\frac{1}{2} \norm{(\bar{Y}-AY)D-U}_F^2 \nonumber \\
&+\lambda \norm{U}_{2,1} \Bigg\}. \label{eq:l12-ls-pen}
\end{palign}
Since $\calA$ is convex, then \eqref{eq:l12-ls-pen} is a convex program which can be solved efficiently in practice. 

\item \textbf{Least squares with block-sparse hard constraints.} Since $\norm{\UStar}_{2,0} \leq 2s$, we can replace the penalty in \eqref{eq:l12-ls-pen} with a hard constraint to obtain,
\begin{palign}\label{eq:spr-ls-hard}
\left(\widehat{A},\widehat{U}\right) \in \argmin{\substack{A \in \calA \\ U : \norm{U}_{2,0} \le 2s}} \frac{1}{2} \norm{(\bar{Y}-AY)D-U}_F^2.
\end{palign}
\end{enumerate}
As will be seen shortly, our main theoretical result (Theorem \ref{thm:main-thm}) is for \eqref{eq:l12-ls-pen} for sparse $A^*$ and with
\begin{align} \label{eq:Astar_l1ball}
\calA = \set{A \in \matR^{n \times n}:\norm{A}_{1,1} \le \norm{A^*}_{1,1}}.
\end{align}
\begin{remark}
 The estimator \eqref{eq:l12-ls-pen} is analogous to the penalized lasso that has been well-studied for robust sparse regression with i.i.d inliers (e.g., \cite{dalalyan2019outlier,sasai2020robust,chen13Robust}). 
\end{remark}
\begin{remark}\label{rem:impl-LS}
In the experiments, where we focus on $\calA$ as in \eqref{eq:Astar_l1ball}, we solve \eqref{eq:l12-ls-pen} via a combination of block coordinate descent and constrained weighted least squares via alternating minimization in $O((n^2T+n^3K)\widetilde{K})$ time for $\widetilde{K}$ iterations (where $\widetilde{K}$ is small in practice).   \eqref{eq:spr-ls-hard} is solved via alternating minimization. Indeed, when $A$ is fixed then $U$ can be solved for exactly by picking the columns of $(\bar{Y}-AY)D$ with the largest $\ell_2$ norms, which can be computed in $O(nT \log T)$ time by computing the norms and sorting. When $U$ is fixed, we revert to $\calA$-constrained weighted least squares which can be solved efficiently (as described above) in $O(n^2T+n^3K)$ time (see Remark~\ref{rem:solve-A} for specifics on $K$). %
\end{remark}
Theorem \ref{thm:main-thm} relies on the following assumptions.
\begin{assumption}[Adversary] \label{assump1}
 The adversary first partitions $[T] = \calS \cup \calS^c$ where the choice of $\calS$ is oblivious to the data $(x_t)_{t=1}^T$. The data corresponding to time-indices in $\calS$ are then corrupted in an arbitrary manner, no assumption is made in this regard.  
\end{assumption}

\begin{assumption}[Stability of $A^*$] \label{assump2} 
The quantity $J(A^*) \coloneq \sum_{i=0}^\infty \norm{(A^*)^i}_2$ is finite.
\end{assumption}

The adversarial model in Assumption \ref{assump1} is similar to the inlier-outlier model that is well-studied in the robust-regression literature (with the clean observations being i.i.d) (see e.g., \cite{merad2023robust,gaiffas2022}).
The term $J(A^*)$ in Assumption \eqref{assump2} was considered in \cite{Jedra20} for analyzing the OLS estimator. If $\rho(A^*) < 1$ then it can be shown that $J(A^*)$ is finite, although it can grown with $n$. If $\norm{A^*}_2 < 1$, a direct calculation shows $J(A^*) \le \left(1-\norm{A^*}_2\right)^{-1}$. Furthermore, if there exist constants $C>0$ and $\beta \in (0,1)$ such that $\norm{(A^*)^i}_2 \le C\beta^i$ holds for all integers $i \geq 1$, then $J(A^*) \le 1+\frac{C\beta}{1-\beta}$. This latter condition is referred to as $(C, \beta)$-stability in the literature (cf., \cite[Def. 6.2]{ziemannMix22}).
We then obtain the following theorem.
\begin{theorem}\label{thm:main-thm}
Under Assumptions \ref{assump1} and \ref{assump2}, let $A^*$ be $k$-sparse and $\calA$ be as in \eqref{eq:Astar_l1ball}. There exist constants $c_1, c_2, c_3, c' > 1$ such that the following is true. Let $\delta \in (0,1)$,
\begin{align*}
&L=J^4(A^*) k^2 \cdot \log \left(\frac{n^2}{k}\right) \cdot \max \set{s,\log \left(\frac{n}{\delta}\right)} \text{ and,}\\
&\lambda = 2\left(\sqrt{n}+n^{1/4} \log^{1/4} \left(\frac{T}{\delta}\right)+2^{1/2} \log^{1/2} \left(\frac{T}{\delta}\right)\right).
\end{align*} 
It holds that if  
\begin{align*}
T \ge c' \cdot L \log^2 \left( \frac{n \cdot L}{\delta}\right), \ \  r = c_3 J(A^*) \sqrt{\log \left(\frac{Tn}{\delta}\right)},
\end{align*}
then with probability at least $1-\delta$, any solution $(\widehat{A}, \widehat{U})$ of \eqref{eq:l12-ls-pen} satisfies
\begin{align*}
&\norm{\widehat{A}-A^*}_F^2+\frac{1}{T}\norm{(\widehat{U}-U^*) D}_{F}^2 \\ 
&\le c_1\frac{J^2(A^*) k\log \left(\frac{Tn}{\delta}\right)\log \left(\frac{n}{\delta}\right)}{T}\\
&+c_2 \max \set{n,\log^2\left(\frac{T}{\delta}\right)}\frac{s}{T}.
\end{align*}
\end{theorem}
The theorem provides an error bound of $$O\left(\frac{k\log n \log(nT/\delta)}{T} + \frac{s\max\set{n,\log^2(T/\delta)}}{T} \right)$$ which holds when $T$ is suitably large. Since $A^*$ and $U^*$ are collectively $(k + 2ns)$-sparse, this rate is likely near optimal (up to $\log$ factors) if one is interested in estimating \emph{both} $A^*$ and $U^*$ (this was also noted in \cite[Section 2]{dalalyan2019outlier} in the i.i.d setting). If only estimating $A^*$ is of interest then the rate in Theorem \ref{thm:main-thm} is sub-optimal - the $s/T$ factor should be $s^2/T^2$ (analogous to the i.i.d setting, cf., \cite{dalalyan2019outlier, ChenGaoRen2016, Gao2017Bernoulli}). A possible approach for proving this would be to adapt the  techniques in \cite{dalalyan2019outlier} (based on incoherence and transfer principles) to our setting of dependent inliers. This is quite non-trivial and left as an important direction for future work.
\begin{remark}
  The stated error bound holds for $T$ suitably large. When $s \leq \log(n/\delta)$ then $T \gtrsim k^2 \log^4(n/\delta)$ suffices. If $s \geq \log(n/\delta)$ then $T \gtrsim k^2 s \log n \log^2(ns)$ suffices. This quadratic dependence on $k$ is an artefact of the proof-technique, in particular, with the argument used for showing the restricted lower isometry for the LHS term $\norm{\Delta {Y'} D}_F^2$ in \eqref{eq:main-inequ}. 
\end{remark}

While the proof of Theorem~\ref{thm:main-thm} is available in Appendix~\ref{app-sec:analysis}, a sketch is provided below for convenience.
\begin{enumerate}
\item In Lemma~\ref{lem:main-inequ}, using the optimality of $\widehat{A},\widehat{U}$ and $A^*, U^*$ we derive the inequality \eqref{eq:main-inequ}, 
\begin{align*}
& \underbrace{\frac{1}{2} \norm{\Delta {Y'} D}_F^2}_{(I)} \le \underbrace{\dotprod{\Delta_A}{\overline{E}DY^\top}}_{(II)}\\
&+{\frac{\lambda}{2} \Bigg(3\norm{(\Delta_UD)_{\tilcalS}}_{2,1}-\norm{(\Delta_UD)_{\tilcalSC}}_{2,1}\Bigg)}
\end{align*}
which holds when $\lambda \ge 2  \max_{t \in \set{2,\dots,T}} \norm{\eta_t}_2$. Here $\Delta_A, \Delta_U, \Delta$ denote the error terms, while ${Y'}$ depends on $Y$ (see \eqref{eq:proof_useful_notation}). The set $\tilcalS$ is the set of time-indices $t \in [T-1]$ for which at least one of $t, t+1$ lies in $\calS$; see \eqref{eq:set_Stil_def}. Thus $\tilcalS^c$ is contained in the set of inliers $\calS^c$.

\item The choice of $\lambda, r$ in Theorem \ref{thm:main-thm} is established in Lemma's ~\ref{lem:bound-lambda} and \ref{lem:all-unc-select} respectively, via standard concentration results. In particular, $r$ is obtained by ensuring that w.h.p, $\norm{x_t}_{\infty} \leq r$ for all $t \in [T]$.

\item In Lemmas~\ref{lem:inf-norm-cost}, and \ref{lem:inf-norm-comple} we establish an upper bound for $(II)$. Lemma~\ref{lem:inf-norm-cost} involves applying Theorem~\ref{thm:bounded-func-conc} to control $\sup_{b \in \calB(r,\abs{\tilcalS})}  \abs{g^\top b}$ for $g {\sim} \nrml \left(0,I_{T-1}\right)$. Lemma~\ref{lem:inf-norm-comple} bounds the tail of an appropriately defined self-normalized martingale using Theorem~\ref{thm:self-nrml-mart}.

\item To bound $(I)$ we use Lemmas~\ref{lem:lwr-bnd-deltaA} and \ref{lem:bnd-cross-trm}. The proofs of both lemmas involve defining appropriate sets of matrices and applying  Theorem~\ref{thm:psd-sup} \cite[Theorem 3.1]{krahmer14}. 

\item Combining the bounds for $(I)$ and $(II)$, we obtain after some simplifications the statement of the theorem.
\end{enumerate}

\section{Experiments} \label{sec:exps}
\begin{figure*}[t]
    \centering
    \begin{subfigure}[b]{0.32\textwidth}
        \centering
        \includegraphics[width=\linewidth]{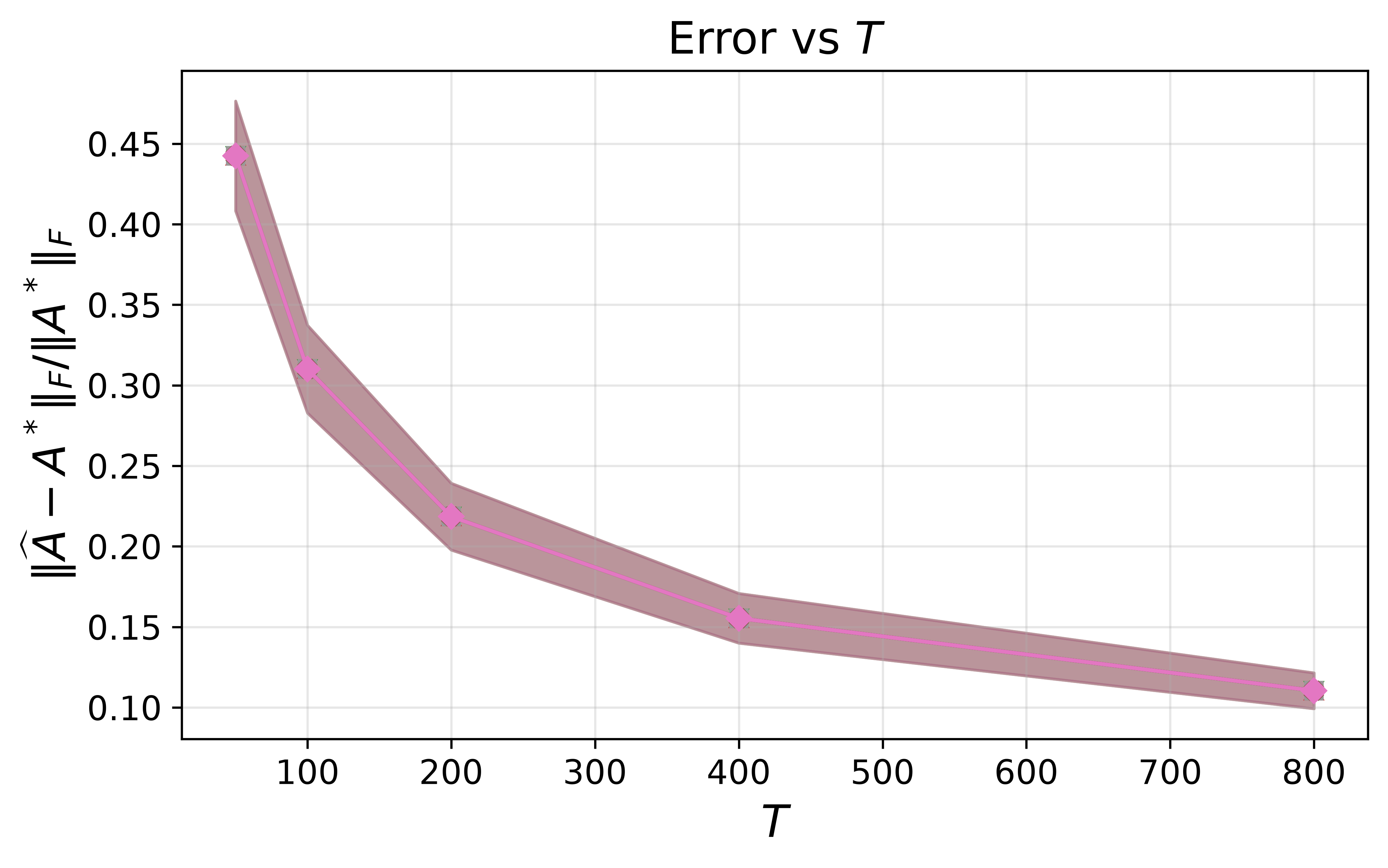}
        \caption{$\epsilon = 0$}
        \label{fig:T_random_heavy_tail_0}
    \end{subfigure}
    \hfill
    \begin{subfigure}[b]{0.32\textwidth}
        \centering
        \includegraphics[width=\linewidth]{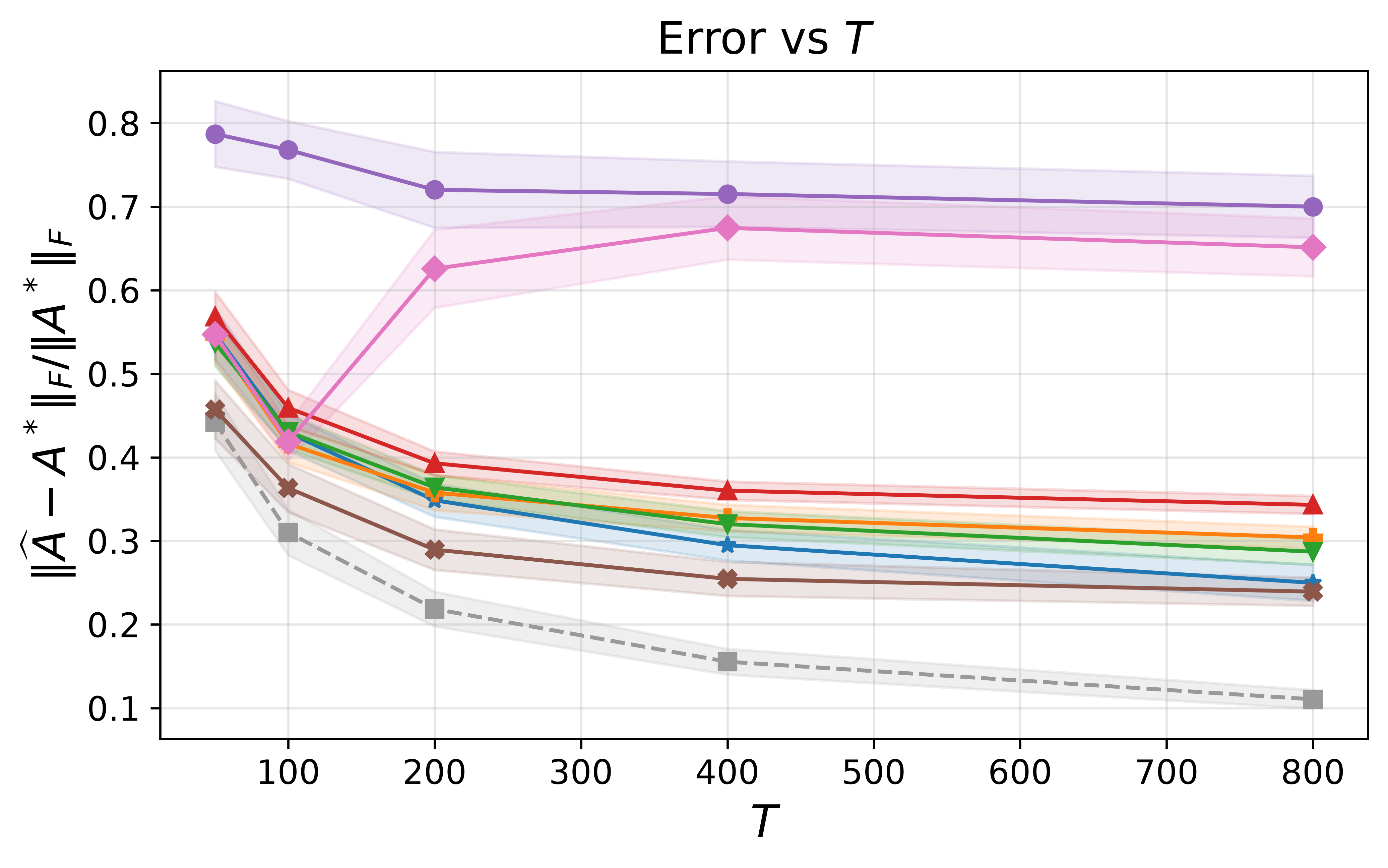}
        \caption{$\epsilon = 0.02$}
        \label{fig:T_random_heavy_tail_002}
    \end{subfigure}
    \hfill
    \begin{subfigure}[b]{0.32\textwidth}
        \centering
        \includegraphics[width=\linewidth]{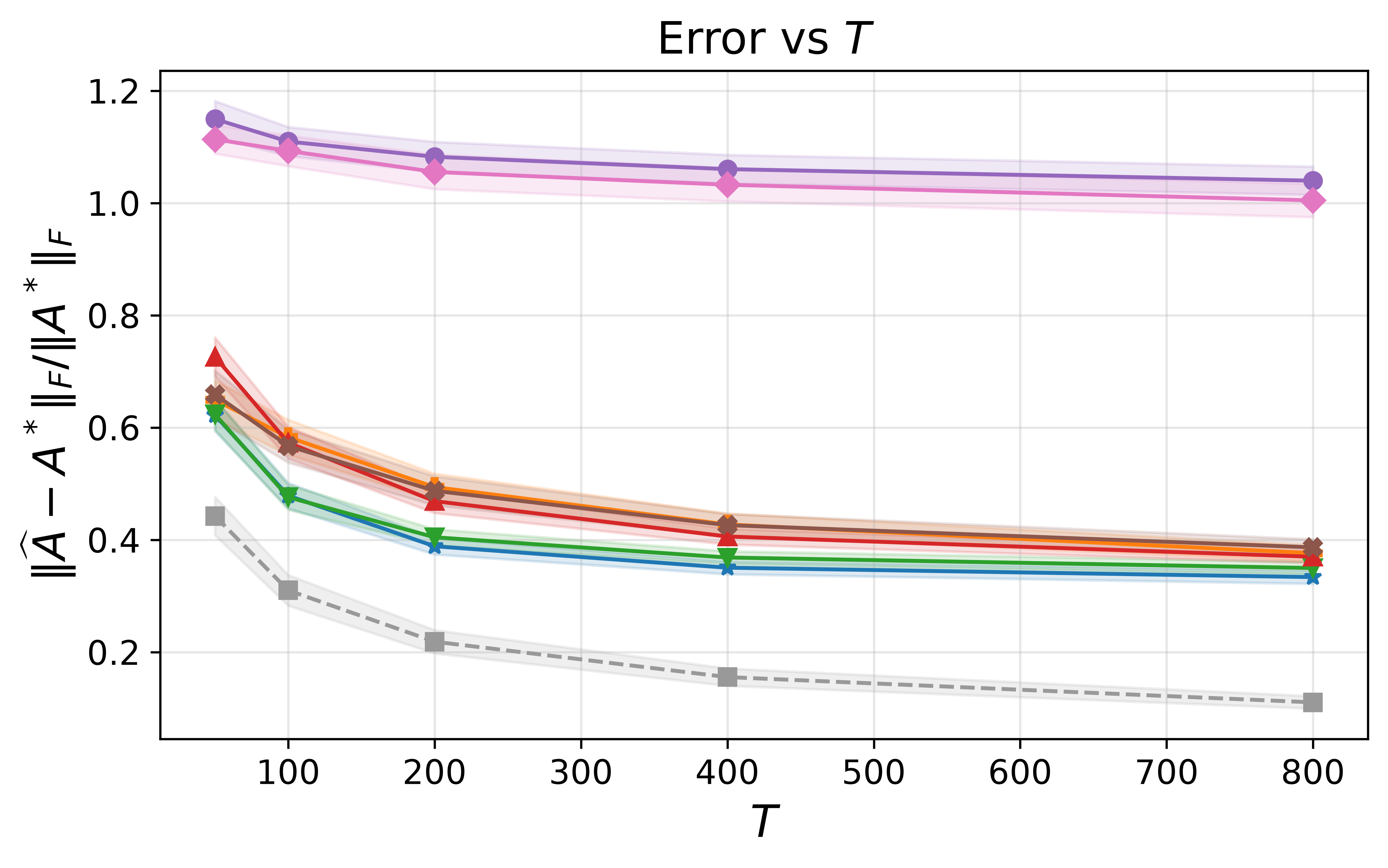}
        \caption{$\epsilon = 0.1$}
        \label{fig:T_random_heavy_tail_01}
    \end{subfigure}
    
    \vspace{1em} 
    
    \begin{subfigure}[b]{0.6\textwidth}
        \centering
        \includegraphics[width=\linewidth]{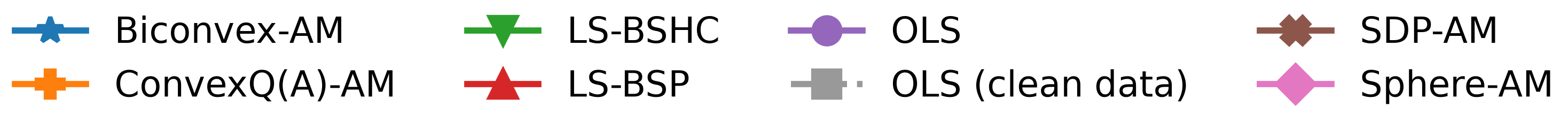}
        \label{fig:T_random_heavy_tail_legend}
    \end{subfigure}

    \makeatletter
    \renewcommand{\@makecaption}[2]{%
      \vspace{\abovecaptionskip}%
      \parbox{\textwidth}{ {\bf #1:} #2}%
      \vspace{\belowcaptionskip}%
    }
    \makeatother
    \caption{
    Frobenius norm error, $\norm{\widehat{A}-A^*}_F/{\norm{A^*}_F}$, for $T \in \{50,100,200,400,800\}$ with (a) $\epsilon=0$, (b) $\epsilon=0.02$, and (c) $\epsilon=0.1$, under the \emph{random heavy-tailed contamination} model.}
    \label{fig:T_random_heavy_tail}
\end{figure*}

\begin{figure*}[t]
    \centering
    \begin{subfigure}[b]{0.32\textwidth}
        \centering
        \includegraphics[width=\linewidth]{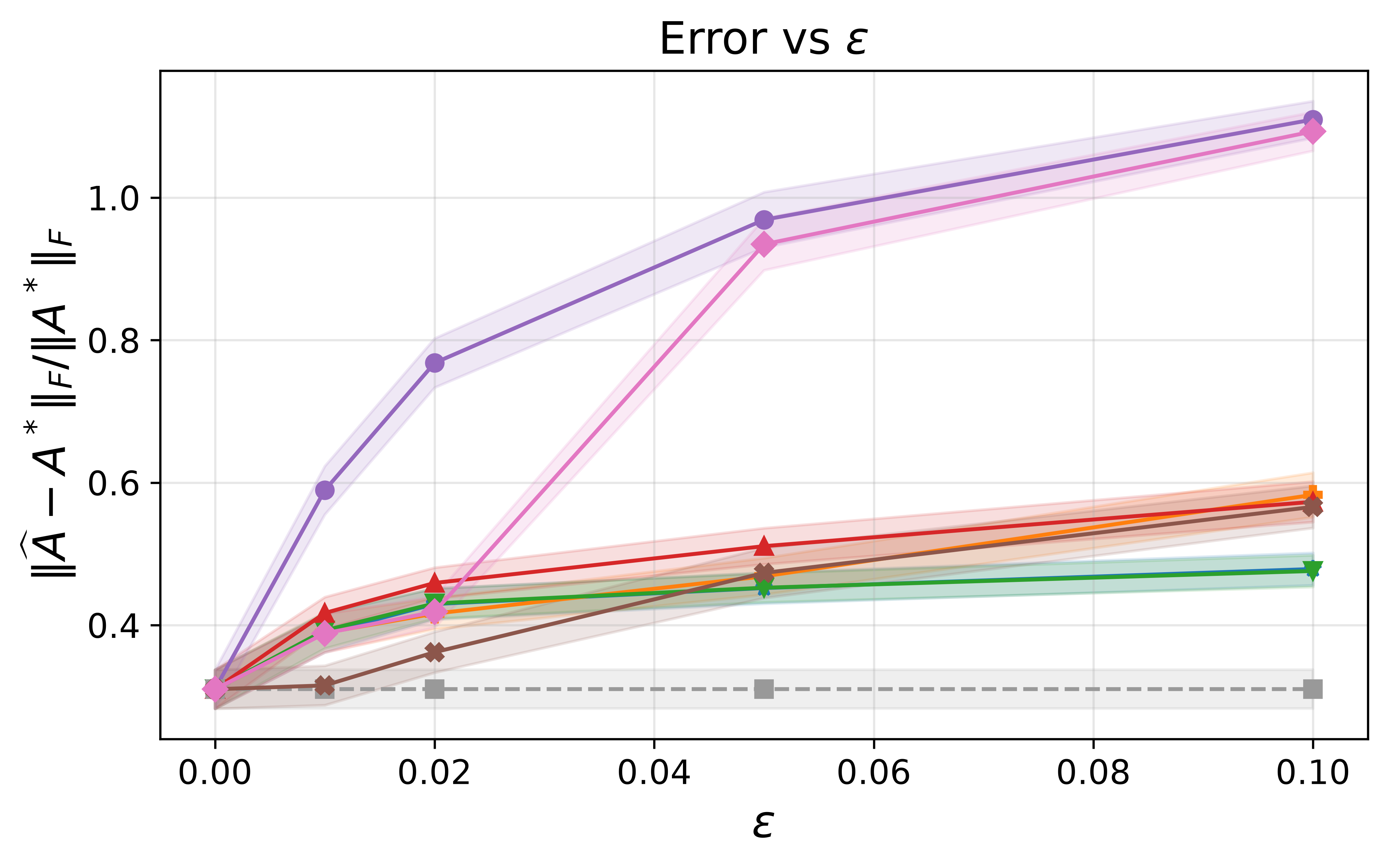}
        \caption{$T=100$}
        \label{fig:eps_random_heavy_tail_100}
    \end{subfigure}
    \hfill
    \begin{subfigure}[b]{0.32\textwidth}
        \centering
        \includegraphics[width=\linewidth]{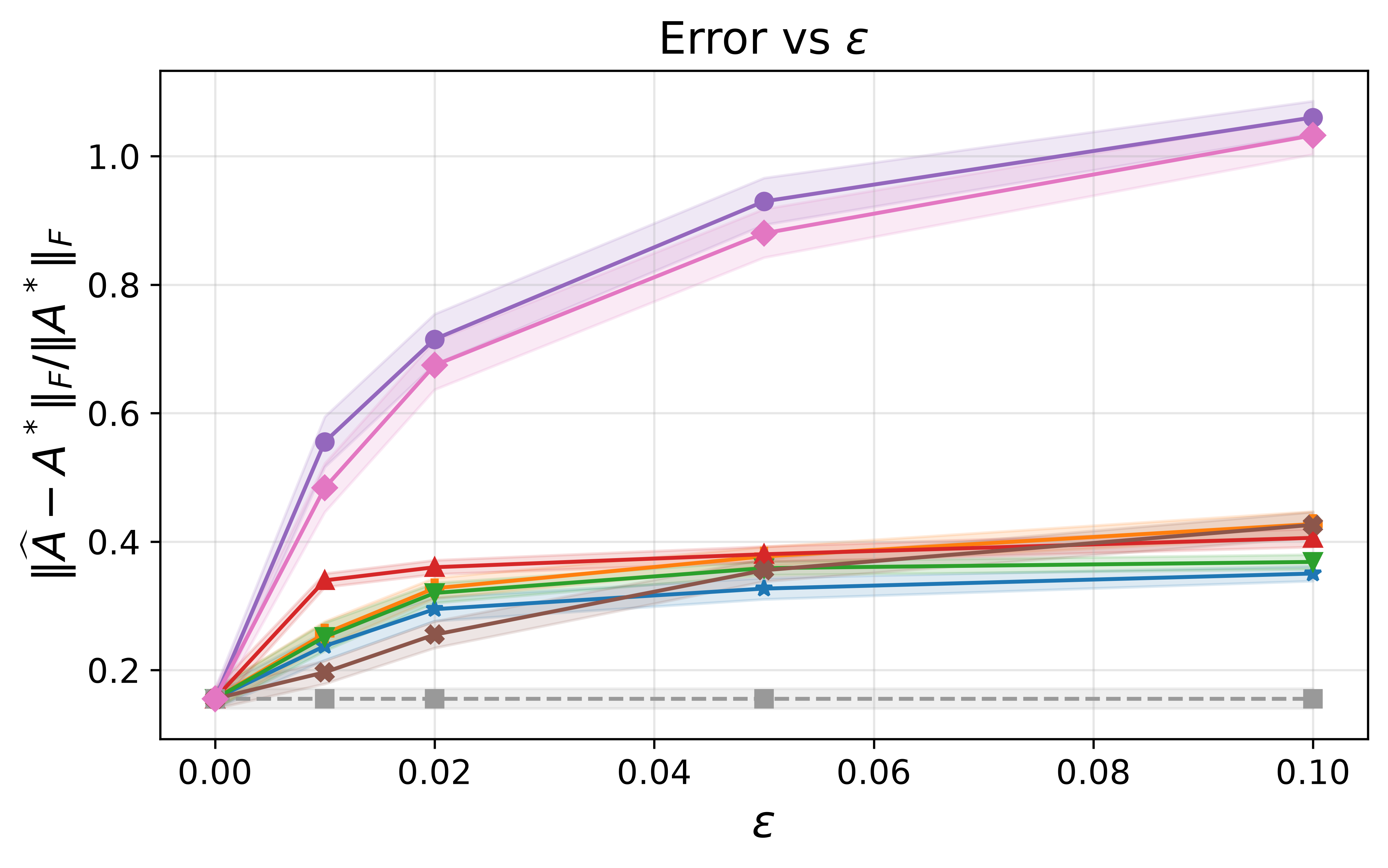}
        \caption{$T=400$}
        \label{fig:eps_random_heavy_tail_400}
    \end{subfigure}
    \hfill
    \begin{subfigure}[b]{0.32\textwidth}
        \centering
        \includegraphics[width=\linewidth]{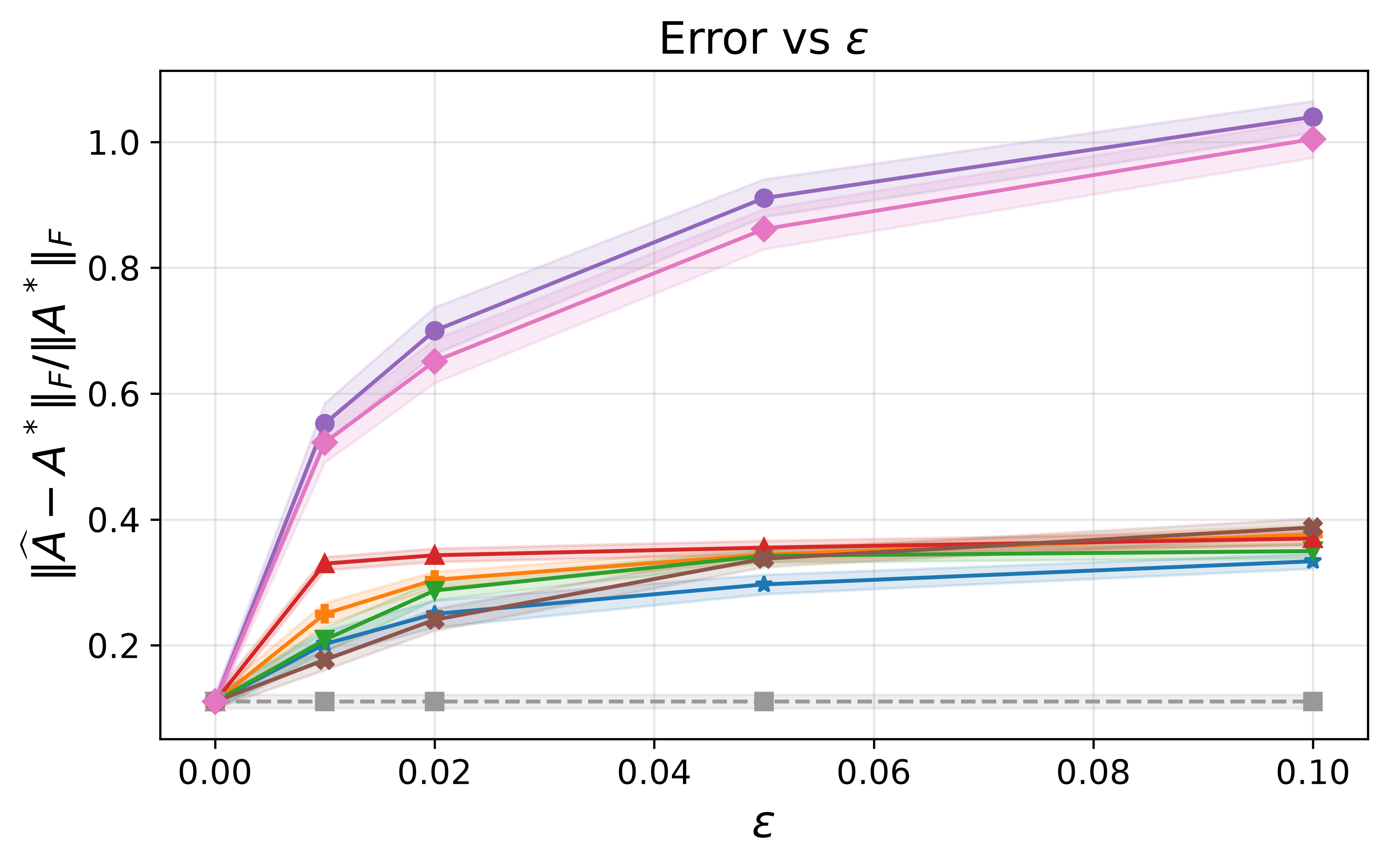}
        \caption{$T=800$}
        \label{fig:eps_random_heavy_tail_800}
    \end{subfigure}
    
    \vspace{1em} 
    
    \begin{subfigure}[b]{0.6\textwidth}
        \centering
        \includegraphics[width=\linewidth]{exp_results_synth/legend}
        \label{fig:eps_random_heavy_tail_legend}
    \end{subfigure}

    \makeatletter
    \renewcommand{\@makecaption}[2]{%
      \vspace{\abovecaptionskip}%
      \parbox{\textwidth}{ {\bf #1:} #2}%
      \vspace{\belowcaptionskip}%
    }
    \makeatother
    \caption{
    Frobenius norm error, $\norm{\widehat{A}-A^*}_F/{\norm{A^*}_F}$, for  $\epsilon \in \{0,0.01,0.02,0.05,0.1\}$ with (a) $T=100$, (b) $T=400$, and (c) $T=800$, under the \emph{random heavy-tailed contamination} model.}
    \label{fig:eps_random_heavy_tail}
\end{figure*}

In our experiments, we use the following notation to denote the estimators.
\begin{table}[H]
\centering
\begin{tabular}{ll}
\hline
\texttt{Sphere-AM} & Sphere Relaxation. \\
\texttt{SDP-AM} & Semi-definite program Relaxation. \\
\texttt{Biconvex-AM} & Bi-convex relaxation. \\
\texttt{ConvexQ(A)-AM} & Convexifying $Q(A)$ to estimate $z$. \\
\texttt{LS-BSP} & LS with block-sparse penalty. \\
\texttt{LS-BSHC} & LS with block-sparse constraints. \\
\texttt{OLS} & $\calA$-constrained least squares. \\
\hline
\end{tabular}
\end{table}
Here, LS stands for least squares, and with a slight abuse of notation, we denote $\calA$-constrained LS by \texttt{OLS}.

\paragraph{System setup.} All experiments were run on a machine with a Intel Core U9-285H (16 cores, 16 threads) and 32GB RAM, running Ubuntu 24.04. The code uses standard Python scientific computing libraries including, \emph{numpy}, \emph{scipy}, \emph{cvxpy}, and \emph{pandas}. All plots were generated with \emph{matplotlib} and all random seeds were fixed for reproducibility. Each experiment used 50 parallel worker threads.

\paragraph{Implementation details.} 
For fixed $A$, \texttt{Sphere-AM} uses the Mor\'{e}-Sorensen trust-region method \citep{sorenson1983}. \texttt{SDP-AM} is rewritten in an equivalent form with centered variables and solved using Burer-Monteiro method \citep{Burer2003}. For \texttt{Biconvex-AM} and \texttt{ConvexQ(A)-AM}, the $z$-update is solved directly with \emph{cvxpy}. When $z$ is fixed, $\widehat{A}$ is estimated via $\calA$-constrained weighted least squares (see Remark \ref{rem:solve-A}). For \texttt{LS-BSP} and \texttt{LS-BSHC}, the implementation follows the steps outlined in Remark~\ref{rem:impl-LS}. See Appendix~\ref{appsec:additional-exp} for more details. 

\paragraph{Experimental setup.} In all experiments, we consider $A^* \in \matR^{n \times n}$ with $n=50$, spectral radius $0.9$, and sparsity $k=150$, and we set $\calA = \set{A \in \matR^{n \times n}:\norm{A}_{1,1} \le \norm{A^*}_{1,1}}$. For a fixed $A^*$, we run $20$ Monte-Carlo simulations; in each simulation, we generate a sequence $(x_t)_{t=0}^T$ for a given $T$ by setting $x_0=0$ and using the VAR \eqref{eq:dyn_process}. The resulting sequence is then contaminated according to the chosen contamination model. We repeat this process for $20$ different matrices $A^*$ and report the mean recovery error together with its variance. For alternating minimization, initialize $A^{(0)}$ as the minimizer for $\calA$-constrained least squares over contaminated data.

Let us consider {\bf random heavy-tailed} contamination, in which outliers drawn from a Cauchy distribution with location and scale both set to $10$, are added to indices selected uniformly at random (with probability depending on $s$). See Appendix~\ref{appsec:additional-exp} for results for other contamination models.

\paragraph{Frobenius norm error against $T$.} Figure~\ref{fig:T_random_heavy_tail} shows the Frobenius error, ${\norm{\widehat{A}-A^*}_F}/{\norm{A^*}_F}$, against $T \in \set{50,100,200,400,800}$. When $\epsilon=0$, all estimators coincide with \texttt{OLS} and therefore achieve the same error. While \texttt{Sphere-AM} fails to recover $A^*$ for $\epsilon >0.02$ or $T>100$, it performs even better than \texttt{Biconvex-AM} when both $\epsilon$ and $T$ are small. In the small-$\epsilon$ regime, \texttt{SDP-AM} appears to outperform the other estimators with \texttt{Biconvex-AM} approaching the same error as $T$ increases (see Figure~\ref{fig:T_random_heavy_tail_002}). However, when $\epsilon = 0.1$, \texttt{Biconvex-AM} outperforms the other estimators, with \texttt{LS-BSHC} achieving comparable performance. The advantage of \texttt{Biconvex-AM} can be attributed to it involving a relatively tighter relaxation of the original problem, which handles filtering both at the level of individual observations and in the interactions between adjacent entries. For all estimators except \texttt{Sphere-AM} and \texttt{OLS}, the error decreases as $T$ increases, although the rate of decrease slows. 

\paragraph{Frobenius norm error against $\epsilon$.} Figure~\ref{fig:eps_random_heavy_tail} shows the Frobenius error, ${\norm{\widehat{A}-A^*}_F}/{\norm{A^*}_F}$ against $\epsilon \in \set{0,0.01,0.02,0.05,0.1}$. We observe similar behavior across all $T \in \set{100,400,800}$ with \texttt{OLS} and \texttt{Sphere-AM} showing the highest error. As expected, the error increases with $\epsilon$. While \texttt{SDP-AM} achieves a smaller error in the small-$\epsilon$ regime, \texttt{Biconvex-AM} performs best as $\epsilon$ increases. Across the range of $\epsilon$ considered, \texttt{LS-BSHC} achieves a recovery error comparable to that of \texttt{Biconvex-AM}. Note that while \texttt{Sphere-AM} generally performs comparably to \texttt{OLS}, it outperforms \texttt{LS-BSHC}, \texttt{LS-BSP} and \texttt{Biconvex-AM} when $T=100$ and $\epsilon \le 0.02$, suggesting it may be more suited for small $T,\epsilon$. 

%

\subsection*{AI use statement}
We have not used generative AI tools for synthetic-data generation, theoretical or conceptual development, mathematical claims or proofs, hypotheses, research methodology, experiments, 
dataset cleaning or reformatting, qualitative analysis, or interpretation of results. We did not use generative AI tools for any of the tasks with recommended disclosure, including brainstorming, literature discovery or summarization, drafting or editing text, creating figures, formatting references, or suggesting titles, keywords, or paper structure.

The algorithms and experiments were designed by the authors. Generative AI tools assisted with  parts of the implementation, primarily by optimizing and refactoring code. The authors reviewed all AI-assisted code and checked its correctness through testing and by comparing results against expected behavior.

We take responsibility for the final content of this work, including text, claims, or 
artifacts produced with the aid of generative AI.

\subsection*{Acknowledgements}
This work was supported by a Nanyang Associate Professorship (NAP) grant
from NTU Singapore.

\bibliographystyle{abbrvnat}
\bibliography{references}

\clearpage

\appendix
\thispagestyle{empty}
\onecolumn
\makeatletter
\addtocontents{toc}{\string\setcounter{tocdepth}{3}} 
\setcounter{tocdepth}{3}

\let\addcontentsline\oldaddcontentsline 
\makeatother

\begin{center}
    {\large \bf Table of Contents (Appendix)}
\end{center}
\startcontents[sections]
\printcontents[sections]{}{1}{}


\section{Technical tools}\label{sec:prelim-tools}
The proof of Theorem~\ref{thm:main-thm}  utilizes a number of technical tools which we recall here, starting with the 
definitions of Talagrand's $\gammafn{\alpha}$ functionals \citep{talagrand2014upper}. 
\begin{definition}[Talagrand's $\gammafn{\alpha}$ functionals \citep{talagrand2014upper}]\label{def:talagrand-gamma}
Let $(\calM,d)$ be a metric space. We define an adimissible sequence to be a sequence of subsets of $\calM$, $(\calM_r)_{r \ge 0}$, where $\absenc{\calM_0} = 1$ and $\absenc{\calM_r} \le 2^{2^r}$ for all $r \ge 1$. For any $0 < \alpha < \infty$, the $\gammafn{\alpha}$ function of $(\calM,d)$ is defined as,
\[
\gammafn{\alpha} (\calM,d) \coloneq \inf \sup_{s\in \calS} \sum_{r=0}^{\infty} 2^{r/\alpha} d(s,\calM_r)
\]
where the $\inf$ is over all admissible sequences of $\calM$.
\end{definition}
It is useful to note that $\gammafn{\alpha}$ satisfies the following properties.
\begin{enumerate}
\item \label{itm:tal-prp1} For any two metrics $d_1$ and $d_2$, $d_1 \le c\cdot d_2 \implies \gammafn{\alpha}(\calM,d_1) \le c \gammafn{\alpha}(\calM,d_2)$ (where $c>0$). 

\item \label{itm:tal-prp2} For $\widehat{\calM} \subseteq \calM$, $\gammafn{\alpha}(\widehat{\calS},d) \le c_{\alpha} \gammafn{\alpha}(\calS,d)$ where $c_{\alpha}>0$ is a constant that only depends on $\alpha$.

\item \label{itm:tal-prp3} If the map $f:(\calM,d_1) \rightarrow (\calS,d_2)$ is onto, and $R$-Lipschitz for all $x,y \in \calM$, then $\gammafn{\alpha}(\calS,d_2) \le c_{\alpha} R \gammafn{\alpha}(\calM,d_1)$ where $c_{\alpha}>0$ is a constant that only depends on $\alpha$.
\end{enumerate}
Property~\ref{itm:tal-prp1} follows from the definition while Properties \ref{itm:tal-prp2} and \ref{itm:tal-prp3} were stated in \cite[Theorem 1.3.6]{Talagrand2005-gf} (with $c_{\alpha} = 1$) for an alternative definition of the $\gamma_{\alpha}$ functional \cite[Definition 2.2.19]{talagrand2014upper} which is equivalent to Definition \ref{def:talagrand-gamma} up to constants depending only on $\alpha$. See also \cite[Section 2.3]{talagrand2014upper}.

For any set $\calP \subset \matR^{n}$, the Gaussian width of $\calP$ is defined as (see e.g., \cite{HDPbook})  
$$w\left(\calP\right) \coloneq \expec \left[\sup_{x\in \calP} {\dotprod{x}{g}}\right] \ \text{ for } \ g \sim \nrml\left(0, I_n\right).$$ 
Recall Talagrand's majorizing measure theorem \cite[Theorem 2.4.1]{talagrand2014upper}, which tells us that 
\begin{align} \label{eq:gamma_gauss_width}
    \gamma_2(\calP ,\norm{\cdot}_2) \asymp w(\calP).
\end{align}

Let $\calM$ be a given set of matrices and denote 
$$d_2(\calM) \coloneq \sup_{M \in \calM} \norm{M}_2, \quad  d_F(\calM) \coloneq \sup_{M \in \calM} \norm{M}_F.$$ 
Theorem $3.1$ of \cite{krahmer14} provides a concentration bound for the suprema of second order subgaussian choas process involving  positive semidefinite matrices.
\begin{theorem}[Theorem 3.1 of \cite{krahmer14}]\label{thm:psd-sup}
Let $\calM$ be a set of matrices and $\eta$ be a vector with independent, zero-mean, $R$-subgaussian entries with unit variance. Let,
\begin{align*}
F &\coloneq \gammafn{2} \left(\calM,\norm{\cdot}_2\right) \left[\gammafn{2} \left(\calM,\norm{\cdot}_2\right)+d_F(\calM)\right]+d_F(\calM)d_2(\calM),\\
V &\coloneq d_2(\calM)\left[\gammafn{2} \left(\calM,\norm{\cdot}_2\right)+d_F(\calM)\right], \text{ and, } U \coloneq d^2_2(\calM).
\end{align*}
There exist constants $c_1,c_2>0$ depending only on $R$ such that for any $t > 0$,
\begin{align*}
\prob \left(\sup_{M \in \calM} \absenc{\norm{M\eta}_2^2-\expec \left[\norm{M\eta}_2^2\right]} \ge c_1 F+t\right) 
\le 2\exp \left(-c_2\min\set{\frac{t^2}{V},\frac{t}{U}}\right).
\end{align*}
\end{theorem}
The following theorem taken from \cite[Chapter 3]{Wainwright2019} is well-known. It shows that Lipschitz and convex functions of independent and bounded random variables concentrate well.
\begin{theorem}[\cite{Wainwright2019}] \label{thm:bounded-func-conc}
Consider $X = (X_1,X_2,\dots,X_n)$ independent random variables with $X_i \in [a,b]$ almost surely. Let $f:\matR^{n} \rightarrow \matR$ be convex and $R$-Lipschitz with respect to $\norm{\cdot}_2$ norm. Then $\forall t \ge 0$
\[\prob \left(|f(X)-\expec f(X)|\ge t\right) \le 2\exp \left(\frac{-t^2}{2R^2(b-a)^2}\right).\] 
\end{theorem}
The following result taken from \cite[Theorem 1]{abbasi11} bounds the $\ell_2$ norm of self-normalized vector-valued martingales.
\begin{theorem}[{\cite[Theorem 1]{abbasi11}}]\label{thm:self-nrml-mart}
Let $(\calF_t)_{t=1}^\infty$ be a filtration. Let $(\eta_t)_{t=1}^\infty$ be a real valued stochastic process such that $\eta_t$ is $\calF_t$ measurable, and $\eta_{t+1}$ is $R$-sub-Gaussian conditioned on $\calF_{t}$. Let $(X_t)_{t=1}^\infty$ be a stochastic process where $X_t \in \matR^d$ is $\calF_t$ measurable. Let $V \in \matR^{n \times n}$ be a p.d matrix and for $t \ge 1$, denote
\begin{align*}
\bar{V}_t = V+\sum_{s=1}^t X_s X_s^\top \ \text{ and } \
S_t = \sum_{s=1}^t \eta_{s+1} X_s.
\end{align*}
Then for any $\delta>0$, it holds with probability $\ge 1-\delta$ that for all $t \ge 1$, 
\begin{align*}
\norm{S_t}^2_{\bar{V}_t^{-1}} \le 2R^2\log \left(\frac{\det(\bar{V}_t)^{1/2}\det(V)^{-1/2}}{\delta}\right)
\end{align*}
where $\norm{S_t}_{\bar{V}_t^{-1}} = \sqrt{S_t^\top (\bar{V}_t^{-1}) S_t}$. 
\end{theorem}

We will also use the following tail bound from \cite{Laurent2000Adaptive}.
\begin{lemma}[Lemma 1.1 of \cite{Laurent2000Adaptive}]\label{lem:chi-sq-dev}
Let $X_1,X_2,\dots,X_d \stackrel{\mathclap{\normalfont\mbox{i.i.d.}}}{\sim} \nrml\left(0,1\right)$ and let $\mu = \left(\mu_1,\mu_2,\dots,\mu_d\right)$ be a non-negative vector. Define,
\begin{align*}
\xi = \sum_{i=1}^d \mu_i (X_i^2-1).
\end{align*}
Then it holds $\forall v \ge 0$,
\begin{align*}
\Pr \left(\xi \ge 2\norm{\mu}_2\sqrt{v}+2\norm{\mu}_\infty v\right) \le e^{-v}.
\end{align*}
\end{lemma}


\section{Proof of Theorem~\ref{thm:main-thm}}\label{app-sec:analysis}
We start by noting that \eqref{eq:l12-ls-pen} can be written as
\begin{align*}
\left(\widehat{A},\widehat{U}\right) \in \argmin{\substack{A \in \calA \\ U \in \matR^{n \times (T-1)}}} \frac{1}{2} \norm{(\overline{Y}-AY-U)D}_F^2+\lambda \norm{UD}_{2,1}.
\end{align*}
Indeed, for any $t$ such that $D_t=0$ and $\widehat{u}_t \neq 0$, setting $\widehat{u}_t=0$ would reduce the objective value. Therefore, 
\begin{align*}
\set{t \in [T-1]:\widehat{u}_t \neq 0} \subseteq \set{t \in [T-1]: D_t = 1}.
\end{align*}
Denoting $\overline{E} = \matrixenc{\eta_2 & \eta_3 & \dots & \eta_{T}}$ and $E = \matrixenc{\eta_1 & \eta_2 & \dots & \eta_{T-1}}$, it is useful to note that
\begin{align*}
\overline{Y} = A^* Y + \overline{E} + \UStar.
\end{align*}
Moreover, it will be convenient to introduce the notation
\begin{align} \label{eq:proof_useful_notation}
\Delta_A = \widehat{A}-A^*, \ \ \Delta_U = \widehat{U}-U^*, \ \ \Delta = \matrixenc{\Delta_A & \Delta_U}, \ \text{ and }Y' = \matrixenc{Y \\ I}.
\end{align}
For any matrix $B \in \matR^{n \times m}$ and $\calI \subseteq [m]$, we will denote $(B)_\calI \in \matR^{n \times m}$ to be the column-wise restriction of $B$ on to $\calI$. The sets
\begin{align} \label{eq:set_Stil_def}
    \tilcalS := \set{t \in [T-1]: \ t \in \calS \ \text{ or } \ t+1 \in \calS^c}, \ \ 
    \tilcalS^c := \set{t \in [T-1]: \ t \in \calS^c \ \text{ and } \ t+1 \in \calS^c}
\end{align}
will play a crucial role in the analysis for decomposition of the error terms. Note that
$$\abs{\tilcalS} \leq 2\abs{\calS} = 2s \ \text{ and } \ \abs{\tilcalS^c} = T-1-\abs{\tilcalS} \geq T - 1 -2s.$$
Moreover $\tilcalS^c \subseteq S^c$, i.e., $\tilcalS^c$ is contained in the set of inliers, hence $\tilde{x}_t = x_t$ for each $t \in \tilcalS^c$. 

The starting point of the analysis is  Lemma~\ref{lem:main-inequ} below, the proof of which follows from a feasibility-optimality argument and is detailed in Appendix~\ref{app-sec:main-inequ}.
\begin{lemma} \label{lem:main-inequ}
If $\lambda \ge 2  \max_{t \in \set{2,\dots,T}} \norm{\eta_t}_2$, then
\begin{align}\label{eq:main-inequ}
\frac{1}{2} \norm{\Delta Y' D}_F^2 &\le \dotprod{\Delta_A}{\overline{E}DY^\top}+\frac{\lambda}{2} \Bigg(3\norm{(\Delta_UD)_{\tilcalS}}_{2,1}-\norm{(\Delta_UD)_{\tilcalSC}}_{2,1}\Bigg).
\end{align}
\end{lemma}
In the rest of this section, we will derive upper and lower bounds for the RHS and LHS of \eqref{eq:main-inequ} respectively.
\paragraph{Establishing an upper bound on right-hand side of \eqref{eq:main-inequ}.} In order to upper bound the RHS of \eqref{eq:main-inequ}, we need to control $\dotprod{\Delta_A}{\overline{E}DY^\top}$ and $\lambda$. We start by showing that $\max_{t \in \set{2,\dots,T}} \norm{\eta_t}_2^2$ is bounded with high probability which we will be used to set $\lambda$. 
For any $\delta \in (0,1)$, consider the event
\[\eventp{2} = \set{\max_{t \in \set{2,\dots,T}}\norm{\eta_t}_2 \le \sqrt{n}+2^{1/2}n^{1/4}\left(\log \left(\frac{T}{\delta}\right) \right)^{1/4}+2^{1/2}\left(\log \left(\frac{T}{\delta}\right) \right)^{1/2}}.
\]
\begin{lemma} \label{lem:bound-lambda}
For any  $\delta \in (0,1)$, we have $$\prob(\calE_2) \geq 1- \delta.$$  
\end{lemma}
The proof of Lemma~\ref{lem:bound-lambda} follows from applying Lemma~\ref{lem:chi-sq-dev} on $\norm{\eta_t}_2$, see Appendix~\ref{app-sec:bound-lambda} for details. Setting $\lambda = 2\overline{\lambda}$ (where $\overline{\lambda} = \sqrt{n}+\left(4n\log \left(\frac{T}{\delta}\right)\right)^{1/4}+\sqrt{2\log \left(\frac{T}{\delta}\right)}$), it follows that on $\eventp{2}$, we have that $\lambda \ge \max_{t \in \set{2,\dots,T}}\norm{\eta_t}_2$ holds.
Given the result on $\lambda$, the next step is to bound $\dotprod{\Delta_A}{\overline{E}DY^\top}$. Note that,
\begin{align*}
\dotprod{\Delta_A}{\overline{E}DY^\top} &= \dotprod{\Delta_A}{\overline{E}D_{\tilcalS} Y_{\tilcalS}^\top}+\dotprod{\Delta_A}{\overline{E}D_{\tilcalSC}Y_{\tilcalSC}^\top}\\
& \le \norm{\Delta_A}_{1,1} \left(\norm{\overline{E}D_{\tilcalS} Y_{\tilcalS}^\top}_{\infty,\infty}+\norm{\overline{E}D_{\tilcalSC}Y_{\tilcalSC}^\top}_{\infty,\infty}\right).
\end{align*}

Now recall the definition of $D_t$. From the definition, we know that for all $t$ such that $D_t=1$, $\norm{\tildex_t}_\infty \le r$ (for some parameter $r$). Also note that since $D_{\tilcalS} Y_{\tilcalS}^\top$ is supported on $\tilcalS$ (which implies $D_{\tilcalS}$ has at most $\absenc{\tilcalS}$ non-zero terms on the diagonal), each column can only have at most $\absenc{\tilcalS}$ non-zero entries. Therefore, each column of $D_{\tilcalS} Y_{\tilcalS}^\top$ is $\absenc{\tilcalS}$-sparse with $\norm{\cdot}_{\infty} \le r$. Denote 
$$\calB\left(r,\absenc{\tilcalS}\right) = \set{b \in \matR^{T-1}:\norm{b}_{\infty} \le r, \norm{b}_0 \le \absenc{\tilcalS}}$$ and consider the event
\[
\eventp{1a}=\set{\norm{\overline{E}D_{\tilcalS} Y_{\tilcalS}^\top}_{\infty,\infty} \le  c\left( w \left(\calB\left(r,\absenc{\tilcalS}\right)\right)+ r \absenc{\tilcalS} \sqrt{\log \left(\frac{T}{\delta}\right)}\right)}.
\]
for a suitably large constant $c > 1$. 
\begin{lemma}\label{lem:inf-norm-cost}
For any $\delta \in (0,1)$, we have 
\[
\prob(\eventp{1a}) \ge 1-\delta.
\]
\end{lemma}
The proof of Lemma~\ref{lem:inf-norm-cost} follows from observing that 
$\norm{\overline{E}D_{\tilcalS} Y_{\tilcalS}^\top}_{\infty,\infty}$ can be bounded by $\sup_{b \in \calB\left(r,\absenc{\tilcalS}\right)}  \absenc{g^\top b}$ for $g {\sim} \nrml \left(0,I_{T-1}\right)$ using the properties of $\overline{E}$ and $D_{\tilcalS} Y_{\tilcalS}^\top$, followed by establishing an upper bound on $\sup_{b \in \calB\left(r,\absenc{\tilcalS}\right)}  \absenc{g^\top b}$ using Theorem~\ref{thm:bounded-func-conc} \cite[from Chapter 3]{Wainwright2019}. See Appendix~\ref{app-sec:inf-norm-cost} for details. 

Denote the event $$\eventp{1b} = \set{D_t=1 \; \forall t \in \tilcalSC}.$$ We can show that for $\delta \in (0,1)$ and $r$ suitably large, $\eventp{1b}$ holds with probability at least $1-\delta$. 
\begin{lemma}\label{lem:all-unc-select}
There exists a constant $c > 1$ such that for any $\delta \in (0,1)$, and $r = cJ(A^*) \sqrt{\log \left(\frac{Tn}{\delta}\right)}$, it holds that
\begin{align*}
\prob\left(\eventp{1b}\right) \ge 1-\delta.
\end{align*}
\end{lemma}
The proof of Lemma~\ref{lem:all-unc-select} follows from using the structure inherent in the uncontaminated data and using standard concentration bounds. The details are available in Appendix~\ref{app-sec:all-unc-select}.

Next we establish a high-probability upper bound on $\norm{\overline{E} D_{\tilcalSC} Y_{\tilcalSC}^\top}_{\infty,\infty}$. To this end, consider the event (for a suitably large constant $c_1 > 1$)
$$\eventp{1c} = \set{\norm{\overline{E} D_{\tilcalSC} Y_{\tilcalSC}^\top}_{\infty,\infty}\le c_1 r \sqrt{T \log \left(\frac{n}{\delta}\right)}}.$$
\begin{lemma}\label{lem:inf-norm-comple}
For any $\delta \in (0,1)$, it holds that 
\[
\prob(\eventp{1c}) \ge 1-\delta.
\]
\end{lemma}
The proof of Lemma~\ref{lem:inf-norm-comple} involves bounding the tails of an appropriately defined self-normalized martingale using Theorem~\ref{thm:self-nrml-mart}, and is provided in Appendix~\ref{app-sec:inf-norm-comple}.

Before moving on to the final term $\dotprod{\Delta_A}{\overline{E}DY^\top}$ we first observe the following.
\begin{remark}\label{rem:support-cone}
Denote $\calK = \supp(A^*)$. Since $\norm{A^*}_{1,1} \ge \norm{\widehat{A}}_{1,1}$, we obtain, 
\begin{align*}
\norm{A^*}_{1,1} &\ge \norm{\Delta_A+A^*}_{1,1}\\
&=\norm{(\Delta_A)_{\calK}+(A^*)_{\calK}}_{1,1}+\norm{(\Delta_A)_{\calK^c}}_{1,1}\\
&\ge \norm{(A^*)_{\calK}}_{1,1}+\norm{(\Delta_A)_{\calK^c}}_{1,1}-\norm{(\Delta_A)_{\calK}}_{1,1}
\end{align*}
which implies
\begin{align*}
 \norm{(\Delta_A)_{\calK^c}}_{1,1} & \le \norm{(\Delta_A)_{\calK}}_{1,1}.
\end{align*}
So $\Delta_A \in \calC_{\calK}$ where,
\begin{align*}
\calC_{\calK} \coloneq \set{H \in \matR^{n \times n}: \norm{(H)_{\calK^c}}_{1,1} \le \norm{(H)_{\calK}}_{1,1}}.
\end{align*}
Note that this implies, $\norm{\Delta_A}_{1,1} \le 2\norm{(\Delta_A)_{\calK}}_{1,1} \le 2\sqrt{k} \norm{\Delta_A}_F$ where $k= |\calK|$. We denote $\widetilde{\calC}_{\calK} = \calC_{\calK} \cap \frobsphere$.
\end{remark}
On $\eventp{1a} \cap \eventp{1b}\cap \eventp{1c}$, we have,
\begin{align*}
\dotprod{\Delta_A}{\overline{E}DY^\top} &\le \norm{\Delta_A}_{1,1} \left(\norm{\overline{E}D_{\tilcalS} Y_{\tilcalS}^\top}_{\infty,\infty}+\norm{\overline{E}D_{\tilcalSC} Y_{\tilcalSC}^\top}_{\infty,\infty}\right) \\
&\le c\norm{\Delta_A}_{1,1}\left(w \left(\calB\left(r,\absenc{\tilcalS}\right)\right)+r\left(\absenc{\tilcalS}+\sqrt{T}\right)\sqrt{\log \left(\frac{n}{\delta}\right)}\right)\\
&\le c\sqrt{k}\norm{\Delta_A}_{F}\left(w \left(\calB\left(r,\absenc{\tilcalS}\right)\right)+r\left(\absenc{\tilcalS}+\sqrt{T}\right)\sqrt{\log \left(\frac{n}{\delta}\right)}\right)
\end{align*}
where the last inequality comes from the fact $\norm{\Delta_{A}}_{1,1} \le 2\sqrt{k} \norm{\Delta_A}_F$ (see Remark \ref{rem:support-cone}).

Next we can bound $w \left(\calB\left(r,\absenc{\tilcalS}\right)\right)$. Note that $\calB\left(r,\absenc{\tilcalS}\right) \subseteq \set{v \in \matR^{T-1}: \norm{v}_1 \le r\absenc{\tilcalS}}$. Therefore,
\begin{align*}
w \left(\calB\left(r,\absenc{\tilcalS}\right)\right) &\le w \left(\set{v \in \matR^{T-1}: \norm{v}_1 \le r\absenc{\tilcalS}}\right)\le cr\absenc{\tilcalS} \sqrt{\log T}
\end{align*}
for a constant $c > 1$, where we use the fact that $w \left(\set{v \in \matR^{T-1}: \norm{v}_1 \le 1}\right) \le \sqrt{\log T}$ \cite[refer to Example 7.5.8]{HDPbook}.

Let $\eventp{1} \coloneq \eventp{1a} \cap \eventp{1b} \cap \eventp{1c}$. On $\eventp{1}$, we have,
\begin{align*}
\dotprod{\Delta_A}{\overline{E}DY^\top} &\le c\sqrt{k}\norm{\Delta_A}_{F}\left(r\absenc{\tilcalS} \sqrt{\log T}+r\left(\absenc{\tilcalS}+\sqrt{T}\right)\sqrt{\log \left(\frac{n}{\delta}\right)}\right) \le \overline{V} \norm{\Delta_A}_F,
\end{align*}
where 
\begin{align*}
\overline{V} = cr\sqrt{k}\bigg(\absenc{\tilcalS}\max\set{\sqrt{\log T},\sqrt{\log \left(\frac{n}{\delta}\right)}}+\sqrt{T\log \left(\frac{n}{\delta}\right)}\bigg)
\end{align*}
and $r=cJ(A^*)\sqrt{\log \left(\frac{Tn}{\delta}\right)}$.
To summarize, we have shown that on $\calE_1$, $\dotprod{\Delta_A}{\overline{E}DY^\top} \le \overline{V} \norm{\Delta_A}_F$ holds.

Setting $\lambda = 2\overline{\lambda}$, on $\eventp{1} \cap \eventp{2}$, we have from \eqref{eq:main-inequ} that
\begin{align*}
&\frac{1}{2} \norm{\Delta Y' D}_F^2 \le \overline{V} \norm{\Delta_A}_F + \overline{\lambda} \left(3\norm{(D\Delta_U)_{\tilcalS}}_{2,1}-\norm{(D\Delta_U)_{\tilcalSC}}_{2,1}\right)
\end{align*}
which implies,
\begin{align*}
\norm{(D\Delta_U)_{\tilcalSC}}_{2,1} \le \frac{\overline{V}}{\overline{\lambda}} \norm{\Delta_A}_F + 3\norm{(D\Delta_U)_{\tilcalS}}_{2,1}.
\end{align*}

\paragraph{Establishing a lower bound on the left-hand side of \eqref{eq:main-inequ}.} Now it remains to lower bound $\frac{1}{2} \norm{\Delta Y' D}_F^2$. To this end, note that, 
\begin{align*}
\norm{\Delta Y' D}_F^2 &= \norm{\matrixenc{\Delta_A & \Delta_U}\matrixenc{Y \\ I_{T-1}}D}_F^2\\
&=\norm{\left(\Delta_AY + \Delta_U\right)D}_F^2\\
&=\norm{\Delta_AYD}_F^2+\norm{ \Delta_UD}_F^2+2\dotprod{\Delta_AYD}{\Delta_U D}.
\end{align*}
Therefore, we can lower bound the $\norm{\Delta Y' D}_F^2$ by lower bounding $\norm{\Delta_A YD}_F^2$ and controlling $\dotprod{\Delta_AYD}{\Delta_U D}$. We start off with a lower bound on $\norm{\Delta_A YD}_F^2$. 
Consider the event $$\eventp{3} = \set{\norm{\Delta_A YD}_F^2 \ge \frac{T}{4} \norm{\Delta_A}^2_F}.$$
\begin{lemma}\label{lem:lwr-bnd-deltaA}
There exists a constant $c_1 > 1$ such that for any $\delta \in (0,1)$, if $$T \ge 4s+c_1 J^4(A^*) \log^2 \left(\frac{1}{\delta}\right) k \log \left(\frac{n^2}{k}\right)$$ then $\prob(\calE_3) \ge 1-\delta$.
\end{lemma}
The proof of Lemma~\ref{lem:lwr-bnd-deltaA} is available in Appendix~\ref{app-sec:lwr-bnd-deltaA}. It follows from an application of Theorem~\ref{thm:psd-sup} (Theorem 3.1 of \cite{krahmer14}) on an appropriate set of matrices and using the results on the Gaussian width of $\widetilde{C}_{\calK}$.

Next consider $\dotprod{\Delta_A Y D}{\Delta_U D}$. For $\delta \in (0,1)$ and suitably large constants $c_1, c_2 > 1$, consider the event
\[
\eventp{4} = \set{\absenc{\dotprod{\Delta_A Y D}{\Delta_U D}} \le c_1 \frac{\overline{V}}{\overline{\lambda}}R \norm{\Delta_A}^2_F+c_2 r R \sqrt{\absenc{\tilcalS} k} \norm{\Delta_UD}_F\norm{\Delta_A}_F}
\]
where $R= J(A^*) \sqrt{k \log \left(\frac{n^2}{k}\right)  \log \left(\frac{T}{\delta}\right)}$.
\begin{lemma}\label{lem:bnd-cross-trm}
For any $\delta \in (0,1)$,  we have
 $$\prob\left(\eventp{4}\right) \ge 1-\delta.$$
\end{lemma}
We refer the reader to the Appendix~\ref{app-sec:bnd-cross-trm} for the complete proof of Lemma~\ref{lem:bnd-cross-trm}. 
The proof involves a careful decomposition over $\tilcalS$ and $\tilcalSC$, and bounding the individual terms using tools recalled in Section \ref{sec:prelim-tools}.  

We can now establish the final lower bound on $\norm{\Delta Y' D}_F^2$. On $\eventp{3} \cap \eventp{4}$,
\begin{align*}
\norm{\Delta Y' D}_F^2 &= \norm{\Delta_A Y D}_F^2+\norm{\Delta_U D}_F^2+2\dotprod{\Delta_A Y D}{\Delta_U D} \\
&\ge  \frac{T}{4} \norm{\Delta_A}_F^2-c_2 \frac{\overline{V}}{\overline{\lambda}} R \norm{\Delta_A}_F^2-c_3 r R\sqrt{k \absenc{\tilcalS}}\norm{\Delta_A}_F \norm{\Delta_U D}_{F}+\norm{\Delta_U D}_{F}^2\\
&\ge  \frac{T}{4} \norm{\Delta_A}_F^2-c_2 \frac{\overline{V}}{\overline{\lambda}} R \norm{\Delta_A}_F^2-c'_3 r^2 R^2k \absenc{\tilcalS}\norm{\Delta_A}^2_F -\frac{1}{2}\norm{\Delta_U D}^2_{F}+\norm{\Delta_U D}_{F}^2.
\end{align*}
If $T \ge c \left( \frac{\overline{V}}{\overline{\lambda}} R+ r^2 R^2k \absenc{\tilcalS}\right)$ for a suitably large constant $c \ge 1$, we get that,
\begin{align*}
\frac{1}{2}\norm{\Delta Y' D}_F^2 \ge c_1 T \norm{\Delta_A}_F^2+c_2 \norm{\Delta_U D}_{F}^2. 
\end{align*}

This means that on $\eventp{1} \cap \eventp{2} \cap \eventp{3} \cap \eventp{4}$, 
\begin{align*}
c_1 T \norm{\Delta_A}_F^2+c_2 \norm{\Delta_U D}_{F}^2 
&\le \overline{V} \norm{\Delta_A}_F+\overline{\lambda}\left(3 \norm{(\Delta_U D)_{\tilcalS}}_{2,1}-\norm{(\Delta_U D)_{\tilcalSC}}_{2,1}\right) \\
&\le \overline{V} \norm{\Delta_A}_F+c_3 \overline{\lambda} \sqrt{\absenc{\tilcalS}} \norm{\Delta_U D}_{F} \\
&\le \frac{\overline{V}}{\sqrt{c_1 T}} \sqrt{c_1 T\norm{\Delta_A}^2_F}+\frac{c_3 \overline{\lambda} \sqrt{\absenc{\tilcalS}}}{\sqrt{c_2}} \sqrt{c_2\norm{\Delta_U D}^2_{F}} \\
& \le \left(\frac{\overline{V}^2}{c_1 T}+c'_3 \overline{\lambda}^2 \absenc{\tilcalS}\right)^{1/2} \left(c_1 T \norm{\Delta_A}_F^2+c_2 \norm{\Delta_U D}_{F}^2\right)^{1/2}.
\end{align*}
The above implies that for some constants $c_1, c_2 \ge 1$,
\begin{align}\label{eq:final-bnd}
\norm{\Delta_A}_F^2+\frac{1}{T}\norm{\Delta_U D}_{F}^2 &\le c_1\frac{r^2k\bigg(\absenc{\tilcalS}\max\set{\sqrt{\log T},\sqrt{\log \left(\frac{n}{\delta}\right)}}+\sqrt{T\log \left(\frac{n}{\delta}\right)}\bigg)^2}{T^2}\\
&+c_2 \frac{\left(\sqrt{n}+n^{1/4} \log^{1/4} \left(\frac{T}{\delta}\right)+2^{1/2} \log^{1/2} \left(\frac{T}{\delta}\right)\right)^2 \absenc{\tilcalS}}{T}. \nonumber
\end{align}

\paragraph{Lower bound on $T$.} Note that for the statements to hold, we need, 
\begin{align} \label{eq:T_conds_necess}
T \ge c \left( \frac{\overline{V}}{\overline{\lambda}} R+ r^2 R^2k \absenc{\tilcalS}\right) \ \text{ and } \   T \ge 4s+c_1 J^4(A^*) \log^2 \left(\frac{1}{\delta}\right) k \log n.  
\end{align}
Recall the definitions of $\overline{V}, \overline{\lambda}, R$ and $r$.
\begin{align*}
R&=J(A^*) \sqrt{k \log \left(\frac{n^2}{k}\right) \log \left(\frac{T}{\delta}\right)}\\
r&=c_1J(A^*) \sqrt{\log \left(\frac{nT}{\delta}\right)}\\
\overline{\lambda} &= \sqrt{n}+n^{1/4} \log^{1/4} \left(\frac{T}{\delta}\right)+2^{1/2} \log^{1/2} \left(\frac{T}{\delta}\right)\\
\overline{V} &= c_2J(A^*) \sqrt{\log \left(\frac{nT}{\delta}\right)}\sqrt{k}\bigg(\absenc{\tilcalS}\max\set{\sqrt{\log T},\sqrt{\log \left(\frac{n}{\delta}\right)}}+\sqrt{T\log \left(\frac{n}{\delta}\right)}\bigg).
\end{align*}
Since $\overline{\lambda} \ge \sqrt{\log\left(\frac{T}{\delta}\right)}$, hence $T \ge c \left( \frac{\overline{V}}{\sqrt{\log\left(\frac{T}{\delta}\right)}} R+ r^2 R^2k \absenc{\tilcalS}\right)$ implies the first condition in \eqref{eq:T_conds_necess}. Also note that since $\absenc{\tilcalS} \le 2s$, hence
\begin{align*}
r^2 R^2 k \absenc{\tilcalS} &\le \widetilde{c}_1 J^4(A^*) s k^2 \log \left(\frac{n^2}{k}\right) \log \left(\frac{T}{\delta}\right) \log \left(\frac{nT}{\delta}\right) \text{ and,}\\
\frac{\overline{V}}{\sqrt{\log\left(\frac{T}{\delta}\right)}} R & \le \widetilde{c}_2 J^2(A^*) k \sqrt{\log \left(\frac{n^2}{k}\right)} \sqrt{\log \left(\frac{nT}{\delta}\right)}\bigg(s\cdot \max\set{\sqrt{\log T},\sqrt{\log \left(\frac{n}{\delta}\right)}}+\sqrt{T\log \left(\frac{n}{\delta}\right)}\bigg).
\end{align*}
Now as long as 
\begin{align*}
T &\ge c_1 J^4(A^*) s k^2 \log \left(\frac{n^2}{k}\right) \log \left(\frac{T}{\delta}\right) \log \left(\frac{nT}{\delta}\right) \text{ and,}\\ 
T &\ge c_2 J^2(A^*) k s \sqrt{\log \left(\frac{n^2}{k}\right)} \sqrt{\log \left(\frac{nT}{\delta}\right)} \max\set{\sqrt{\log T},\sqrt{\log \left(\frac{n}{\delta}\right)}}, \text{ and,} \\
T &\ge c_3 J^2(A^*) k \sqrt{\log \left(\frac{n^2}{k}\right)} \sqrt{\log \left(\frac{nT}{\delta}\right)} \sqrt{T \log \left(\frac{n}{\delta}\right)}
\end{align*}
hold for suitably large constants $c_1,c_2,c_3 \geq 1$, then $T$ would satisfy \eqref{eq:T_conds_necess}. It is easy to see that if we let, 
$T \ge c J^4(A^*) k^2 \log \left(\frac{n^2}{k}\right) \cdot \max \set{s,\log \left(\frac{n}{\delta}\right)} \log^2 \left(\frac{nT}{\delta}\right)$ then \eqref{eq:T_conds_necess} is ensured.
To isolate $T$, we can let $M = \sqrt{c} J^2(A^*)k \sqrt{\log \left(\frac{n^2}{k}\right)}\max\set{\sqrt{s},\sqrt{\log \left(\frac{n}{\delta}\right)}}$, so that the the aforementioned condition can be rewritten as $\sqrt{T} \ge M \log \left({\frac{nT}{\delta}} \right)$.   
Using \cite[Lemma 21]{Papini2022}, we get that 
\begin{align*}
\sqrt{T} \ge 2M \log \left(M \sqrt{\frac{n}{\delta}} \right)
\end{align*}
and taking the square of both sides, we get the sufficient condition
\begin{align*}
T \ge cJ^4(A^*) k^2 \cdot \log \left(\frac{n^2}{k}\right) \cdot  \max \set{s,\log \left(\frac{n}{\delta}\right)} \log^2 \left( \frac{J^4(A^*) k^2 \cdot \log \left(\frac{n^2}{k}\right) \cdot \max \set{s,\log \left(\frac{n}{\delta}\right)} n}{\delta}\right).
\end{align*}

\begin{remark}\label{rem:upper-st}
Note that,
\begin{align*}
T \ge cJ^4(A^*) k^2 \cdot \log \left(\frac{n^2}{k}\right) \cdot  \max \set{s,\log \left(\frac{n}{\delta}\right)} \log^2 \left( \frac{J^4(A^*) k^2 \cdot \log \left(\frac{n^2}{k}\right) \cdot \max \set{s,\log \left(\frac{n}{\delta}\right)} n}{\delta}\right)
\end{align*}
implies,
\begin{align*}
\frac{s}{T} \le \frac{1}{cJ^4(A^*) k^2 \cdot \log \left(\frac{n^2}{k}\right) \cdot \log^2 \left( \frac{ n}{\delta}\right)}
\end{align*}
\end{remark}

\paragraph{The final bound.} Substituting $r$ in \eqref{eq:final-bnd} we obtain after some simplifications that
\begin{align*}
\norm{\Delta_A}_F^2+\frac{1}{T}\norm{\Delta_U D}_{F}^2 &\le c_1J^2(A^*) k \log \left(\frac{Tn}{\delta}\right)\bigg(\max\set{\log T,\log \left(\frac{n}{\delta}\right)}\bigg)\frac{s^2}{T^2}\\
&+c_2\frac{J^2(A^*) k\log \left(\frac{Tn}{\delta}\right)\log \left(\frac{n}{\delta}\right)}{T}\\
&+c_3 \max \set{n,\log\left(\frac{T}{\delta}\right)}\frac{s}{T}.
\end{align*}
Using the result from Remark~\ref{rem:upper-st}, we can further simplify this to
\begin{align*}
\norm{\Delta_A}_F^2+\frac{1}{T}\norm{\Delta_U D}_{F}^2 &\le c_1\frac{J^2(A^*) k \log \left(\frac{Tn}{\delta}\right)\bigg(\max\set{\log T,\log \left(\frac{n}{\delta}\right)}\bigg)}{J^4(A^*) k^2 \cdot \log \left(\frac{n^2}{k}\right) \cdot \log^2 \left( \frac{ n}{\delta}\right)}\frac{s}{T}\\
&+c_2\frac{J^2(A^*) k\log \left(\frac{Tn}{\delta}\right)\log \left(\frac{n}{\delta}\right)}{T}\\
&+c_3 \max \set{n,\log\left(\frac{T}{\delta}\right)}\frac{s}{T}\\
&\le c_1\frac{J^2(A^*) k\log \left(\frac{Tn}{\delta}\right)\log \left(\frac{n}{\delta}\right)}{T}\\
&+c_2 \max \set{n,\log^2\left(\frac{T}{\delta}\right)}\frac{s}{T}
\end{align*}
which completes the proof.

\subsection{Proof of Lemma~\ref{lem:main-inequ}}\label{app-sec:main-inequ}

\begin{proof}
Note that since $\widehat{A},\widehat{U}$ is the optimal solution and $A^*, \UStar$ is a feasible solution, we get,
\begin{align*}
&\frac{1}{2} \norm{\left(\overline{Y}-\widehat{A}Y-\widehat{U}\right) D}_F^2+\lambda \norm{\widehat{U} D}_{2,1} \le \frac{1}{2} \norm{\left(\overline{Y}-A^*Y-\UStar\right) D}_F^2+\lambda \norm{\UStar D}_{2,1}
\end{align*}
which implies
\begin{align*}
&\frac{1}{2}\norm{(A^*-\widehat{A})YD + \overline{E}D+(\UStar-\widehat{U})D}_F^2 \le \frac{1}{2} \norm{\overline{E}D}_F^2+ \left(\lambda \norm{\UStar D}_{2,1}-\lambda \norm{\widehat{U} D}_{2,1}\right).
\end{align*}
Denoting $\Delta_A = \widehat{A}-A^*$ and $\Delta_U = \widehat{U}-U^*$  note that,
\begin{align*}
\Delta_A Y + \Delta_U = \underbrace{\matrixenc{\Delta_A & \Delta_U}}_{\equiv \Delta} \underbrace{\matrixenc{Y \\ I}}_{\equiv Y'} = \Delta Y'.
\end{align*}
Therefore, we get,
\begin{align*}
&\frac{1}{2} \norm{(\Delta Y'-\overline{E})D}_F^2 \le \frac{1}{2} \norm{\widehat{E}D}_F^2+ \left(\lambda \norm{\UStar D}_{2,1}-\lambda \norm{\widehat{U} D}_{2,1}\right)\\
\end{align*}
which is equivalent to,
\begin{align*}
\frac{1}{2} \norm{\Delta Y'D}_F^2 &\le \dotprod{\Delta Y'}{\overline{E}D}+ \left(\lambda \norm{\UStar D}_{2,1}-\lambda \norm{\widehat{U} D}_{2,1}\right)\\
& = \dotprod{\Delta_A}{\overline{E}DY^\top}+ \dotprod{\Delta_U}{\overline{E}D}+ \lambda\left( \norm{\UStar D}_{2,1}-\norm{\widehat{U} D}_{2,1}\right).
\end{align*}

Now note that
\begin{align*}
\dotprod{\Delta_U}{\overline{E}D} &= \dotprod{\Delta_U D}{\overline{E}}= \sum_{t=1}^{T-1} \dotprod{D_t(\widehat{u}_t-\uStar_t)}{\eta_{t+1}}\\
& \le \sum_{t=1}^{T-1} \norm{D_t(\widehat{u}_t-\uStar_t)}_2 \norm{\eta_{t+1}}_2 \le \left(\max_{t \in \set{2,\dots,T}} \norm{\eta_t}_2\right)\norm{\Delta_U D} \\
& \le \frac{\lambda}{2}\norm{\Delta_U D}
\end{align*}\
where $\lambda \ge 2 \max_{t \in \set{2,\dots,T}} \norm{\eta_t}_2$. Then we have,
\begin{align*}
&\dotprod{\Delta_U}{\overline{E}D} +\lambda \left(\norm{\UStar D}_{2,1}- \norm{\widehat{U} D}_{2,1}\right) \le \lambda \left(\frac{1}{2}\norm{\Delta_U D}+\norm{\UStar D}_{2,1}- \norm{\widehat{U} D}_{2,1}\right).
\end{align*}

By a decomposition over $\tilcalS$ and $\tilcalSC$, we obtain
\begin{align*}
\norm{\UStar D}_{2,1}-\norm{\widehat{U} D}_{2,1} &=\norm{\UStar D}_{2,1}-\norm{\Delta_UD+ \UStar D}_{2,1}\\
&=\norm{(\UStar D)_{\tilcalS}}_{2,1}-\norm{(\Delta_UD)_{\tilcalS}+ (\UStar D)_{\tilcalS}}_{2,1}-\norm{(\Delta_UD)_{\tilcalSC}}_{2,1},
\end{align*}
which implies,
\begin{align*}
&\lambda \Big(\frac{1}{2} \norm{\Delta_U D}_{2,1}+\norm{\UStar D}_{2,1}-\norm{\hat{U}D}_{2,1}\Big)\\
&=\lambda \Bigg(\frac{1}{2} \norm{\Delta_U D}_{2,1}+\norm{(\UStar D)_{\tilcalS}}_{2,1}-\norm{(\Delta_UD)_{\tilcalS}+ (\UStar D)_{\tilcalS}}_{2,1}-\norm{(\Delta_UD)_{\tilcalSC}}_{2,1}\Bigg)\\
&\le  \lambda \Bigg(\frac{1}{2} \norm{\Delta_U D}_{2,1}+\norm{(\Delta_UD)_{\tilcalS}}_{2,1}-\norm{(\Delta_UD)_{\tilcalSC}}_{2,1}\Bigg)\\
&=  \frac{\lambda}{2} \Bigg(3\norm{(\Delta_UD)_{\tilcalS}}_{2,1}-\norm{(\Delta_UD)_{\tilcalSC}}_{2,1}\Bigg).
\end{align*}
Therefore, if $\lambda \ge 2  \max_{t \in \set{2,\dots,T}} \norm{\eta_t}_2$, then,
\begin{align*}
\frac{1}{2} \norm{\Delta Y' D}_F^2 &\le \dotprod{\Delta_A}{\overline{E}DY^\top}+\frac{\lambda}{2} \Bigg(3\norm{(\Delta_UD)_{\tilcalS}}_{2,1}-\norm{(\Delta_UD)_{\tilcalSC}}_{2,1}\Bigg).
\end{align*}
\end{proof}
\subsection{Proof of Lemma~\ref{lem:bound-lambda}}\label{app-sec:bound-lambda}

\begin{proof}
Using the Lemma~\ref{lem:chi-sq-dev} with $\mu= \left(1 \ 1 \ \dots \ 1\right)^\top$, we get
\begin{align*}
\prob \left(\norm{\eta_t}_2^2 \le n+2\sqrt{vn}+2v\right) \ge 1-e^{-v}.
\end{align*}
Letting $v=\log \left(\frac{T}{\delta}\right)$, and taking the union over $t = 2,\dots, T$ yields
\begin{align*}
\prob \left(\max_{t \in \set{2,\dots,T}}\norm{\eta_t}_2^2 \le n+2\sqrt{n\log \left(\frac{T}{\delta}\right)}+2\log \left(\frac{T}{\delta}\right)\right) 
\ge 1-\delta
\end{align*}
and taking the square root leads to 
\begin{align*}
\prob &\bigg(\max_{t \in \set{2,\dots,T}}\norm{\eta_t}_2 \le \sqrt{n}+\left(4n\log \left(\frac{T}{\delta}\right)\right)^{1/4}+\sqrt{2\log \left(\frac{T}{\delta}\right)}\bigg) \ge 1-\delta.
\end{align*}
\end{proof}

\subsection{Proof of Lemma~\ref{lem:inf-norm-cost}}\label{app-sec:inf-norm-cost}
\begin{proof}
Note that $\overline{E}$ can be written as follows:
\begin{align*}
\overline{E} &= \matrixenc{\eta_2 & \eta_3 & \dots & \eta_T}_{n \times (T-1)}=\matrixenc{g_1^\top \\ \vdots \\ g_d^\top}
\end{align*}
where $g_i \stackrel{\mathclap{\text{i.i.d}}}{\sim} \nrml \left(0,I_{T-1}\right)$. Therefore, 
\begin{align*}
\sup_{\norm{b}_{\infty} \le r, \norm{b}_0 \le \absenc{\tilcalS}} \norm{\overline{E}b}_{\infty} &= \sup_{b \in \calB\left(r,\absenc{\tilcalS}\right)} \max_{i \in [n]} \absenc{g_i^\top b}\\
&= \max_{i \in [n]} \sup_{b \in \calB\left(r,\absenc{\tilcalS}\right)}  \absenc{g_i^\top b}.
\end{align*}
Now consider $\sup_{b \in \calB\left(r,\absenc{\tilcalS}\right)}  \absenc{g^\top b}$ for $g {\sim} \nrml \left(0,I_{T-1}\right)$. We can see that bounding $\max_{i \in [n]} \sup_{b \in \calB\left(r,\absenc{\tilcalS}\right)}  \absenc{g_i^\top b}$ can be achieved by finding an upper bound on $\sup_{b \in \calB\left(r,\absenc{\tilcalS}\right)}  \absenc{g^\top b}$ and then applying union bound over $i \in [n]$. To this end, we show that the supremum function in question is Lipschitz and then use concentration bounds for Lipschitz functions of standard i.i.d Gaussian's to achieve the desired bound.  

\begin{remark}\label{rem:lipschitz-sup-gb}
Let $f(x) = \sup_{b \in \calB\left(r,\absenc{\tilcalS}\right)} \absenc{x^\top b}$. We can see that $f(x)$ is Lipschitz over $\matR$ with respect to $\norm{\cdot}_2$ norm. This can be seen by observing,
\begin{align*}
\absenc{f(x)-f(y)} &\le \left(\sup_{b \in \calB\left(r,\absenc{\tilcalS}\right)} \norm{b}_2\right) \norm{x-y}_2\\
&\le r\absenc{\tilcalS} \norm{x-y}_2
\end{align*}
where the last inequality follows from the fact that $\norm{b}_2 \le \norm{b}_1 \le r\absenc{\tilcalS}$ holds for all $b \in \calB\left(r,\absenc{\tilcalS}\right)$.
\end{remark}

Given Remark~\ref{rem:lipschitz-sup-gb}, we obtain from Theorem~\ref{thm:bounded-func-conc} \cite[from Chapter 3]{Wainwright2019} that for any $q > 0$,
\begin{align*}
&\prob \left(\absenc{\sup_{b \in \calB\left(r,\absenc{\tilcalS}\right)}  \absenc{g^\top b}-\expec\left(\sup_{b \in \calB\left(r,\absenc{\tilcalS}\right)}  \absenc{g^\top b}\right)} \ge q\right) \le 2\exp\left(\frac{-cq^2}{r^2 \absenc{\tilcalS}^2}\right).
\end{align*}
Let $\mu \left(\calB\left(r,\absenc{\tilcalS}\right)\right) \coloneq \expec \left(\sup_{b \in \calB\left(r,\absenc{\tilcalS}\right)}  \absenc{g^\top b}\right)$ where we recall that $\mu \left(\calB\left(r,\absenc{\tilcalS}\right)\right)$ is known as the Gaussian complexity of $\calB\left(r,\absenc{\tilcalS}\right)$ \cite[see Section 7.5.3]{HDPbook}. This can be bounded in terms of the Gaussian width of $\calB\left(r,\absenc{\tilcalS}\right)$ as follows, see \cite[Lemma 7.5.11]{HDPbook}. 
$$\mu \left(\calB\left(r,\absenc{\tilcalS}\right)\right) \lesssim w \left(\calB\left(r,\absenc{\tilcalS}\right)\right)+\norm{y}_2 \text{ for } y \in \calB\left(r,\absenc{\tilcalS}\right).$$ 
Since $0 \in \calB\left(r,\absenc{\tilcalS}\right)$, we have,
\begin{align*}
\mu \left(\calB\left(r,\absenc{\tilcalS}\right)\right) \lesssim w \left(\calB\left(r,\absenc{\tilcalS}\right)\right).
\end{align*}

Therefore, with probability at least $ 1-2\exp \left(\frac{-cq^2}{r^2\absenc{\tilcalS}^2}\right)$, 
\begin{align*}
\sup_{b \in \calB\left(r,\absenc{\tilcalS}\right)}  \absenc{g^\top b} \le c_1 w \left(\calB\left(r,\absenc{\tilcalS}\right)\right)+q
\end{align*}
for some constant $c_1$. Letting $q = c_2 r \absenc{\tilcalS} \sqrt{\log \left(\frac{n}{\delta}\right)}$ for an appropriate constant $c_2$, we get,
\begin{align*}
\prob \bigg(\sup_{b \in \calB\left(r,\absenc{\tilcalS}\right)}  \absenc{g^\top b} &\le c_1 w \left(\calB\left(r,\absenc{\tilcalS}\right)\right)+c_2 r \absenc{\tilcalS} \sqrt{\log \left(\frac{n}{\delta}\right)}\bigg) \ge 1-\frac{\delta}{n}
\end{align*}
and taking the union bound over $i\in [n]$,
\begin{align*}
\prob \bigg(\norm{\overline{E}D_{\tilcalS}Y_{\tilcalS}^\top}_{\infty, \infty} &\le c\bigg( w \left(\calB\left(r,\absenc{\tilcalS}\right)\right)+ r \absenc{\tilcalS} \sqrt{\log \left(\frac{n}{\delta}\right)}\bigg)\bigg) \ge 1-\delta.
\end{align*}
Therefore, $\prob \left(\eventp{1a}\right) \ge 1-\delta$.
\end{proof}

\subsection{Proof of Lemma~\ref{lem:all-unc-select}}\label{app-sec:all-unc-select}
\begin{proof}
Let $\eta = \vect\left(\matrixenc{\eta_1 & \eta_2 & \dots & \eta_T}\right)$ and note that 
$x_t = \widetilde{A}(t) \eta$ where 
$$\widetilde{A}(t) :=\matrixenc{(A^*)^{t-1} & (A^*)^{t-2} & \dots & I_n & 0 & \dots & 0} \in \matR^{n \times nT}.$$ 
For any $i\in [n]$, $x_{ti} = a^\top(t,i) \eta$ where $a^\top(t,i)$ is the $i$-th row of $\widetilde{A}(t)$ and $x_{ti}$ is the $i$-th element of $x_t$. Note that,
\begin{align*}
\expec \left(x^2_{ti}\right) = \norm{a(t,i)}_2^2 \le \norm{\widetilde{A}(t)}_2^2. 
\end{align*}
For any $v=\vect\left(\matrixenc{v_1 & v_2 & \dots & v_T}\right)$ with $\norm{v}_2=1$,
\begin{align*}
\norm{\widetilde{A}(t)v}_2 &= \norm{\sum_{t=1}^T (A^*)^{t-i} v_i}_2 \le \sum_{t=1}^T \norm{(A^*)^{t-i}}_2\le \sum_{t=1}^{\infty} \norm{(A^*)^{t-i}}_2 \le J(A^*),
\end{align*}
and therefore, $\expec \left(x^2_{ti}\right) \le J^2(A^*)$. 

Note that since $\eta$ contains i.i.d standard Gaussian entries, $x_{ti} \sim \nrml\left(0, \norm{a(t,i)}_2^2\right)$, and therefore, $x_{ti} = \norm{a(t,i)}_2 W_{ti}$ where $W_{ti} \sim \nrml(0,1)$. As $W^2_{ti}$ is sub-exponential, we have
\begin{align*}
\prob \left(W^2_{ti} \ge q\right) &\le 2\exp\left(\frac{-q}{c\norm{W^2_{ti}}_{\psi_1}}\right)= 2\exp\left(\frac{-q}{c\norm{W_{ti}}^2_{\psi_2}}\right)= 2\exp\left(\frac{-q}{c'}\right).
\end{align*}
By setting $q=c'\log \left(\frac{Tn}{\delta}\right)$ for all $\delta \in (0,c_1)$ for constants $c'$ and $c_1 < 1$, we get $\prob \left(W^2_{ti} \ge c'\log \left(\frac{Tn}{\delta}\right)\right) \le \frac{\delta}{Tn}$ and therefore, $\prob \left(x^2_{ti} \le c'J^2(A^*)\log \left(\frac{Tn}{\delta}\right)\right) \ge 1-\frac{\delta}{Tn}$. Taking a union bound over $t \in [T]$ and $i \in [n]$ yields
\begin{align*}
\prob \left(\forall t \in [T] :\norm{x_{t}}_{\infty} \le cJ(A^*) \sqrt{\log \left(\frac{Tn}{\delta}\right)}\right) \ge 1-\delta.
\end{align*}
Taking $r = cJ(A^*) \sqrt{\log \left(\frac{Tn}{\delta}\right)}$, this clearly implies $\prob\left(\eventp{1b}\right) \ge 1-\delta$.
\end{proof}

\subsection{Proof of Lemma~\ref{lem:inf-norm-comple}}\label{app-sec:inf-norm-comple}
\begin{proof}
Let $b_t = \indic\set{t \in \tilcalSC}$ and let $M_T = \overline{E}Y_{\tilcalSC}^\top$. We can see that,
\begin{align*}
M_T = \sum_{t\in \tilcalSC} \eta_{t+1} x_t^\top = \sum_{t \in [T-1]} b_t \eta_{t+1} x_t^\top. 
\end{align*}
Note that for any $i,j \in [n]$, we have 
$$M_{T, ij} = e_i^\top M_T e_j = \sum_{t \in [T-1]} b_t \eta_{t+1,i} x_{t,j}^\top.$$ 
Let $\calF_t = \sigma(\eta_1,\dots,\eta_t)$ denote the sigma-algebra at any time $t \geq 1$. Then $b_t x_t$ is $\calF_t$ measurable while $\eta_{t+1} \vert \calF_t$ is sub-Gaussian, for all $t \geq 1$. 
We can hence bound $\absenc{M_{T, ij}}$ using the tail bound for self normalized martingales from Theorem~\ref{thm:self-nrml-mart}. 

To this end, let us first denote $\overline{v}=v+\sum_{t\in [T-1]} b_t x^2_{tj}$ (for $v>0$). Then, there exist constants $c_1 \geq 1$ and $c_2 < 1$ such that
\begin{align*}
\prob \left(\overline{v}^{-1}M^2_{ij} \le c_1 \log \left(\frac{\overline{v}^{1/2} {v}^{-1/2}}{\delta}\right)\right) \ge 1-\delta
\end{align*}
for $\delta \in (0,c_2)$. Let $\eventp{1c}'$ denote the above event
and note that on $\eventp{1b}$,
\begin{align*}
\sum_{t\in [T-1]} b_t x^2_{tj} \le \sum_{t\in [T-1]} x^2_{tj} \le Tr^2.  
\end{align*}
Choosing $v=Tr^2$, this implies $\overline{v} \le v+Tr^2  \le 2Tr^2$.

Hence on $\eventp{1b} \cap \eventp{1c}'$, we have shown that
\begin{align*}
&\frac{1}{2Tr^2} \absenc{M_{ij}}^2 \le c_1 \log \left(\frac{2}{\delta}\right) \iff \absenc{M_{ij}}^2 \le 2c_1 r^2T \log \left(\frac{2}{\delta}\right).
\end{align*}
By suitably rescaling $\delta$ and taking a union bound over $i,j \in [n]$, we then obtain
\begin{align*}
\prob \left(\forall i,j \in [n]: \absenc{M_{ij}} \le c_1 r \sqrt{T \log \left(\frac{n}{\delta}\right)}\right) \ge 1-\delta.
\end{align*}
Since $\eventp{1b} \cap \eventp{1c}' \subseteq \eventp{1c}$, on $\eventp{1b} \cap \eventp{1c}'$, we obtain the statement of the lemma after suitably rescaling $\delta$.
\end{proof}

\subsection{Proof of Lemma~\ref{lem:lwr-bnd-deltaA}}\label{app-sec:lwr-bnd-deltaA}
\begin{proof}
Consider $\norm{\Delta_AYD}_F^2$. Then,
\begin{align*}
\norm{\Delta_AYD}_F^2 & = \norm{\Delta_AYD_{\tilcalS}}_F^2+\norm{\Delta_AYD_{\tilcalSC}}_F^2\\
&\ge \norm{\Delta_AYD_{\tilcalSC}}_F^2\\
&= \norm{\Delta_AXD_{\tilcalSC}}_F^2 \tag{since $YD_{\tilcalSC}=XD_{\tilcalSC}$} \\
&\ge \norm{\Delta_A}_F^2 \inf_{H \in \widetilde{C}_{\calK}} \norm{HXD_{\tilcalSC}}_F^2,
\end{align*}
where the last inequality comes from the fact that $\frac{\Delta_A}{\norm{\Delta_A}_F} \in \calC_{\calK}$.

Note that $\vect(HXD_{\tilcalSC}) = (I_{T-1} \otimes H) \vect(XD_{\tilcalSC}) = (I_{T-1} \otimes H)(D_{\tilcalSC} \otimes I_n) \vect(X) = (D_{\tilcalSC} \otimes H) \vect(X)$. 

Let $x = \vect(X)$,  $\eta = \vect\left(\matrixenc{\eta_1 & \eta_2 & \dots & \eta_{T-1} }\right)$ and denote $\Gamma \in \matR^{n(T-1) \times n(T-1)}$ as,
\[
\Gamma = 
\matrixenc{
I & 0 & 0 & \dots & 0 \\
A^* & I & 0 & \dots & 0 \\
{(A^*)}^2 & A^* & I & \dots & 0 \\
\vdots & \vdots & \vdots & \dots & \vdots \\
{(A^*)}^{T-2} & {(A^*)}^{T-3} & {(A^*)}^{T-4} & \dots & I  \\ 
}.
\]
We can see that $x = \Gamma \eta$, and hence, 
$\vect(HXD_{\tilcalSC}) = (D_{\tilcalSC} \otimes H) \Gamma \eta$. 

Consider the set $\calM = \set{(D_{\tilcalSC} \otimes H) \Gamma: H\in \widetilde{C}_{\calK}}$. 
We then have that 
\begin{align*}
\inf_{H \in \widetilde{C}_{\calK}} \norm{HXD_{\tilcalSC}}_F^2 &= \inf_{H \in \widetilde{C}_{\calK}} \norm{\vect\left(HXD_{\tilcalSC}\right)}_2^2\\
&=\inf_{M \in \calM} \norm{M\eta}_2^2.
\end{align*}
Applying Theorem~\ref{thm:psd-sup} (Theorem 3.1 of \cite{krahmer14}) to $\calM$ yields
\begin{align}
&\prob \left(\sup_{M \in \calM} \absenc{\norm{M\eta}_2^2-\expec\left(\norm{M\eta}_2^2\right)} \ge c_1 P+q\right) \le 2\exp \left(-c_2\min\set{\frac{q^2}{F^2},\frac{q}{L}}\right). \label{eq:M-bnd-krah} 
\end{align}
where
\begin{align*}
P &= \gammafn{2} \left(\calM,\norm{\cdot}_2\right) \left(\gammafn{2} \left(\calM,\norm{\cdot}_2\right)+d_F(\calM)\right)+d_F(\calM)d_2(\calM),\\
F &= d_2(\calM)\left(\gammafn{2} \left(\calM,\norm{\cdot}_2\right)+d_F(\calM)\right), \\
\text{ and } L &= d^2_2(\calM)
\end{align*}

We start by observing that for any $M\in \calM$, $\norm{M}_F^2 \ge \absenc{\tilcalSC} \norm{H}_F^2 = \absenc{\tilcalSC}$. It was shown in \cite{Jedra20}  that $\norm{\Gamma}_2 \le J(A^*)$. Therefore,
\begin{align*}
d_2(\calM) &=\sup_{H \in \widetilde{C}_{\calK}} \norm{(D_{\tilcalSC} \otimes H) \Gamma}_2\le \sup_{H \in \widetilde{C}_{\calK}}\norm{D_{\tilcalSC}}_2 \norm{H}_2 \norm{\Gamma}_2 \le J(A^*) \text{ and,}\\
d_F(\calM) &=\sup_{H \in \widetilde{C}_{\calK}} \norm{(D_{\tilcalSC} \otimes H) \Gamma}_F \le \norm{\Gamma}_2 \sup_{H \in \widetilde{C}_{\calK}} \norm{D_{\tilcalSC} \otimes H}_F = J(A^*) \sup_{H \in \widetilde{C}_{\calK}} \norm{H}_F \norm{D_{\tilcalSC}}_F \le J(A^*)\sqrt{T}. 
\end{align*}

We next show that the map $(\widetilde{C}_{\calK}, \norm{\cdot}_2) \mapsto (\calM,\norm{\cdot}_2)$ is Lipschitz. For any $H_1, H_2 \in \widetilde{C}_{\calK}$, 
\begin{align*}
\norm{(D_{\tilcalSC} \otimes H_1) \Gamma-(D_{\tilcalSC} \otimes H_2) \Gamma}_2 = \norm{(D_{\tilcalSC} \otimes (H_1-H_2)) \Gamma}_2  \le \norm{D_{\tilcalSC}}_2 \norm{\Gamma}_2 \norm{H_1-H_2}_2 \le J(A^*) \norm{H_1-H_2}_2.
\end{align*}
Using Property \ref{itm:tal-prp3} recalled earlier in Section \ref{sec:prelim-tools}, we then obtain,
\begin{align*}
\gamma_2\left(\calM,\norm{\cdot}_2\right) \le cJ(A^*) \gamma_2\left(\widetilde{C}_{\calK},\norm{\cdot}_2\right).
\end{align*}

Let us now consider the following respective upper bounds for $P, Q, L$.
\begin{align*}
\overline{P} &= c_1J^2(A^*) \gamma_2\left(\widetilde{C}_{\calK},\norm{\cdot}_2\right)\bigg(\gamma_2\left(\widetilde{C}_{\calK},\norm{\cdot}_2\right)+\sqrt{T}\bigg)+J^2(A^*)\sqrt{T}, \\
\overline{F} &= c_1 J^2(A^*) \left(\gamma_2\left(\widetilde{C}_{\calK},\norm{\cdot}_2\right)+\sqrt{T}\right), \\
\text{ and, } \ \overline{L} &= J^2(A^*), \\
\end{align*}
for sufficiently large $c_1>1$. Then from \eqref{eq:M-bnd-krah}, 
\begin{align*}
&\prob \left(\inf_{M\in \calM} \norm{M\eta}_2^2 \ge \absenc{\tilcalSC}-c_2 \overline{P}-q\right) \ge 1-2\exp\left(-c_3 \min \set{\frac{q^2}{\overline{F}^2},\frac{q}{\overline{L}}}\right).
\end{align*}
Choose $q = \overline{F}\sqrt{m}$ for $m \ge 1$. Since $\overline{F} \ge \overline{L}$, we have $\min \set{\frac{q^2}{\overline{F}^2},\frac{q}{\overline{L}}} = \sqrt{m}$, hence,
\begin{align*}
&\prob \left(\inf_{M\in \calM} \norm{M\eta}_2^2 \ge \absenc{\tilcalSC}-c_2 \overline{P}-\overline{F}\sqrt{m}\right) \ge 1-2\exp\left(-c_3 \sqrt{m}\right).
\end{align*}
Using Property \ref{itm:tal-prp1} (since $\norm{\cdot}_2 \leq \norm{\cdot}_F$) and recalling the inequality in \eqref{eq:gamma_gauss_width}, we have
\begin{align} \label{eq:proof_tmp1}
\gamma_2\left(\widetilde{C}_{\calK},\norm{\cdot}_2\right) \lesssim \gamma_2\left(\widetilde{C}_{\calK},\norm{\cdot}_F\right)\asymp w\left(\widetilde{C}_{\calK}\right) = w\left(\widetilde{C}_{\calK,\vect}\right),
\end{align}
where $\widetilde{C}_{\calK,\vect}$ is defined as
$$\widetilde{C}_{\calK,\vect} = \set{\vect(H) \in \matR^{n^2}: \norm{(H)_{\calK^c}}_{1,1} \le \norm{(H)_{\calK}}_{1,1}} \cap \set{x \in \matR^{n^2}: \norm{x}_2=1}.$$ 
The last equality in \eqref{eq:proof_tmp1} follows from the fact that $(\widetilde{C}_{\calK},\norm{\cdot}_F) \rightarrow (\widetilde{C}_{\calK,\vect},\norm{\cdot}_2)$ is a bijective map. 
It was shown in \cite[Proposition 9.33]{Foucart2013os} that $w(\widetilde{C}_{\calK,\vect}) \lesssim \sqrt{k \log \left(\frac{n^2}{k}\right)}$ which implies  $\gamma_2\left(\widetilde{C}_{\calK},\norm{\cdot}_2\right) \lesssim \sqrt{k \log \left(\frac{n^2}{k}\right)}$.

Therefore choosing $m=c\log^2 \left(\frac{1}{\delta}\right)$, we have that
\begin{align*}
c_2 \overline{P}+\overline{F}\sqrt{m} & \le c' J^2(A^*) \Bigg(k\log n+\sqrt{kT\log n}+\sqrt{mT}+\sqrt{mkT\log n}\Bigg)\\
&\le \overline{c} J^2(A^*) \Bigg(k \log \left(\frac{n^2}{k}\right)+\sqrt{kT \log \left(\frac{n^2}{k}\right)}+\log \left(\frac{1}{\delta}\right)\left(\sqrt{T}+\sqrt{kT\log \left(\frac{n^2}{k}\right)}\right)\Bigg)\\
&\le \overline{c} J^2(A^*) \log \left(\frac{1}{\delta}\right) \Bigg(k\log \left(\frac{n^2}{k}\right)+\sqrt{kT\log \left(\frac{n^2}{k}\right)}\Bigg).
\end{align*}
Let $T \ge c_1 J^4(A^*) \log^2 \left(\frac{1}{\delta}\right) k \log \left(\frac{n^2}{k}\right)+4s$ for a suitably large constant $c_1$. Then we have that
$$\frac{\absenc{\tilcalSC}}{2} \ge \frac{T-2s-1v}{2} \ge \overline{c} J^2(A^*) \log \left(\frac{1}{\delta}\right) \Bigg(k\log \left(\frac{n^2}{k}\right)+\sqrt{kT\log \left(\frac{n^2}{k}\right)}\Bigg).$$ 

Therefore with probability $\ge 1-\delta$, it holds that $\inf_{M \in \calM} \norm{M\eta}_2^2 \ge \frac{\absenc{\tilcalSC}}{2}$. Note that since $T \ge 4s$, this implies $\absenc{\tilcalSC} \ge T-2s \ge T/2$ which completes the proof.
\end{proof}

\subsection{Proof of Lemma~\ref{lem:bnd-cross-trm}}\label{app-sec:bnd-cross-trm}
\begin{proof}
Note that,
\begin{align*}
\absenc{\dotprod{\Delta_A Y D}{\Delta_U D}} &\le \absenc{\dotprod{\Delta_A Y D}{(\Delta_U D)_{\tilcalS}}}+\absenc{\dotprod{\Delta_A Y D}{(\Delta_U D)_{\tilcalSC}}}\\
&\le \absenc{\dotprod{\Delta_A Y_{\tilcalS} D_{\tilcalS}}{(\Delta_U D)_{\tilcalS}}}+\absenc{\dotprod{\Delta_A Y_{\tilcalSC} D_{\tilcalSC}}{(\Delta_U D)_{\tilcalSC}}} \\
& \le \norm{(\Delta_U D)_{\tilcalS}}_{2,1} \norm{\Delta_A Y_{\tilcalS} D_{\tilcalS}}_{2,\infty}+\norm{(\Delta_U D)_{\tilcalSC}}_{2,1} \norm{\Delta_A Y_{\tilcalSC} D_{\tilcalSC}}_{2,\infty}.
\end{align*}
Recall that $\norm{(\Delta_U D)_{\tilcalSC}}_{2,1} \le 3\norm{(\Delta_U D)_{\tilcalS}}_{2,1}+\frac{\overline{V}}{\overline{\lambda}} \norm{\Delta_A}_F$. Therefore,
\begin{align*}
\absenc{\dotprod{\Delta_A Y D}{\Delta_U D}} & \le \norm{(\Delta_U D)_{\tilcalS}}_{2,1} \norm{\Delta_A Y_{\tilcalS} D_{\tilcalS}}_{2,\infty}+\left(3\norm{(\Delta_U D)_{\tilcalS}}_{2,1}+\frac{\overline{V}}{\overline{\lambda}} \norm{\Delta_A}_F\right) \norm{\Delta_A Y_{\tilcalSC} D_{\tilcalSC}}_{2,\infty}.
\end{align*}
In order to control $\absenc{\dotprod{\Delta_A Y D}{\Delta_U D}}$, we need to control $\norm{\Delta_A Y_{\tilcalSC} D_{\tilcalSC}}_{2,\infty}$ and $\norm{\Delta_A Y_{\tilcalS} D_{\tilcalS}}_{2,\infty}$. 

We can first consider $\norm{\Delta_A Y_{\tilcalS} D_{\tilcalS}}_{2,\infty}$. Note that,
\begin{align*}
\norm{\Delta_A Y_{\tilcalS} D_{\tilcalS}}_{2,\infty} = \max_{t} \norm{\Delta_A (Y_{\tilcalS} D_{\tilcalS})_{:,t}}_{2} \le \norm{\Delta_A}_{1,1} \max_t \norm{(Y_{\tilcalS} D_{\tilcalS})_{:,t}}_{\infty} \le r\norm{\Delta_A}_{1,1} \le 2r\sqrt{k}\norm{\Delta_A}_{F}.
\end{align*}
where $(Y_{\tilcalS} D_{\tilcalS})_{:,t}$ represents the $t$-th column of $Y_{\tilcalS} D_{\tilcalS}$. 

Next, let us take $\norm{\Delta_A Y_{\tilcalSC} D_{\tilcalSC}}_{2,\infty}$. Observe that,
\begin{align*}
\norm{\Delta_A Y_{\tilcalSC} D_{\tilcalSC}}_{2,\infty} &= \max_{t \in [T-1]} \norm{\Delta_A (Y_{\tilcalSC} D_{\tilcalSC})_{:,t}}_{2}\\
&= \max_{t \in [T-1]} \norm{\Delta_A (X_{\tilcalSC} D_{\tilcalSC})_{:,t}}_{2}\\
&\le \max_{t \in [T]} \norm{\Delta_A X_{:,t}}_{2}\\
&= \max_{t \in [T]} \norm{\Delta_A x_t}_2\\
&=\norm{\Delta_A}_F \max_{t \in [T]} \norm{\frac{\Delta_A}{\norm{\Delta_A}_F} x_t}_2\\
& \le \norm{\Delta_A}_F \max_{t \in [T]} \sup_{H \in \widetilde{C}_{\calK}} \norm{H x_t}_2\\
& \le \norm{\Delta_A}_F \max_{t \in [T]} \sup_{H \in \widetilde{C}_{\calK}} \norm{H \widetilde{A}(t)\eta}_2.
\end{align*}
where $(Y_{\tilcalSC} D_{\tilcalSC})_{:,t}$ and $X_{:,t}$ represent $t$-th columns of $Y_{\tilcalSC} D_{\tilcalSC}$ and $X$.

Let $t \in [T]$ be fixed and let $\calM_t = \set{H\widetilde{A}(t): H \in \widetilde{C}_{\calK}}$. Then, we can see that $$\sup_{H \in \widetilde{C}_{\calK}} \norm{H \widetilde{A}(t)\eta}_2 = \sup_{M \in \calM_t} \norm{M\eta}_2.$$ Note that for any $M \in \calM_t$ (where $M=H\widetilde{A}(t)$ for some $H \in \widetilde{C}_{\calK}$), 
\begin{align*}
\expec\left(\norm{M\eta}^2_2\right) = \norm{M}^2_F = \norm{H\widetilde{A}(t)}^2_F \le \norm{\widetilde{A}(t)}^2_2 \norm{H}^2_F \le J^2(A^*).
\end{align*}
Also note that,
\begin{align*}
d_2(\calM_t) &= \sup_{M \in \calM_t} \norm{M}_2 = \sup_{H \in \widetilde{C}_{\calK}} \norm{H\widetilde{A}(t)}_2 \le \norm{\widetilde{A}(t)}_2 \sup_{H \in \widetilde{C}_{\calK}} \norm{H}_2 \le J(A^*) \\
d_F(\calM_t) &= \sup_{M \in \calM_t} \norm{M}_F = \sup_{H \in \widetilde{C}_{\calK}} \norm{H\widetilde{A}(t)}_F \le \norm{\widetilde{A}(t)}_2 \sup_{H \in \widetilde{C}_{\calK}} \norm{H}_F \le J(A^*). 
\end{align*}
Furthermore, we can also see that the mapping $(\widetilde{C}_{\calK},\norm{\cdot}_2) \rightarrow (\calM_t,\norm{\cdot}_2)$ is $J(A^*)$-Lipschitz. To see this, we can consider $H_1, H_2 \in \widetilde{C}_{\calK}$. Then,
\begin{align*}
\norm{H_1 \widetilde{A}(t)-H_2 \widetilde{A}(t)}_2 = \norm{(H_1-H_2) \widetilde{A}(t)}_2 \le \norm{\widetilde{A}(t)}_2\norm{H_1-H_2}_2 \le J(A^*) \norm{H_1-H_2}_2.
\end{align*}
Therefore, we can see that $\gammafn{2} \left(\calM_t,\norm{\cdot}_2\right) \le c J(A^*) \gammafn{2} \left(\widetilde{C}_{\calK},\norm{\cdot}_2\right)$ for a constant $c$.

We can now apply Theorem~\ref{thm:psd-sup} (Theorem 3.1 of \cite{krahmer14}) to $\calM$.
\begin{align}
&\prob \left(\sup_{M \in \calM} \absenc{\norm{M\eta}_2^2-\expec\left(\norm{M\eta}_2^2\right)} \ge c_1 P+q\right) \le 2\exp \left(-c_2\min\set{\frac{q^2}{F^2},\frac{q}{L}}\right). \label{eq:Mt-bnd-krah} 
\end{align}
where
\begin{align*}
P &= \gammafn{2} \left(\calM,\norm{\cdot}_2\right) \left(\gammafn{2} \left(\calM,\norm{\cdot}_2\right)+d_F(\calM)\right)+d_F(\calM)d_2(\calM),\\
F &= d_2(\calM)\left(\gammafn{2} \left(\calM,\norm{\cdot}_2\right)+d_F(\calM)\right) \text{ and, } L = d^2_2(\calM).
\end{align*}
Let us denote,
\begin{align*}
\overline{P}_2 &= c_1J^2(A^*) \gamma_2\left(\widetilde{C}_{\calK},\norm{\cdot}_2\right)\bigg(\gamma_2\left(\widetilde{C}_{\calK},\norm{\cdot}_2\right)+1\bigg)+J^2(A^*)\\
\overline{F}_2 &= c_1 J^2(A^*) \left(\gamma_2\left(\widetilde{C}_{\calK},\norm{\cdot}_2\right)+1\right) \text{ and, } \overline{L}_2 = J^2(A^*)\\
\end{align*}
for sufficiently large $c_1>1$. Then from \eqref{eq:Mt-bnd-krah}, 
\begin{align*}
&\prob \left(\sup_{M\in \calM_t} \norm{M\eta}_2^2 \le J^2(A^*)+c_2 \overline{P}_2+q\right) \ge 1-2\exp\left(-c_3 \min \set{\frac{q^2}{\overline{F}_2^2},\frac{q}{\overline{L}_2}}\right).
\end{align*}
Let $q = \overline{F}_2\sqrt{m}$ for $m \ge 1$. Since $\overline{F}_2 \ge \overline{L}_2$, we have $\min \set{\frac{q^2}{\overline{F}_2^2},\frac{q}{\overline{L}_2}} = \sqrt{m}$. Therefore, 
\begin{align*}
&\prob \left(\sup_{M\in \calM_t} \norm{M\eta}_2^2 \le J(A^*)+c_2 \overline{P}_2+\overline{F}_2\sqrt{m}\right) \ge 1-2\exp\left(-c_3 \sqrt{m}\right).
\end{align*}
Recall from the proof of Lemma \ref{lem:lwr-bnd-deltaA} that 
\begin{align*}
\gamma_2\left(\widetilde{C}_{\calK},\norm{\cdot}_2\right) \lesssim \gamma_2\left(\widetilde{C}_{\calK},\norm{\cdot}_F\right)\asymp w\left(\widetilde{C}_{\calK}\right) \lesssim \sqrt{k \log \left(\frac{n^2}{k}\right)}.
\end{align*}

Therefore letting $m=c\log^2 \left(\frac{1}{\delta}\right)$,
\begin{align*}
c_2 \overline{P}_2+\overline{F}_2\sqrt{m} & \le c' J^2(A^*) \Bigg(k\log \left(\frac{n^2}{k}\right)+\sqrt{k\log \left(\frac{n^2}{k}\right)}+\sqrt{m}+\sqrt{mk\log \left(\frac{n^2}{k}\right)}+1\Bigg)\\
&\le \overline{c} J^2(A^*) \Bigg(k\log \left(\frac{n^2}{k}\right)+\log \left(\frac{1}{\delta}\right)\sqrt{k\log \left(\frac{n^2}{k}\right)}\Bigg)\\
&\le \overline{c} J^2(A^*) k \log \left(\frac{n^2}{k}\right) \log \left(\frac{1}{\delta}\right) .
\end{align*}
 Therefore with probability at least $1-\delta$, $$\sup_{M \in \calM_t} \norm{M\eta}_2^2 \le \overline{c} J^2(A^*) k \log \left(\frac{n^2}{k}\right) \log \left(\frac{1}{\delta}\right)$$.  
 
 Note that this is for a fixed $t$. Taking union bound over $t \in [T]$ with an appropriate rescaling of $\delta$, we have that with probability at least $1-\delta$,
 \[
 \norm{\Delta_A Y_{\tilcalSC} D_{\tilcalSC}}_{2,\infty} \le c J(A^*) \sqrt{k \log \left(\frac{n^2}{k}\right) \log \left(\frac{T}{\delta}\right)} \norm{\Delta_A}_F.
  \]
  
Let $R=J(A^*) \sqrt{k \log \left(\frac{n^2}{k}\right) \log \left(\frac{T}{\delta}\right)}$. Combining the bounds on $\norm{\Delta_A Y_{\tilcalSC} D_{\tilcalSC}}_{2,\infty}$ and $\norm{\Delta_A Y_{\tilcalS} D_{\tilcalS}}_{2,\infty}$, we get,
 \begin{align*}
\absenc{\dotprod{\Delta_A Y D}{\Delta_U D}} & \le \norm{(\Delta_U D)_{\tilcalS}}_{2,1} 2r\sqrt{k} \norm{\Delta_A}_F+\left(3\norm{(\Delta_U D)_{\tilcalS}}_{2,1}+\frac{\overline{V}}{\overline{\lambda}} \norm{\Delta_A}_F\right) c R \norm{\Delta_A}_F \\
& \le c_1 \frac{\overline{V}}{\overline{\lambda}} R \norm{\Delta_A}_F^2+c_2 r R\sqrt{k}\norm{\Delta_A}_F \norm{(\Delta_U D)_{\tilcalS}}_{2,1}\\
& \le \widetilde{c}_1 \frac{\overline{V}}{\overline{\lambda}} R \norm{\Delta_A}_F^2+\widetilde{c}_2 r R\sqrt{k \absenc{\tilcalS}}\norm{\Delta_A}_F \norm{(\Delta_U D)_{\tilcalS}}_{F}.
\end{align*}
\end{proof}


\section{Detailed overview of related work}
\subsection{Robust regression under adversarial contamination}
Consider the supervised learning problem where there exists a distribution $\calD$ over $\calX \times \calY$ where the given $(x_i, y_i)_{i=1}^m \stackrel{i.i.d}{\sim} \calD$, the goal is to learn the relationship between $\calY$ and $\calX$. Specifically, let $y_i = x_i^\top \beta^*+\varepsilon_i$, where $\varepsilon_i$ is the noise and $x_i \in \matR^{d}$. Given $(x_i, y_i)_{i=1}^m$, estimating $\beta^*$ is a fundamental problem in machine learning and typically can be solved efficiently using methods such as ordinary least squares when the observations follow from $\calD$. However, the problem becomes more challenging when a subset of $(x_i, y_i)_{i=1}^m$ is contaminated by an adversary. Estimating $\beta^*$ in when a subset of $(x_i, y_i)_{i=1}^m$ is contaminated is known as robust regression under (adversarial) contamination. Robust regression under adversarial contamination in the i.i.d. setting has been studied under several formal contamination models.
\begin{enumerate}
\item {\bf Inlier-Outlier ($\mathcal{I}\cup \mathcal{S}$) model}:  the sample indices are partitioned as $\mathcal{I} \cup \mathcal{S}$ with $|\mathcal{S}| \le \epsilon m$. The adversary selects the outlier set (of indices) $\mathcal{S}$ oblivious to the data. Once $\calS$ is selected, the adversary can potentially observe the entire dataset (including the data in $\calI$) and decide on the corrupted values for the entries in $\calS$.  

\item {\bf Huber's contamination model}: each observation is drawn independently from a mixture distribution, hence with probability $1-\epsilon$ from $\calD$, and with probability $\epsilon$ from an arbitrary distribution.

\item {\bf Strong contamination model}:  an adversary observes the entire dataset and then selects up to an $\epsilon$-fraction of samples to corrupt arbitrarily. The adversary may adaptively choose corrupted samples after observing the data thus affecting the distributional properties of the overall dataset.
\end{enumerate}
We cover a representative sample of key results for each contamination model, though this list is not exhaustive. 
\paragraph{Inlier-Outlier Model ($\mathcal{I}\cup \mathcal{S}$).}
\begin{itemize}
\item \cite{gaiffas2022} study robust supervised learning with coordinate gradient descent under Lipschitz smooth losses and finite moment assumptions on covariates and labels. They compute robust gradient estimates using Median-of-Means (MOM) and Trimmed Mean estimators, and provide bounds on the generalization error. 
\item \cite{merad2023robust} consider high-dimensional linear models with an $\epsilon$ fraction of corrupted samples, under sparsity assumptions on $\beta^*$. Under finite $\ell_4$-norm assumptions on covariates and finite $\ell_2$-norm assumptions on labels, they show that robust gradient estimators enable consistent parameter recovery. For smooth losses, they apply multistage mirror descent; for non-smooth Lipschitz losses satisfying pseudo-linear minorization, they use multistage dual averaging. They further establish that the Trimmed Mean estimator is a robust gradient estimator for sparse regression under $\epsilon$-contamination.
  
\item \cite{dalalyan2019outlier} analyze sparse regression with adversarially corrupted labels and show that under restricted eigenvalue and incoherence conditions on the covariance matrix of the clean data, the $\ell_1$-penalized Huber M-estimator achieves $\tilde{O}(\sqrt{\frac{k}{m}}+\epsilon)$ when $m$ points are observed and $\beta^*$ is $k$-sparse. 
\end{itemize}

\paragraph{Huber's Contamination Model.}

\begin{itemize}
\item \cite{diakonikolas2019effi} show that for Huber's contamination model with $\epsilon$-contamination, no algorithm with statistical query access to the $\epsilon$-contaminated distribution, can approximate $\beta^*$ to a $\sqrt{\epsilon}$ error without using needing samples of order $\tilde{O}(d^2)$.
\item \cite{diakonikolas2023near} explore linear regression when $\calD = \nrml(0,I_d)$ and give an algorithm with sample complexity $\tilde{O}(d/\epsilon^2)$ with $\ell_2$ error of the order $O(\epsilon)$ in almost linear time by reducing the problem to a robust mean estimation problem. 
\item \cite{zhang2023robust} propose M-estimators for robust regression under Huber's contamination model. Assuming Lipschitz conditions and locally quadratic nature on the loss function, and specific assumptions on the covariates and the noise, they show that M-estimators can achieve $\tilde{O}(\sqrt{d/m+\epsilon})$.

\item \cite{diakonikolas2024robust} extend the ideas from \cite{diakonikolas2023near} to the setting when $\beta^*$ is $k$-sparse and $\norm{\beta^*}_2$ is bounded and show that with $\tilde{O}\left(\frac{k^2\log d}{\epsilon^2}\right)$ many samples you can estimate $\beta^*$ up to $O(\epsilon)$ error in polynomial time.
\end{itemize}

\paragraph{Strong Contamination Model.}
\begin{itemize}

\item \cite{liu2019highM} introduce the Robust Descent Condition (RDC) and shows that gradient estimators satisfying RDC yield statistical guarantees for recovery of $\beta^*$ when combined with Iterative Hard Thresholding. For linear regression with a $k$-sparse parameter, with $\tilde{O}(1/\sqrt{k} \log d)$-fraction of the data contaminated, they show that a version of trimmed gradient estimator recover $\beta^*$ upto $\tilde{O}(\sqrt{k} (\epsilon\log (m d)+ \sqrt{\log d/m}$ (in $\ell_2$ error) when $m = \Omega(k\log d)$. 

\item \cite{liu2019high} propose a Robust Sparse Gradient Estimator (RGSE), computed using a convex relaxation of robust PCA followed by hard thresholding and filtering. For sparse regression where the covariance matrix is identity they show how to recover $\beta^*$ with minimax optimal rate of $\tilde{O}(\epsilon)$. When the covariance matrix is sparse but unknown, they show how to achieve $\tilde{O}(\epsilon)$. 

\item \cite{sasai2022robust} propose a weighted $\ell_1$-penalized Huber M-estimator under bounded fourth-moment covariates and i.i.d. noise, where weights are computed via convex relaxation of robust PCA followed by truncation. The extension in \cite{sasai2024outlier} consider sub-Gaussian covariates satisfying restricted eigenvalue conditions. 

\item \cite{pensia2021robust} consider arbitrary contamination of at most $\epsilon n$ samples under mean-zero, identity-covariance, $(4,2)$-hypercontractive covariates and independent mean-zero noise, and analyzes Huber regression, Least Trimmed Squares (LTS), and Least Absolute Deviation (LAD) estimators. Under stability conditions, they show that Huber regression achieves $\tilde{O}(\sqrt{d/m}+\sqrt{\epsilon})$. They also establish $\ell_2$ error guarantees for LTS and LAD under subset strong convexity and subset strong smoothness conditions. 
\item \cite{bhatia2015robust} establish recovery guarantees for Least Trimmed Squares in high-dimensional settings under adversarial corruption. They provide a thresholding based gradient descent algorithm that converges to an error of $\le \epsilon$ in $\tilde{O}\left(\log \left(\frac{1}{\epsilon \sqrt{m}}\right)\right)$ iterations when the corruption is bounded and the normalized covariates satisfy subset strong convexity and subset strong smoothness \footnote{The follow-up work \citep{Bhatia2017Consistent} extends this to establish consistency guarantees albeit under an oblivious adversary.}. 
\item The work of \cite{karmalkar2019} analyzes Least Absolute Deviation regression under adversarial contamination. They establish $\ell_2$ error guarantees when $\beta^*$ is $k$-sparse and the covariates come from a Gaussian distribution. They also establish error guarantees for the setting when $\beta^*$ is dense as well. 
\item \cite{chen13Robust} consider the setting where both covariates and labels may be corrupted, including distributed corruption across entries of individual samples. They propose algorithms based on thresholding, the Dantzig selector, Lasso regression, and thresholding regression, and show that the three methods modified with a trimmed variant of inner product, recover $\beta^*$ to small error. 
\item \cite{sasai2020robust} study robust sparse regression using an $\ell_1$-penalized formulation under adversarial contamination and uses specific properties of the Huber function to establish recovery guarantees similar to that of~\cite{dalalyan2019outlier}. They show that weighted-$\ell_2$ error (weighted with the covariance matrix) is of the order of $\tilde{O}(\sqrt{k/m}+\epsilon)$ when the covariance matrix satisfies restricted eignvalue conditions.
\item \cite{cherapanamjeri2020} develop robust estimators using semidefinite programming to compute gradient estimates under hypercontractive or sub-Gaussian covariates. Under sub-Gaussian assumptions, they achieve  sample complexity and error up to logarithmic factors of the mini-max optimal and optimal error and sample complexity optimal up to logarithmic factors, under hypercontractivity assumptions.  
\item \cite{RobustLin2025} propose a bi-convex optimization formulation, proves the existence of an integral KKT point that is locally optimal, and provides an algorithm with provable error guarantees. For $k$-sparse regression problem, they establish necessary conditions for exact support recovery, and $\tilde{O}(\sqrt{k}\epsilon)$ $\ell_2$-error.

\item \cite{bakshi2020robust} study linear regression under $(c,k)$-hypercontractive (for some parameters $c,k$) distributions with an $\epsilon$-fraction of adversarial corruption. They encode the problem as a polynomial system and apply a sum-of-squares relaxation to obtain computationally efficient estimators that achieve information-theoretically optimal convergence rates establishing tight upper and lower bounds. 
\item \cite{minsker2023robust} consider high-dimensional linear regression in which an $\epsilon$-fraction of the measurements are adversarially corrupted, with the adversary having full knowledge of the data and the underlying distribution. The covariates are assumed to be sub-Gaussian, while the additive noise is allowed to be dense and heavy-tailed. They establish minimax lower bounds that depend on both the contamination level and the tail behavior of the noise. They propose algorithms based on Square-Root Slope estimator that achieves error that is upto logarithmic factors of the minimax optimal rate for the $\epsilon$. 
\end{itemize}

\subsection{Robust learning of linear dynamical systems}\label{app-sec:robust-lds-lit-review}
While robust regression under adversarial contamination has been extensively studied in the i.i.d. setting, comparatively fewer works address this problem for linear dynamical systems and higher-order VAR models. Existing works primarily focus on the following settings.
%
%
%
\subsubsection{Heavy-tailed covariates or process noise}
\cite{kanakeri2025outlier} study the problem of learning a linear dynamical system from $N$ independent trajectories in the presence of heavy-tailed noise. In their work, the data is generated through a linear dynamical system $x_{t+1}=A^*x_t+\eta_{t+1}$ with $\eta_{t+1}$ coming from a zero-mean i.i.d. heavy-tailed distribution with bounded second and fourth moments. The authors partition trajectories into buckets, estimate the system matrix locally within each bucket, and then aggregate the estimates by taking their median under the Frobenius norm. They show robustness to heavy-tailed noise  with non-asymptotic $O(1/\sqrt{N})$ estimation error rate w.r.t the operator norm. This multi-trajectory setting does not require any stability assumption on $A^*$.

\cite{wong2019lasso} explores $\beta$-mixing with heavy-tailed covariates. They consider a model where the goal is to predict $x_t$ using $(x_{t-1}, x_{t-2}, \dots, x_{t-d})$, using a linear predictor under the condition that the dynamic process for $x_t$ is stationary. The paper assumes the covariates $x_t$ follow a sub-Weibull distribution allowing for heavier-tails. They propose a least-squares estimator with $\ell_1$-penalty and show that when the underlying true predictor matrix is sparse and $(x_t)_{t>1}$ are centered, the estimator can recover the underlying predictor matrix up to an error dependent on the weight given to the penalty and restricted eigenvalue conditions. 

\cite{halder2022robust} examines the problem of learning a VAR(d), i.e. $x_t = \sum_{i=1}^d A_i x_{t-i}+\eta_t$, where $\eta_t$ can be heavy-tailed. They assume $\eta_t$ has conditionally symmetric density (condition on the history of the observations) and $x_t$ satisfies either $\beta$-mixing or Markov chain conditions (as well as $x_t$ is stationary) and show that a variant of Lasso allows for consistency under mixing as well as Markov chain conditions.

\cite{qiu15robust} consider a stationary VAR model $x_{t+1} =A^* x_{t} +\eta_{t+1}$ where covariates and process noise follow a joint elliptical distribution, to capture heavy-tailed behavior. They propose a quantile-based robust variant of the Dantzig estimator and show that the recovery rates for $A^*$ matches the parametric rates under the entry-wise max norm. 

\cite{dijk2025tail} study high-dimensional heavy-tailed time series using a factor-adjusted VAR model. They consider a VAR(d) model, $x_{t}=\sum_{i=1}^d A_i x_{t-i}+\eta_t$ where $\eta_t$ is zero mean with uncorrelated (potentially heavy-tailed) entries. The learner observes a $y_t=x_t+\nu_t$ where $\nu_t$ is a factor-adjusted component. The authors apply element-wise truncation and a multi-stage parameter estimation procedure using lasso, under the assumption of finite $(2+2\epsilon)$-th moments for some $\epsilon \in (0,1)$. 

\cite{Lu2025Robust} analyze dynamic spatial autoregressive models under heavy-tailed distributions, proposing a constrained Yule-Walker estimation framework with statistical recovery guarantees. They consider a spacial autoregressive model where $x_t = A_0 W x_t+A_1 x_{t-1} + A_1 W x_{t-1}+\eta_{t}$, with $W$ a weighting matrix capturing the relationships in the spatial domain. The work proposes a robust estimator for covariance estimation under this heavy-tailed setting.

\cite{wang2022rate} explores system recovery from a VAR(d) model, $x_t = \sum_{i=1}^d A_i x_{t-i} +\eta_t$ where $\eta_t$ is a serially uncorrelated white noise innovation where the white noise can be heavy-tailed. They propose a constrained Yule-Walker estimator that is robust against heavy-tailed noise with bounded $2+2\epsilon$-th moments for any $\epsilon>0$. The paper establishes minimax optimal rates by establishing minimax lower bounds matching their recovery guarantees. 

\cite{wang2024robust} investigate least absolute deviation regression for high-dimensional heavy-tailed time series with finite $\alpha$-moments for $\alpha \in (1,2]$. They consider VAR(d) of the form, $x_{t}=\sum_{i=1}^d A_i x_{t-i}+\eta_t$, where $(\eta_t)_{t \ge 1}$ is i.i.d. and has zero mean and finite $\alpha$-moments for $\alpha \in (1,2]$. The paper proposes a truncated estimator for which they show that the population risk for the estimator converges to the optimal population risk, as $T$ increases.

\subsection{Sparse process noise}
The works of \cite{kim2026necessity} and \cite{kim2026huber} consider a linearly parameterized (potentially) non-linear system given by $x_{t+1} = A^* \phi(x_t)+\eta_t$ where $\phi: \matR^n \rightarrow \matR^m$ is a non-linear mapping and $\eta_t$ is analogous to the process noise in \eqref{eq:dyn_process}. They consider scenarios where $\eta_t$ is i.i.d. Gaussian or sub-Gaussian as well as where $\eta_t$ comes from a sparse adversarial process (sub-Gaussian but not necessarily zero mean or independent). \cite{kim2026necessity} propose a two-stage algorithm for recovery of $A^*$. In the first stage, they solve a Least Absolute Deviation problem independently for each coordinate of the autoregressive vectors to obtain candidate rows of the system matrix $A$. In the second stage, they apply a filtering step based on the initial estimates and then perform Least Squares on the filtered data, recovering the rows of $A$ up to a small error. The follow-up work by~\cite{kim2026huber} extends this analysis and shows that under the same assumptions, Huber estimator recovers the rows of $A^*$ up-to a small error. In both works, they show that for symmetric process noise, they can recover the rows of $A^*$ with a rate of $O(1/\sqrt{T})$ when the system is linear.

\cite{kim2025macroeconomic} explores robust learning of a VAR(d) model, i.e., $x_{t} = \sum_{i=1}^d A_i x_{t-i}+\eta_t+v_t$ where $\eta_t$ is the process noise and $v_t$ is an innovative outlier. They show that least squares with $\ell_1$-penalty on each $v_t$, can be used to recover the matrices $(A_i)_{i=1}^d$ (by using an equivalence with Huber estimator and applying iteratively reweighted least squares) under the assumption that $\eta_t+v_t$ is sampled from an i.i.d. distribution with bounded variance (and additional structural assumptions). Their guarantees are in the form of asymptotic normality of $\matrixenc{A_1 & A_2 & \dots & A_d}$. 
 
 \cite{farahmand2011doubly} explores the related problem of doubly robust smoothing of dynamical systems where the goal is to learn the dynamical system and the corruption, given system dynamics and under distributional assumptions. In this setting, a learner observes a $y_t$ generated through $x_t = A_t x_{t-1}+\eta(x)_t+o(x)_t$ and $y_t=B_t x_{t}+\eta(y)_t+o(y)_t$ where $o(x)_t, o(y)_t$ are outliers and $\eta(x)_t, \eta(y)_t$ are noise coming from zero mean distributions with known covariance. The learner is also assumed to know $(A_t,B_t)_{t=1}^T$ and the noise covariance. The goal is to learn $(x_t,o(x)_t, o(y)_t)_{t=1}^T$ assuming that the stacked vector of outliers is sparse. They explore a least-squares estimator with $\ell_1$-penalty and introduces a provably convergent algorithm via coordinate descent.
%
%

%
%
%
\subsection{Learning from corrupted states}

In the work of \cite{kanakeri2025outlier} discussed earlier, the authors further prove (for the strong contamination adversarial model) that if an $\epsilon$-fraction of the $N$ trajectories are corrupted, the estimator still recovers the system matrix up to a small error (w.r.t operator norm). Their theoretical result for this adversarial setting, and even their algorithm does not apply to the single trajectory setting considered in our work. 

The work of \cite{wong2017lasso}  considers a VAR(d) model, $x_{t}=\sum_{i=1}^{d} A_i x_{t-i}+\eta_{t}$, where $\eta_t$'s are independent centered Gaussians. Under the assumption that $\rho(A_i)<1$ for all $i \in [d]$, and assuming that the learner observes $\tilde{x}_t = x_t+w_t$ where $w_t$ comes from an independent zero-mean Gaussian distribution with known covariance, they establish non-asymptotic recovery guarantees for the lasso estimator with $\ell_1$ penalty.

\cite{Bhatia0KK16} consider a 1-dimensional auto-regressive process with degree $d$, i.e. $x_t=\sum_{i=1}^d A_i x_{t-i}+\eta_t$ where $(x_{t-i})_{i=0}^d,\eta_t \in \matR$. They consider an outlier model related to our setting where the learner observes $\tilde{x}_t=x_t+\theta_t$ where $\theta = \matrixenc{\theta_1 & \theta_2 & \dots & \theta_T}$ is $k$-sparse. By transforming the AR(d) into a form of $y=X^\top w+\eta$ (where $y=\matrixenc{x_1 & x_2 & \dots & x_n}$ and $X,w,\eta$ defined appropriately to recover the AR(d)), under subset strong convexity and subset strong smoothness of the corrupted covariate matrix, and assuming the outliers are independent of noise (which is not stated explicitly but seems to be needed in the proof, e.g., within Appendix B.6) they propose a block-sparse hard thresholding algorithm which is robust against additive outliers and provably consistent. 

A line of work has considered the setting where $\tilde{x}_t$ is formed by independently sub-sampling the entries of $x_t$ at each time point $t$ (and setting other entries to $0$) for $(x_t)_{t=1}^T$ generated from \eqref{eq:dyn_process}. When the system is stationary the works of \cite{rao2016system, milind2017estimation} employ covariance estimation to recover $A^*$ under different structural assumptions such as sparsity, and low-rankness. The works of \cite{JalaliWillett18, wong2017lasso} explore lasso-type estimators with an additional term to correct for the bias (which makes the problem potentially non-convex), and provably estimate $A^*$ under sparsity and stability assumptions. When $A^*$ is sparse and strictly stable, \cite{castro_missing25} establish minimax optimal recovery guarantees with respect to the entry-wise $\max$ norm, for the Dantzig selector. \cite{Rao2017_compmeas} consider the setting where $\tilde{x}_t$ is formed by taking compressive linear measurements of $x_t$ (at each $t$), and study the limits of covariance-based estimation.


\section{Robust estimators based on least-trimmed squares} \label{app:rob_estimators}

In this section, we provide a detailed overview of the estimators introduced in Section~\ref{sec:rob_estimators}.

\subsection{Sphere relaxation}\label{sec:sphere}

Consider the optimization problem~\eqref{eq:quad-form} and the two constraints $z \in \set{0,1}^{T}$ and $\vecOne_T^\top z = T-s$. A simple calculation shows that when $z \in \set{0,1}^{T}$, $\vecOne_T^\top z = T-s$ is equivalent to $\norm{z}_2^2 = z^\top z = T-s$. Therefore, we can see that optimization problem~\eqref{eq:quad-form} is equivalent to,  
\begin{align*}
& \min_{A \in \calA,z \in \set{0,1}^{T}} z^\top Q(A) z \quad \text{s.t.} \quad \norm{z}_2^2 = T-s, \quad \vecOne_T^\top z = T-s 
\end{align*}
where the vector $z$ is a binary vector on the surface of a sphere of radius $\sqrt{T-s}$ centered at the origin. Given this, removing the constraint $z \in \set{0,1}^{T+1}$, i.e. letting $z \in \matR^{T+1}$, results in,
\begin{align*}
& \min_{A \in \calA,z \in \matR^T} z^\top Q(A) z \quad \text{s.t.} \quad \norm{z}_2^2 = T-s, \quad  \vecOne_T^\top z = T-s. 
\end{align*}

\paragraph{Given $A$, estimating $z$.} Now consider the scenario where we have been given $A$ and we are tasked with finding $z$. Note that when $A$ is fixed, i.e. when $Q(A)$ is fixed, the optimization problem~\eqref{eq:sphere} translates to solving a trust-region problem which has polynomial time algorithms \citep{sorenson1983}.  However, since $z \in \matR^{T}$ is not necessarily a binary vector, we project it on to $\set{0,1}^{T}$ via the following rounding method.
\begin{enumerate}
\item \label{stp:extrct} Extract the subset of $\est{S} \subset \set{1,\dots,T}$ corresponding to the $T-s$ entries of $\set{z_1,\dots,z_T}$ with the largest magnitudes (ignoring the sign).
\item For $t \in \est{S}$ set $z_t=1$ and for $t \not \in \est{S}$ set $z_t= 0$.
\end{enumerate}

\paragraph{Given $z$, estimating $A$.} Given a binary vector $z$ (if $z$ is not binary we can use the rounding discussed above), we use \eqref{eq:quad-form} as stated in step~\ref{stp:A-calc} of Algorithm~\ref{alg:robust-estimator} to solve a $\calA$-constrained weighted least squares problem to derive an estimate $A$.

%
%
\subsection{Semi-definite program (SDP) relaxation}\label{sec:SDP}

Since $z^\top Q(A) z = \dotprod{Q(A)}{zz^\top}$ we can replace $zz^\top$ with the p.s.d matrix $Y$. Note that when $Y=zz^\top$, $Y_{tt} = z_t^2$ and since $z_t \in \{0,1\}$, $Y_{tt} = z_t$ which gives us $\diag(Y)=z$. We apply the dual of Shor's relaxation \citep{Boyd1997} to derive the following.
\begin{align*}
& \min_{A \in \calA,Y \succeq 0, z \in [0,1]^T} \langle Q(A), Y \rangle  \\
& \text{s.t.}\; {\diag(Y) = z}, \ \vecOne_T^\top z = T-s \text{ and } \matrixenc{1 & z^\top \\ z & Y} \succeq 0.
\end{align*}
Note that for a fixed $A$, \eqref{eq:SDP} is a SDP and can be solved efficiently to obtain $Y, z$.

\paragraph{Given $A$, estimating $z$.} Let $Y, z$ be the solution of \eqref{eq:SDP} when $A$ is fixed. Since $z \in [0,1]^{T}$ is not necessarily a binary vector, we can project it on to $\set{0,1}^{T}$ as described above.

\paragraph{Given $z$, estimating $A$.} Given a binary vector $z$ (if $z$ is not binary we can use the rounding discussed above), we use \eqref{eq:quad-form} as stated in step~\ref{stp:A-calc} of  Algorithm~\ref{alg:robust-estimator} to solve a $\calA$-constrained weighted least squares problem to derive an estimate $A$.

\subsection{Bi-convex relaxation}\label{sec:biconvex}

Note that while \eqref{eq:comb-opt} captures the problem of filtering the contaminants and learning a linear dynamical process over the uncontaminated observations, the formulation involves a product of $z_t z_{t+1}$ (for $t \in \set{0,1,\dots,T-1}$) in the objective function. In this section, we will introduce a set of new variables to remove the multiplicative dependencies of $z_t, z_{t+1}$. 

Let $b_t = z_t z_{t+1}$ for each $t \in \set{1,\dots,T-1}$. Then note that the objective of \eqref{eq:comb-opt} can be re-written as $\sum_{t=1}^{T-1} b_t \norm{\tildex_{t+1}-A\tildex_t}_2^2$ by introducing an additional constraint $b_t = z_t z_{t+1}$ for $t \in \set{1,\dots,T-1}$. However, while this removes the multiplicative dependencies between $z_t$ and $z_{t+1}$ in the objective, we are still left with the same dependencies in the constraints. We can remove these dependencies using the following key observation.
\begin{remark}
When $z_t \in \set{0,1}$ for $t \in \set{1,\dots, T}$ and $b_t \in \set{0,1}$ for $t \in \set{1,\dots, T-1}$, the constraint $b_t = z_t  z_{t+1}$ is equivalent to the set of constraints $b_t \le z_t, b_t \le z_{t+1}$ and $b_t \ge z_t+z_{t+1}-1$. To see this we can consider the possible values of $z_t$ and $z_{t+1}$.
\begin{itemize}
\item {\bf Case 1: ($z_t = z_{t+1}=1$)} The constraint $b_t \ge z_t+z_{t+1}-1=1+1-1=1$ implies $b_t = 1 = z_t \cdot z_{t+1}$.

\item {\bf Case 2: ($z_t=0$ or $z_{t+1}=0$)} The two constraints $b_t \le z_t$ and $b_t \le z_{t+1}$ imply $b_t \le \min\set{z_t,z_{t+1}} = 0$ which implies $b_t = 0 = z_t z_{t+1}$.
\end{itemize}
\end{remark}
Therefore, we can use this to derive the optimization problem,
\begin{align*}
& \min_{\substack{A \in \calA \\ z \in \set{0,1}^T \\ b \in \set{0,1}^{T-1}}} \sum_{t=1}^{T-1} b_t \norm{\tildex_{t+1}-A\tildex_t}_2^2 \\
& \text{s.t.}\; \sum_{t=1}^T z_t = T-s \\
& z_t+z_{t+1}-1 \le b_t \le \min \set{z_t,z_{t+1}} \;\forall t\in [T-1].
\end{align*}
Note that the objective is biconvex. Relaxing the binary constraints on $z_t$ and $b_t$ to the interval $[0,1]$ results in the optimization problem
\begin{align*}
& \min_{\substack{A \in \calA \\ z \in [0,1]^T \\ b \in [0,1]^{T-1}}} \sum_{t=1}^{T-1} b_t \norm{\tildex_{t+1}-A\tildex_t}_2^2\\
& \text{s.t.}\; \sum_{t=1}^T z_t = T-s  \nonumber\\
& z_t+z_{t+1}-1 \le b_t \le \min \set{z_t,z_{t+1}} \; \forall t\in [T-1] \nonumber
\end{align*}
which has convex constraints (since $\calA$ is convex by assumption).

\paragraph{Given $A$, estimating $(z_t)_{t=1}^T$ and $(b_t)_{t=1}^{T-1}$.} When $A$ is fixed note that the optimization problem~\eqref{eq:biconvex} translates to a convex program with a linear objective. Solving this gives us $(z_t)_{t=1}^{T}$ and $(b_t)_{t=1}^{T-1}$ such that $z_t \in [0,1]$ and $b_t \in [0,1]$. In order to project $z_t, b_t$ on to $\set{0,1}$, we consider the following rounding method.
\begin{enumerate}
\item Extract the subset $\est{S} \subset \set{1,\dots,T}$ corresponding to the $T-s$ entries of $\set{z_1,\dots,z_T}$ with the largest magnitude.

\item \label{stp:set-val} Set the values at $t\in S$ to $1$ and set all other entries ($t \not \in S$) to $0$.

\item For all $t \in [T-1]$, set $b_t = z_t z_{t+1}$ where $z_t, z_{t+1}$ come from step~\ref{stp:set-val}.
\end{enumerate}

\paragraph{Given $z$, estimating $A$.} Given a binary vector $z$ (if $z$ is not binary we can use the rounding discussed above), we use \eqref{eq:quad-form} as stated in step~\ref{stp:A-calc} of the Algorithm~\ref{alg:robust-estimator} to solve a $\calA$-constrained weighted least squares problem to derive an estimate $A$.

%
%
\subsection{Convexifying \texorpdfstring{$Q(A)$}{Q(A)} to estimate \texorpdfstring{$z$}{z}} \label{sec:convexQA}

Assume $A$ is fixed and the goal is to estimate $z$. Note that $Q(A)$ is a 1-banded matrix and is not necessarily positive semi-definite (p.s.d). However, since $\vecOne_T^\top z = T-s$, when $z \in \matR^{T}$, the objective of the optimization problem~\eqref{eq:quad-form} can be written as $z^\top \left(Q(A)+\alpha(A) \cdot I_{T}\right)z$ (where $\alpha(A)$ is some parameter). To see this, we can simply calculate the addition to the objective $z^\top \left(\alpha(A) \cdot I_{T}\right)z = \alpha \cdot z^\top z =\alpha(A) \cdot (T-s)$ which is a fixed amount when $\alpha(A)$ is fixed. Therefore, this shows that 
by choosing $\alpha(A)$ to be suitably large, we can ensure that $\tilde{Q}(A) = Q(A)+\alpha(A) I_T$ is a p.s.d matrix. Let $\alpha(A) = \min\set{0, -\lambda_{min}(Q(A))}$ where $\lambda_{min}(Q(A))$ is the smallest eigenvalue of $Q(A)$. This gives us the optimization formulation
\begin{align*}
& \min_{z\in \set{0,1}^{T}} z^\top \tilde{Q}(A) z \quad \text{s.t.} \quad  \vecOne_T^\top z = T-s.
\end{align*}
By relaxing $\set{0,1}^T$ to $[0,1]^{T}$ we obtain
\begin{align*}
& \min_{z \in [0,1]^{T}} z^\top \tilde{Q}(A) z \quad \text{s.t.} \quad \vecOne_T^\top z = T-s
\end{align*}
which is a convex program w.r.t $z$ (for a fixed $A$).

\paragraph{Given $A$, estimating $z$.} When $A$ is fixed note that $\tilde{Q}(A)$ is a fixed p.s.d matrix. Therefore, we can use standard solvers (e.g., interior point methods) to find an estimate $z$. However, since $z \in [0,1]^{T}$, this does not guarantee a binary vector. In order to project the vector $z$ on to $\set{0,1}^{T}$, we consider the following rounding method.
\begin{enumerate}
\item \label{stp:extrct-quad} Extract the subset $\est{S} \subset \set{1,\dots,T}$ corresponding to the $T-s$ entries of $\set{z_1,\dots,z_T}$ with the largest magnitude.
\item For $t \in S$ set $z_t=1$ and for $t \not \in S$ set $z_t= 0$.
\end{enumerate}

\paragraph{Given $z$, estimating $A$.} Given a binary vector $z$ (if $z$ is not binary we can use the rounding discussed in the previous chapter to round it to get a binary vector), we use \eqref{eq:quad-form} as stated in step~\ref{stp:A-calc} of the Algorithm~\ref{alg:robust-estimator} to solve an $\calA$-constrained weighted least-squares problem to derive an estimate $A$.


\section{Additional experiments}\label{appsec:additional-exp}

\subsection{Overview of the implementation}

The following is a detailed overview of the techniques used to implement each estimator in the experiments. Note that in some cases use approximation methods due to the computational cost of exact solvers. 

\paragraph{Sphere relaxation.} Note that the sphere relaxation is a variant of the Trust Region Subproblem. To see this, let $z = \frac{T-s}{T}\vecOne_T+w$ where $\dotprod{w}{\vecOne_T}=0$. This gives 
\begin{align*}
&\vecOne_T^\top z = T-s \iff \vecOne_T^\top w = 0\\
& \norm{z}_2^2 = T-s \iff \frac{(T-s)^2}{T}+\norm{w}_2^2 = T-s \iff \norm{w}_2^2 = \frac{s(T-s)}{T}\\
& z^\top Q(A) z = w^\top Q(A) w+ 2\left(\frac{T-s}{T} \right)\vecOne_T^\top Q(A) w+c,
\end{align*}
where $c$ is independent of $w$.

Therefore for fixed $A$, we obtain a Trust Region Subproblem where the goal is to find a $w$ orthogonal to $\vecOne_T$, which can be solved using the Mor\'{e}-Sorensen algorithm~\citep{sorenson1983}. Our implementation solves a tridiagonal variant of this subproblem using Scipy's \emph{solve\_banded}.

\paragraph{SDP relaxation.} Consider our SDP relaxation,
\begin{align*}
& \min_{A \in \calA,Y \succeq 0, z \in [0,1]^T} \langle Q(A), Y \rangle  \\
& \text{s.t.}\; {\diag(Y) = z}, \vecOne_T^\top z = T-s \text{ and } \matrixenc{1 & z^\top \\ z & Y} \succeq 0.
\end{align*}
Let $h = 2z-\vecOne_T$. Then $\vecOne_T^\top h = T-2s$. 
Define $X = \matrixenc{1 & h^\top \\ h & H}$ so that
\[
\matrixenc{1 & z^\top \\ z & Y} = LXL^\top
\]
where $L = \matrixenc{1 & 0 \\ \frac{1}{2}\vecOne_T & \frac{1}{2}I_T}$. Since $L$ is invertible, $X \succeq 0$ as well. Moreover, since $\diag(Y)=z$, and since $$Y_{tt} = \frac{1}{4} (1+2h_t+H_{tt}) = z_{t} = \frac{1}{2}(1+h_t)$$ we obtain $\diag(H) = \vecOne_T$. Furthermore, $\dotprod{Q(A)}{Y} = \dotprod{C}{X}+\text{const}$ for a matrix $C$ depending on $Q(A)$. Hence, the feasible set is the elliptope intersected with the single linear constraint $\vecOne_T^\top h = T-2s$. The map $(z,Y)\leftrightarrow X$ is a bijection between the feasible sets that preserves the objective up to a constant, so the optimal values agree and optimal solutions correspond one-to-one. Given this, we use the Burer-Monteiro \citep{Burer2003, boumal2018nonconvex} to solve for $z$. 

\paragraph{Bi-convex relaxation.} For the bi-convex relaxation, when $A$ is fixed, we solve for $z$ using \emph{cvxpy}. 

\paragraph{Convexifying $Q(A)$ to estimate $z$.} Here, we solve the resulting convex problem for $z$ directly using \emph{cvxpy}. 

\paragraph{$\calA$-constrained weighted least squares.} For \texttt{Sphere-AM}, \texttt{SDP-AM}, \texttt{Biconvex-AM}, and \texttt{ConvexQ(A)-AM}, when $z$ is fixed, we solve for $A$ using FISTA \citep{beck2009fast} with a projection on to $\norm{\cdot}_{1,1}$ ball (where the projection is computed in $O(n^2\log n)$ by a sorting-based algorithm \citep{duchi2008efficient}). 

\paragraph{Least squares with $\ell_1/\ell_2$ block-sparse penalization.} Let $\rho_t(A)=D_t(\tilde x_{t+1}-A\tilde x_t)$. When $A$ is fixed, the objective separates across columns as $\sum_t\big[\tfrac12\|\rho_t-u_t\|^2+\lambda\|u_t\|_2\big]$, and each term is minimized by block soft-thresholding, $u_t=\Big(1-\frac{\lambda}{\|\rho_t\|_2}\Big)_+\rho_t$. When $U$ is fixed, the penalty is a constant, and the problem reduces to weighted least squares,
\[
A=\arg\min_{A\in\mathcal A}\ \sum_{t:D_t=1}\big\|(\tilde x_{t+1}-u_t)-A\tilde x_t\big\|_2^2.
\]
With $\calA$ set to the $\norm{\cdot}_{1,1}$ ball, we again solve for $A$ using FISTA algorithm \citep{beck2009fast} with a projection on to $\norm{\cdot}_{1,1}$ ball (where the projection is computed in $O(n^2\log n)$ by a sorting-based algorithm \citep{duchi2008efficient}).

\paragraph{Least squares with block-sparse hard constraints.} When $A$ is fixed, solving for $U$ involves selecting the $2s$ columns of $U$ with the largest $\ell_2$ norm. With $U$ fixed this translates to $\calA$-constrained weighted-least-squares.

\subsection{Additional results on synthetic experiments}

\paragraph{Contamination models.} In this set of experiments we consider four different contamination models. %
\begin{enumerate}
\item {\bf Random heavy-tailed}:  outliers drawn from a Cauchy distribution with location and scale both set to $10$ are added to indices selected uniformly at random (with probability depending on $s$). 
\item {\bf Block burst}:  equally spaced blocks of five contiguous entries are contaminated with additive errors drawn from a Cauchy distribution with location and scale both set to $10$.    
\item {\bf Sign flip}:  data points are selected uniformly at random (with probability depending on $s$), and their signs are flipped (i.e., they are multiplied by $-1$).      
\item {\bf High-leverage}:  the $s$ entries with the largest $\ell_2$ norm are contaminated by adding Gaussian noise with standard deviation $8$.
\end{enumerate}

\subsubsection{Block burst contamination}

\begin{figure}[htbp]
    \centering
    \begin{subfigure}[b]{0.32\textwidth}
        \centering
        \includegraphics[width=\linewidth]{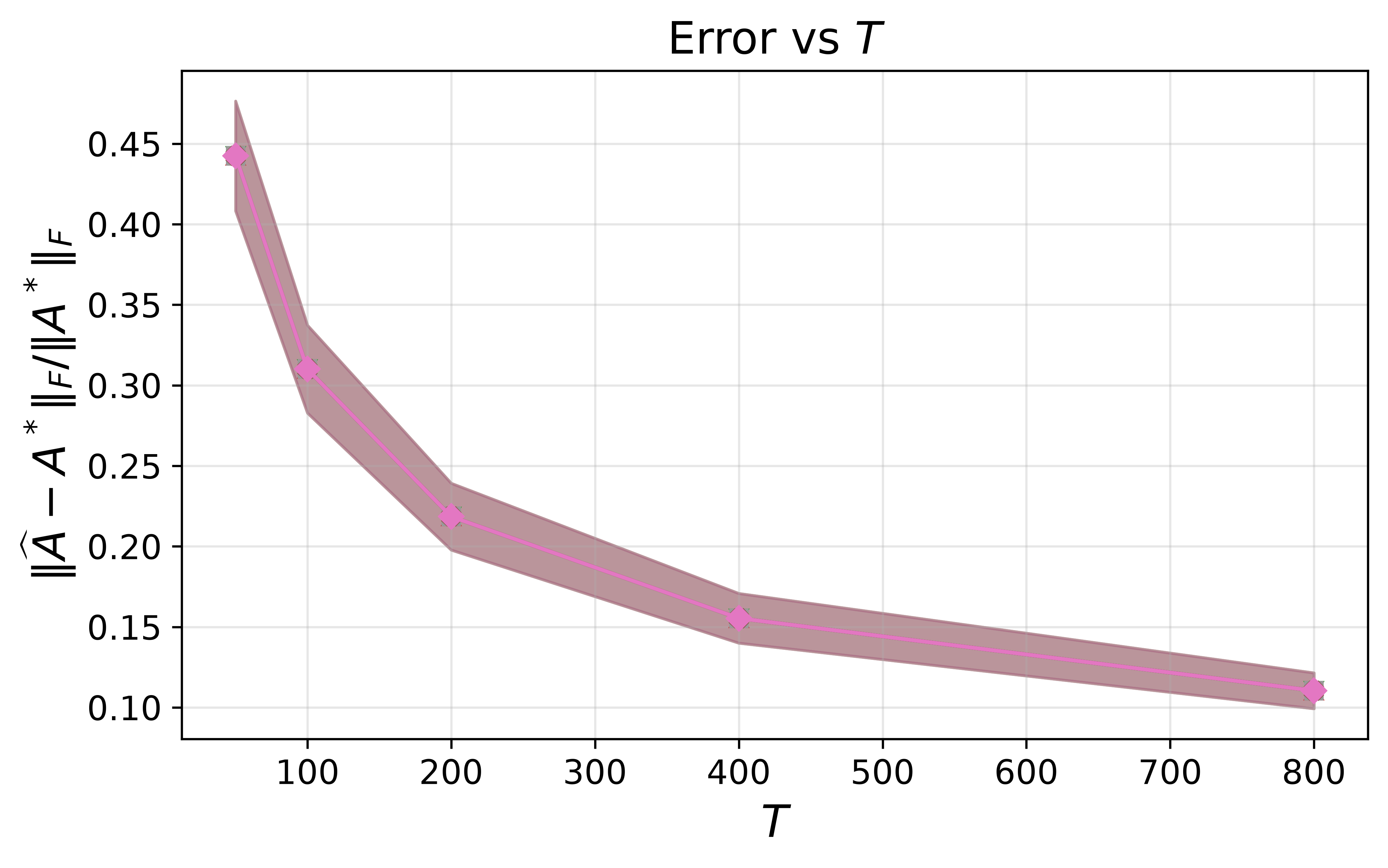}
        \caption{$\epsilon = 0$}
        \label{fig:T_block_burst_0}
    \end{subfigure}
    \hfill
    \begin{subfigure}[b]{0.32\textwidth}
        \centering
        \includegraphics[width=\linewidth]{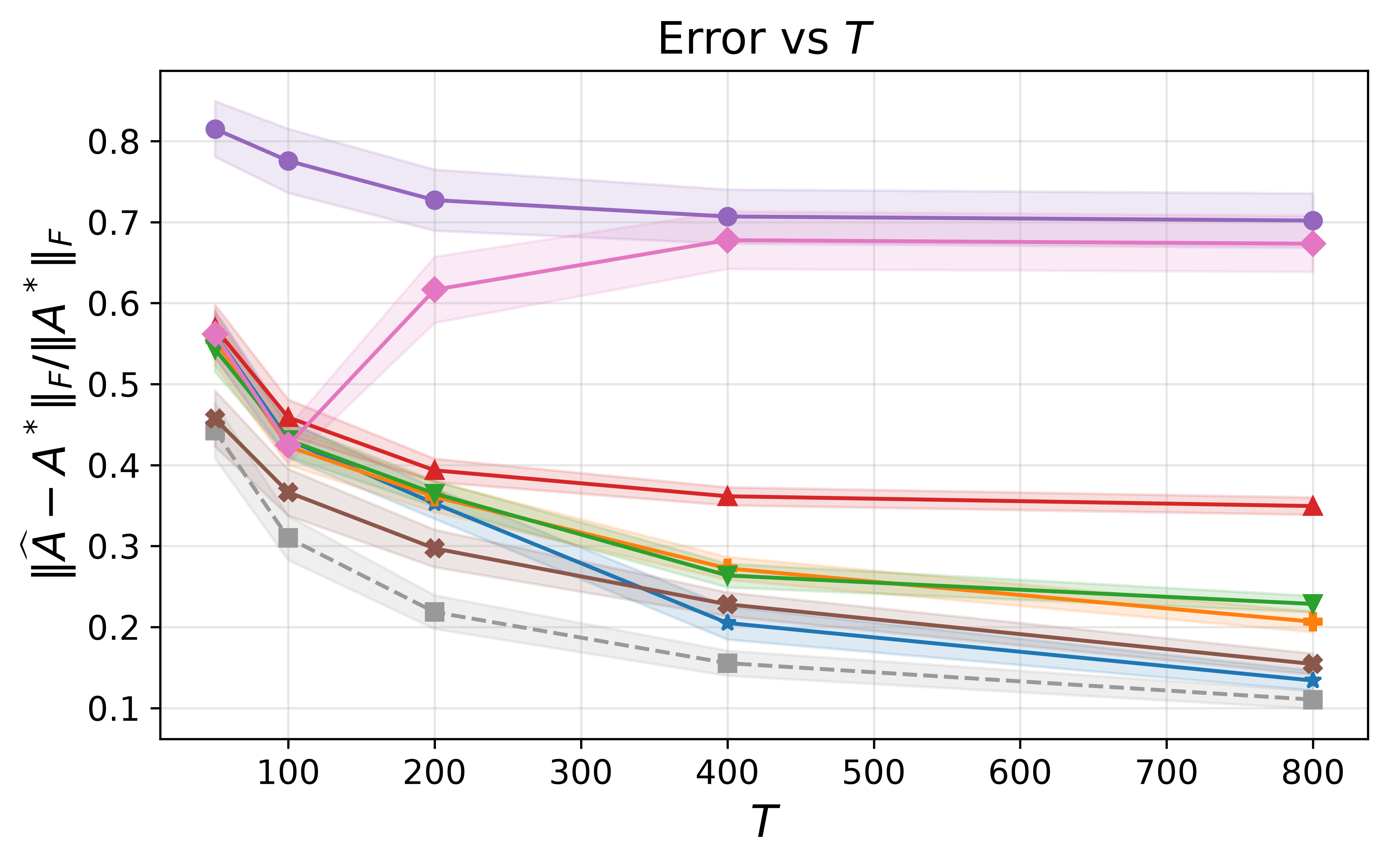}
        \caption{$\epsilon = 0.02$}
        \label{fig:T_block_burst_002}
    \end{subfigure}
    \hfill
    \begin{subfigure}[b]{0.32\textwidth}
        \centering
        \includegraphics[width=\linewidth]{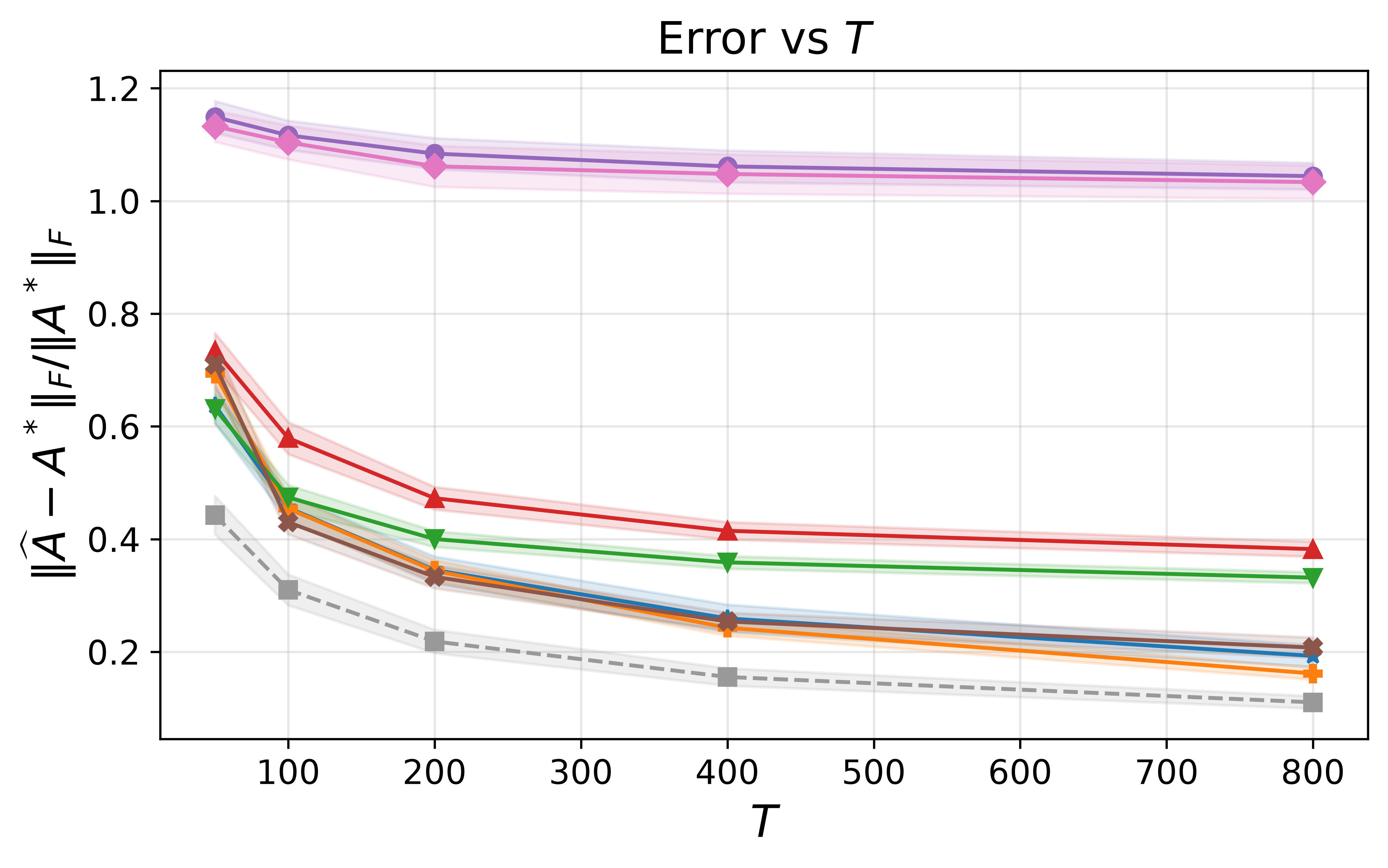}
        \caption{$\epsilon = 0.1$}
        \label{fig:T_block_burst_01}
    \end{subfigure}
    
    \vspace{1em} 
    
    \begin{subfigure}[b]{0.6\textwidth}
        \centering
        \includegraphics[width=\linewidth]{exp_results_synth/legend}
        \label{fig:T_block_burst_legend}
    \end{subfigure}

    \makeatletter
    \renewcommand{\@makecaption}[2]{%
      \vspace{\abovecaptionskip}%
      \parbox{\textwidth}{{\bf #1:} #2}%
      \vspace{\belowcaptionskip}%
    }
    \makeatother
    \caption{
    Frobenius norm error, $\norm{\widehat{A}-A^*}_F/{\norm{A^*}_F}$, for $T \in \{50,100,200,400,800\}$ with (a) $\epsilon=0$, (b) $\epsilon=0.02$, and (c) $\epsilon=0.1$, for \emph{block burst contamination} model.}
    \label{fig:T_block_burst}
\end{figure}

\paragraph{Frobenius norm error against $T$.} Figure~\ref{fig:T_block_burst} shows the Frobenius error, ${\norm{\widehat{A}-A^*}_F}/{\norm{A^*}_F}$, against $T \in \set{50,100,200,400,800}$. When $\epsilon=0$, all estimators achieve the same error since they coincide with \texttt{OLS}. While \texttt{Sphere-AM} fails to recover $A^*$ for $\epsilon > 0.02$ or $T > 100$, it performs even better than \texttt{Biconvex-AM} when both $\epsilon$ and $T$ are small. In small-$\epsilon$ regime, \texttt{SDP-AM} initially outperforms the other estimators, with \texttt{Biconvex-AM} achieving the lowest error as $T$ increases (see Figure~\ref{fig:T_block_burst_002}). However, when $\epsilon = 0.1$, \texttt{ConvexQ(A)-AM} outperforms the other estimators, with \texttt{Biconvex-AM} and \texttt{SDP-AM} achieving comparable performance. For all estimators except \texttt{Sphere-AM} and \texttt{OLS}, the error decreases as $T$ increases, although the rate of decrease slows.

\begin{figure}[htbp]
    \centering
    \begin{subfigure}[b]{0.32\textwidth}
        \centering
        \includegraphics[width=\linewidth]{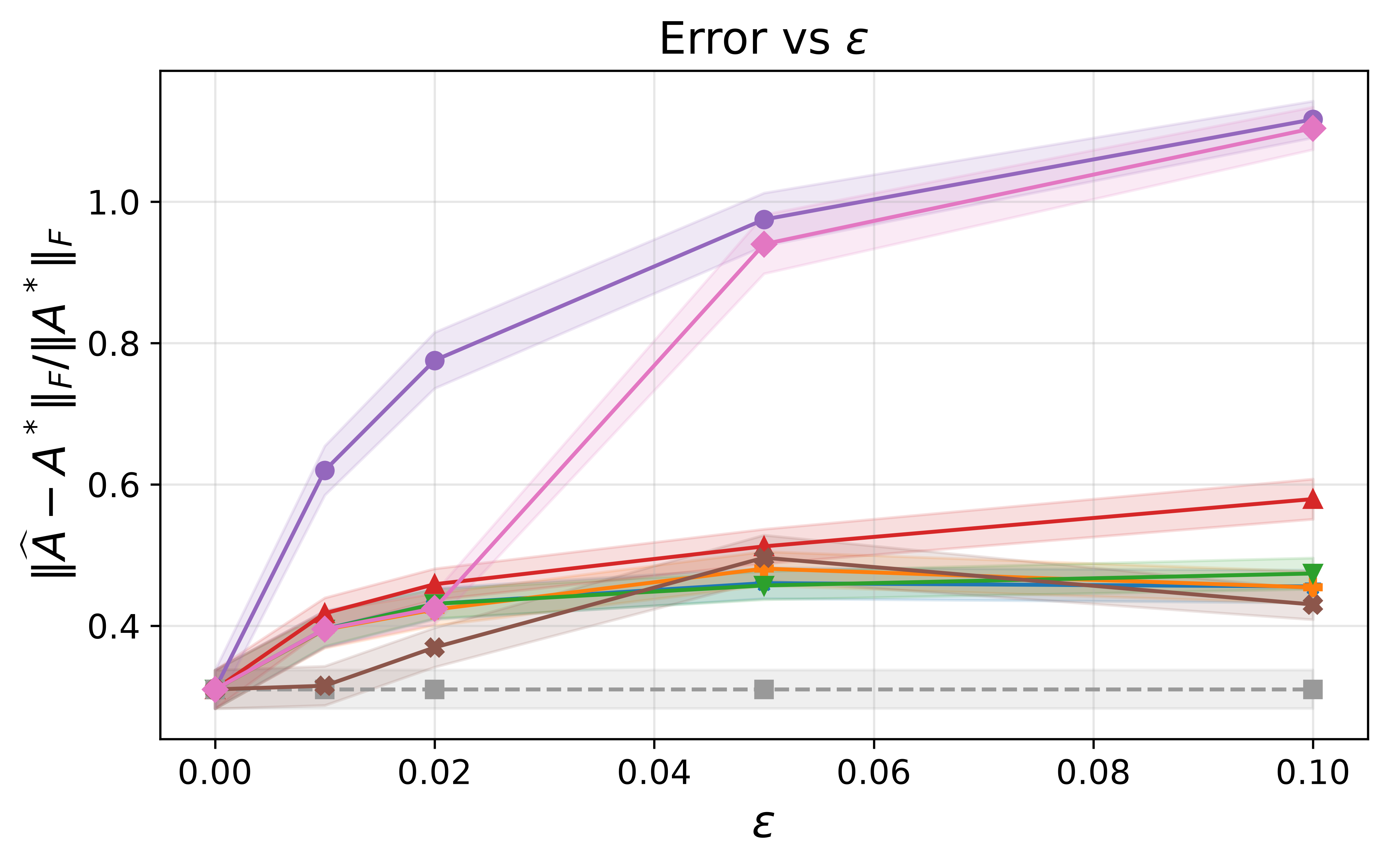}
        \caption{$T=100$}
        \label{fig:eps_block_burst_100}
    \end{subfigure}
    \hfill
    \begin{subfigure}[b]{0.32\textwidth}
        \centering
        \includegraphics[width=\linewidth]{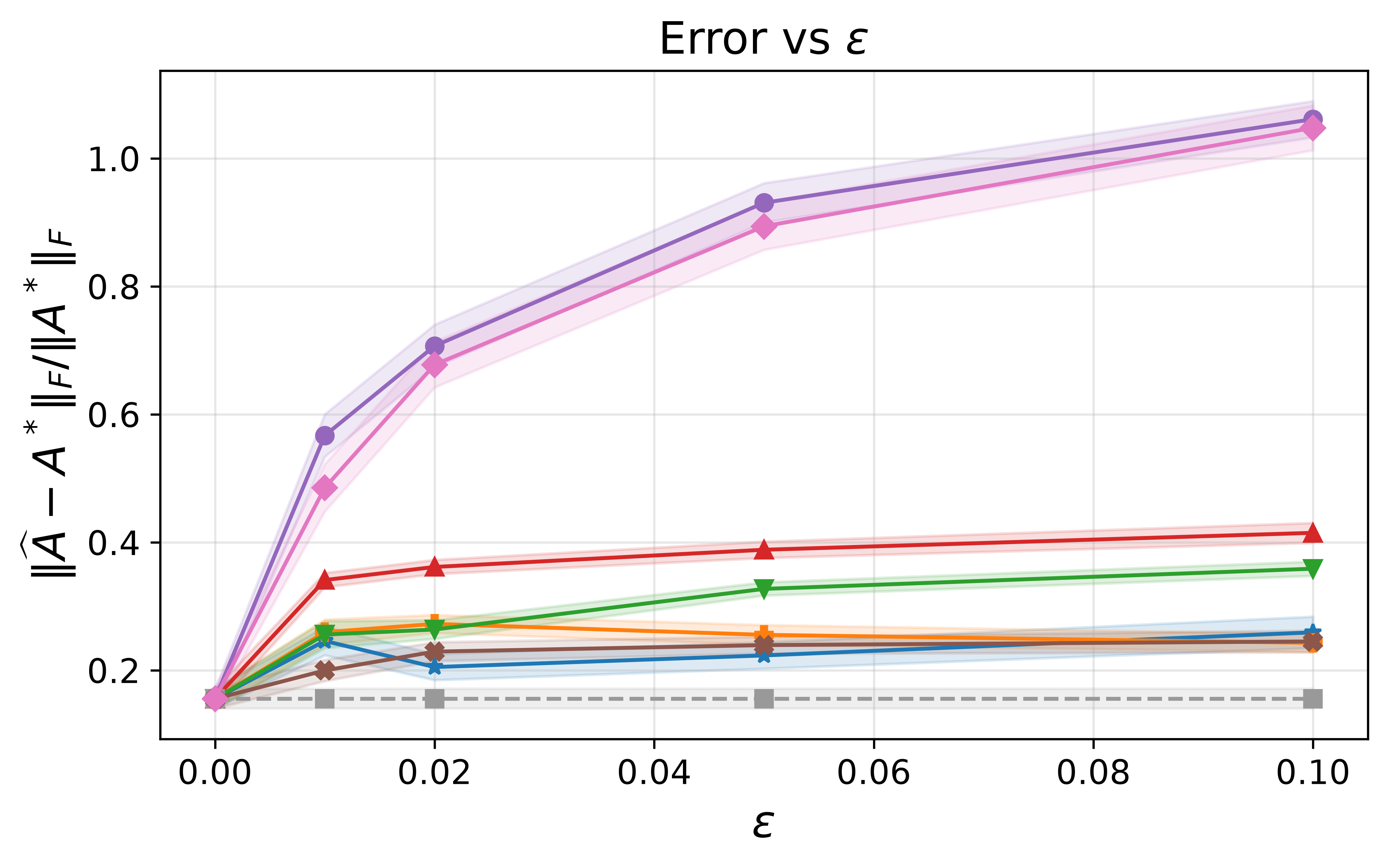}
        \caption{$T=400$}
        \label{fig:eps_block_burst_400}
    \end{subfigure}
    \hfill
    \begin{subfigure}[b]{0.32\textwidth}
        \centering
        \includegraphics[width=\linewidth]{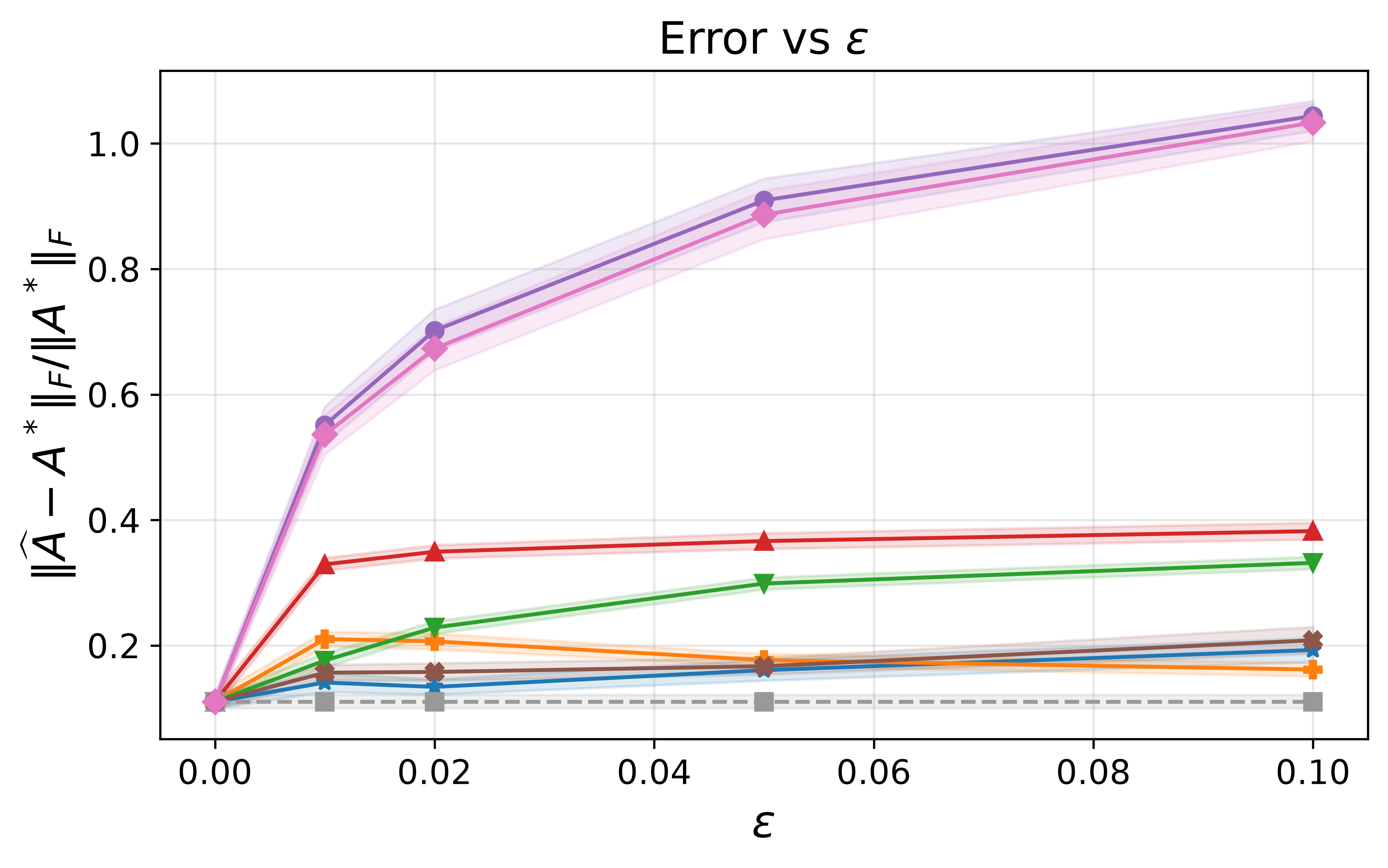}
        \caption{$T=800$}
        \label{fig:eps_block_burst_800}
    \end{subfigure}
    
    \vspace{1em} 
    
    \begin{subfigure}[b]{0.6\textwidth}
        \centering
        \includegraphics[width=\linewidth]{exp_results_synth/legend}
        \label{fig:eps_block_burst_legend}
    \end{subfigure}
    \makeatletter
    \renewcommand{\@makecaption}[2]{%
      \vspace{\abovecaptionskip}%
      \parbox{\textwidth}{{\bf #1:} #2}%
      \vspace{\belowcaptionskip}%
    }
    \makeatother
    \caption{
    Frobenius norm error, $\norm{\widehat{A}-A^*}_F/{\norm{A^*}_F}$, for $\epsilon = \{0,0.01,0.02,0.05,0.1\}$ with (a) $T=100$, (b) $T=400$, and (c) $T=800$, for \emph{block burst contamination} model.}
    \label{fig:eps_block_burst}
\end{figure}

\paragraph{Frobenius norm error against $\epsilon$.} Figure~\ref{fig:eps_block_burst} shows the Frobenius error, ${\norm{\widehat{A}-A^*}_F}/{\norm{A^*}_F}$, against $\epsilon \in \set{0,0.01,0.02,0.05,0.1}$. We observe similar behavior across all $T \in \set{100,400,800}$ with the \texttt{OLS} and the \texttt{Sphere-AM} showing the highest error. As expected, the error increases with $\epsilon$. While \texttt{SDP-AM} achieves a smaller error in the small-$\epsilon$-regime, \texttt{Biconvex-AM} and \texttt{ConvexQ(A)-AM} perform best as $\epsilon$ increases. Note that while \texttt{Sphere-AM} generally performs comparably to \texttt{OLS}, it outperforms all estimators except \texttt{SDP-AM} when $T \le 100$ and $\epsilon \le 0.02$, similar to the heavy-tailed contamination example.

\subsubsection{Sign flip contamination}

\begin{figure}[htbp]
    \centering
    \begin{subfigure}[b]{0.32\textwidth}
        \centering
        \includegraphics[width=\linewidth]{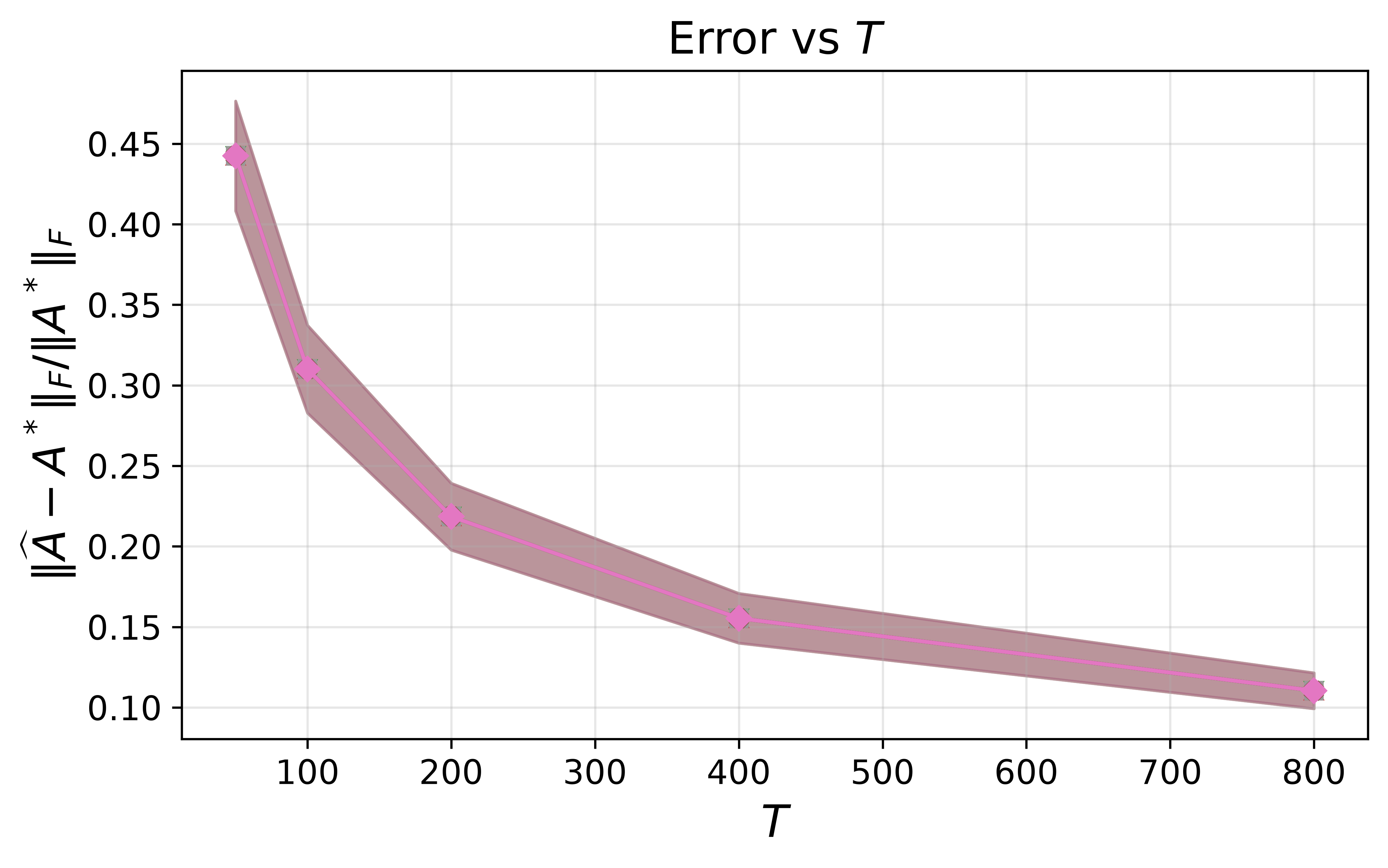}
        \caption{$\epsilon = 0$}
        \label{fig:T_sign_flip_0}
    \end{subfigure}
    \hfill
    \begin{subfigure}[b]{0.32\textwidth}
        \centering
        \includegraphics[width=\linewidth]{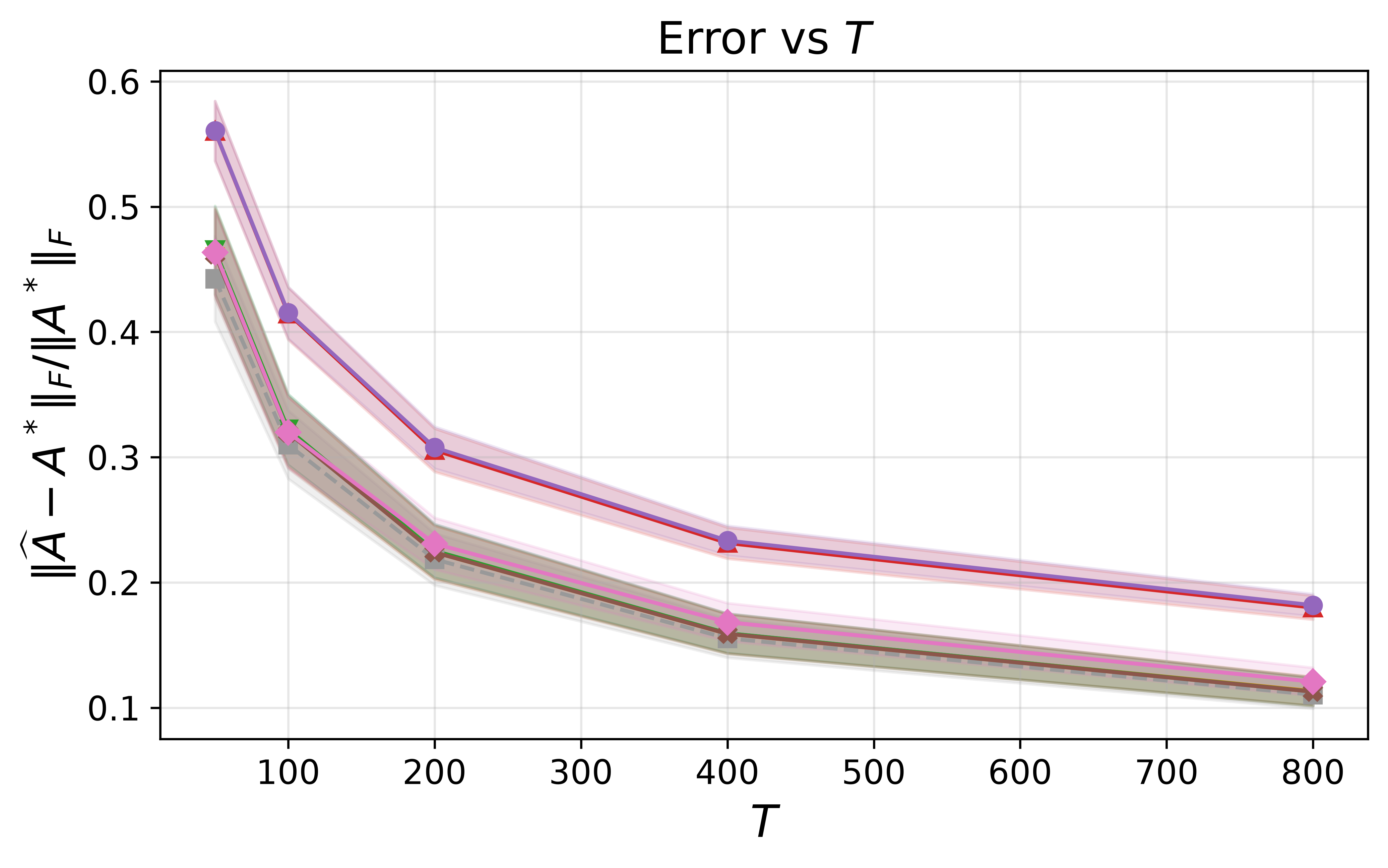}
        \caption{$\epsilon = 0.02$}
        \label{fig:T_sign_flip_002}
    \end{subfigure}
    \hfill
    \begin{subfigure}[b]{0.32\textwidth}
        \centering
        \includegraphics[width=\linewidth]{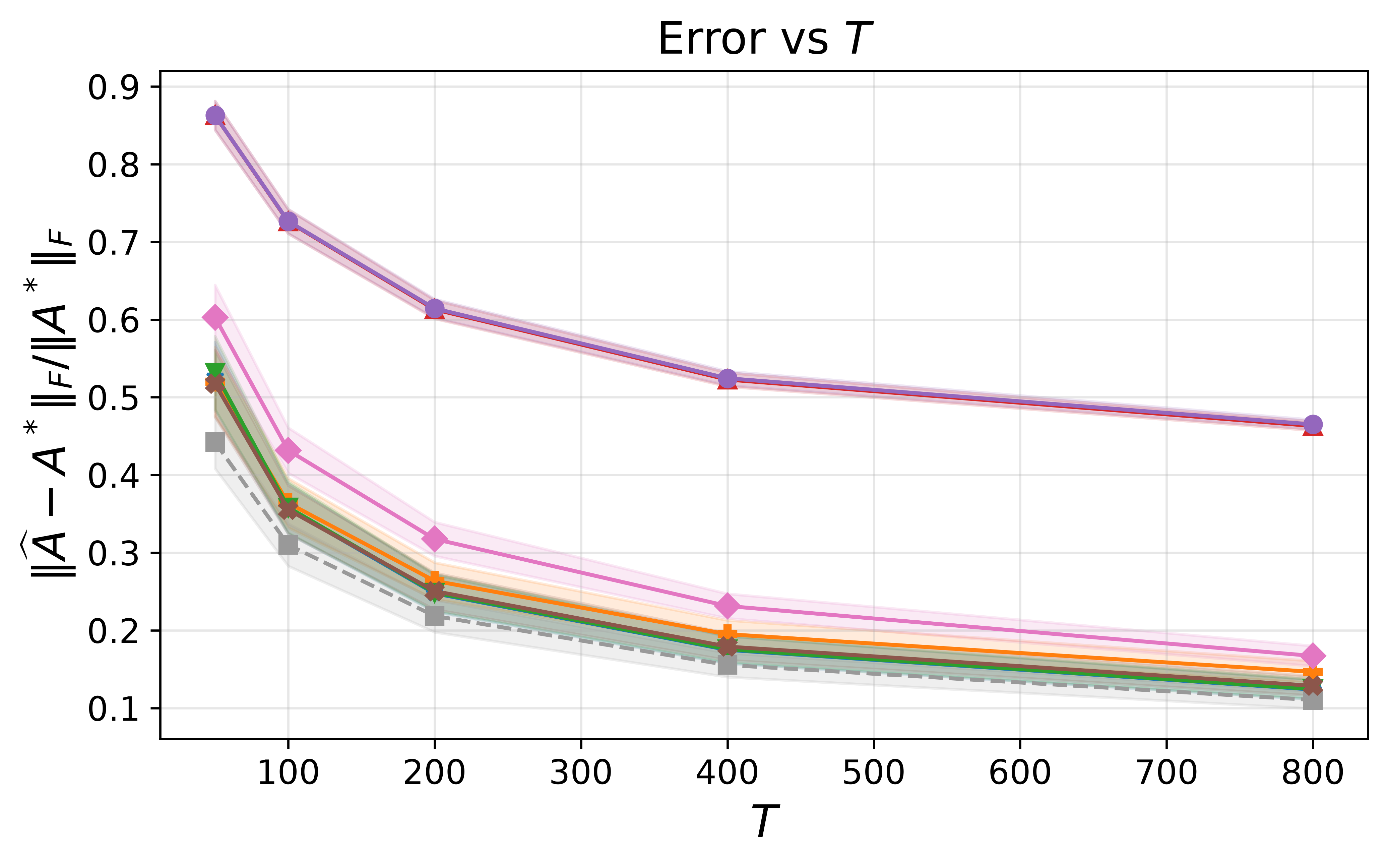}
        \caption{$\epsilon = 0.1$}
        \label{fig:T_sign_flip_01}
    \end{subfigure}
    
    \vspace{1em} 
    
    \begin{subfigure}[b]{0.6\textwidth}
        \centering
        \includegraphics[width=\linewidth]{exp_results_synth/legend}
        \label{fig:T_sign_flip_legend}
    \end{subfigure}
    \makeatletter
    \renewcommand{\@makecaption}[2]{%
      \vspace{\abovecaptionskip}%
      \parbox{\textwidth}{{\bf #1:} #2}%
      \vspace{\belowcaptionskip}%
    }
    \makeatother

    \caption{
    Frobenius norm error, $\norm{\widehat{A}-A^*}_F/{\norm{A^*}_F}$, for $T \in \{50,100,200,400,800\}$ with (a) $\epsilon=0$, (b) $\epsilon=0.02$, and (c) $\epsilon=0.1$, under the \emph{sign flip contamination} model.}
    \label{fig:T_sign_flip}
\end{figure}

\paragraph{Frobenius norm error against $T$.} Figure~\ref{fig:T_sign_flip} shows the Frobenius error, ${\norm{\widehat{A}-A^*}_F}/{\norm{A^*}_F}$, against $T \in \set{50,100,200,400,800}$. When $\epsilon=0$, all estimators achieve the same error as they coincide with \texttt{OLS}. An interesting observation is that, under sign flip contamination, all estimators including \texttt{OLS} perform well. Also note that, \texttt{LS-BSP} coincides with \texttt{OLS}, and \texttt{Sphere-AM} outperforms both. There is almost no gap among the remaining estimators, which all exhibit similar rates. The behavior of \texttt{LS-BSP} can be attributed to the fact that sign flip preserves the magnitude of the entries and with our choice of $\lambda$,  $U = 0$ becomes optimal, and the estimator collapses to \texttt{OLS}. Notably,\texttt{OLS} exhibits the same error-decay with respect to $T$ as the other estimators under sign flip contamination, so this behavior does not conflict with our guarantees for \texttt{LS-BSP}.

\begin{figure}[htbp]
    \centering
    \begin{subfigure}[b]{0.32\textwidth}
        \centering
        \includegraphics[width=\linewidth]{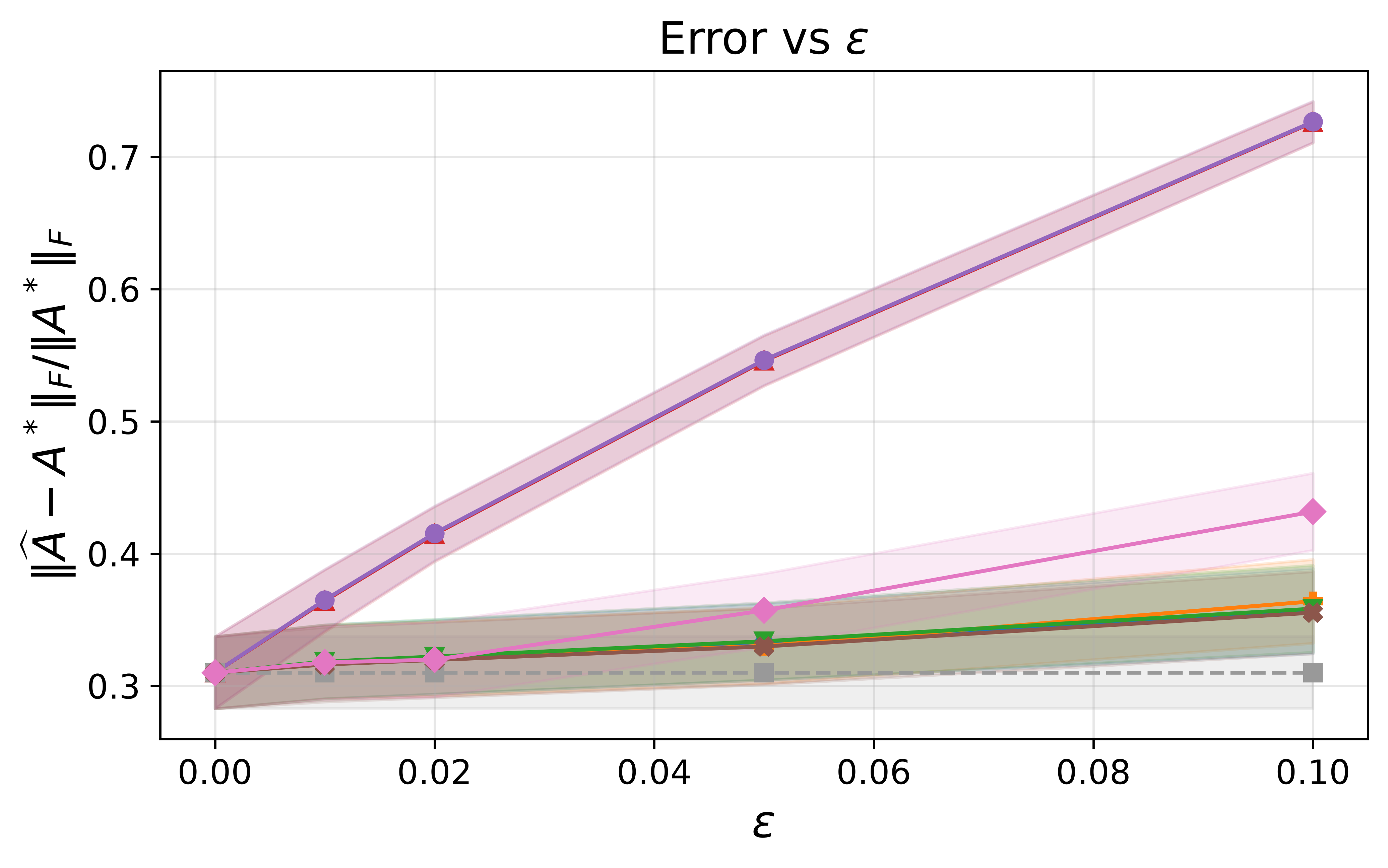}
        \caption{$T=100$}
        \label{fig:eps_sign_flip_100}
    \end{subfigure}
    \hfill
    \begin{subfigure}[b]{0.32\textwidth}
        \centering
        \includegraphics[width=\linewidth]{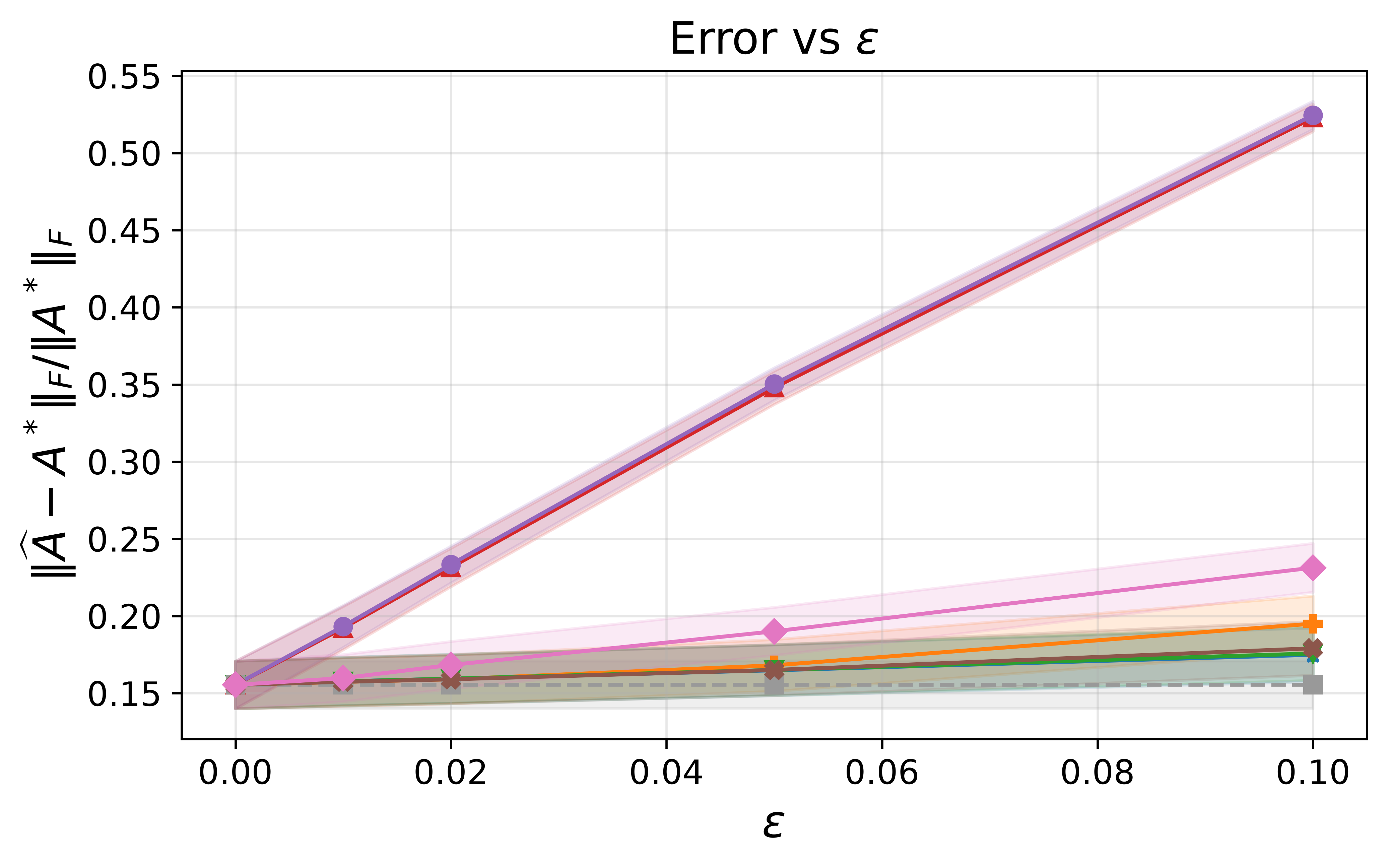}
        \caption{$T=400$}
        \label{fig:eps_sign_flip_400}
    \end{subfigure}
    \hfill
    \begin{subfigure}[b]{0.32\textwidth}
        \centering
        \includegraphics[width=\linewidth]{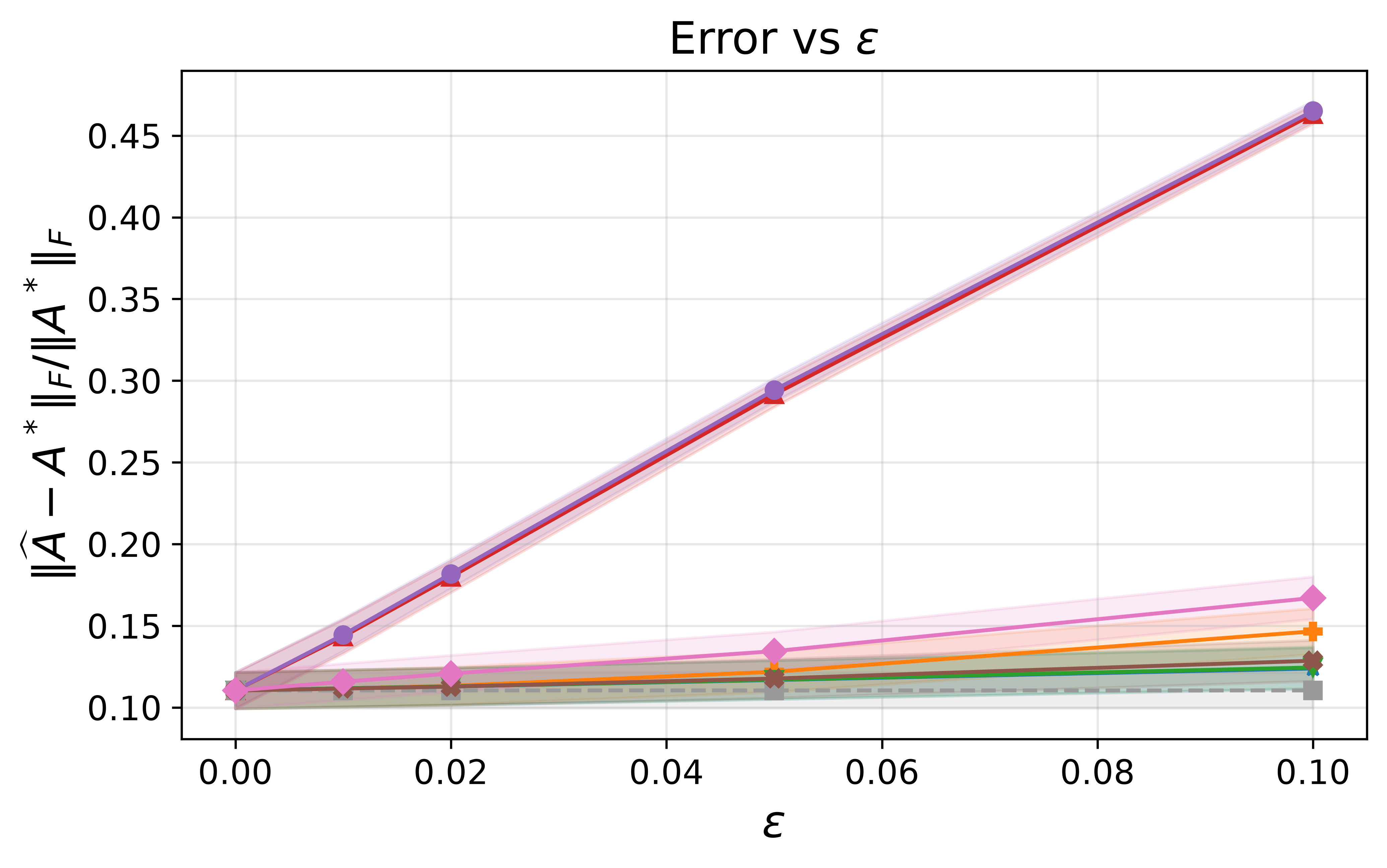}
        \caption{$T=800$}
        \label{fig:eps_sign_flip_800}
    \end{subfigure}
    
    \vspace{1em} 
    
    \begin{subfigure}[b]{0.6\textwidth}
        \centering
        \includegraphics[width=\linewidth]{exp_results_synth/legend}
        \label{fig:eps_sign_flip_legend}
    \end{subfigure}
    \makeatletter
    \renewcommand{\@makecaption}[2]{%
      \vspace{\abovecaptionskip}%
      \parbox{\textwidth}{{\bf #1:} #2}%
      \vspace{\belowcaptionskip}%
    }
    \makeatother
    \caption{
    Frobenius norm error, $\norm{\widehat{A}-A^*}_F/{\norm{A^*}_F}$, for $\epsilon = \{0,0.01,0.02,0.05,0.1\}$ with (a) $T=100$, (b) $T=400$, and (c) $T=800$, under the \emph{sign flip contamination} model.}
    \label{fig:eps_sign_flip}
\end{figure}

\paragraph{Frobenius norm error against $\epsilon$.} Figure~\ref{fig:eps_sign_flip} shows the Frobenius error, ${\norm{\widehat{A}-A^*}_F}/{\norm{A^*}_F}$, against $\epsilon \in \set{0,0.01,0.02,0.05,0.1}$. We observe similar behavior across all $T \in \set{100,400,800}$, with \texttt{OLS} and \texttt{LS-BSP} showing the highest error. As expected, the error increases with $\epsilon$. Note that \texttt{Biconvex-AM}, \texttt{SDP-AM}, and \texttt{LS-BSHC} outperform the other estimators across all settings. 

\subsubsection{High-leverage contamination}

\begin{figure}[htbp]
    \centering
    \begin{subfigure}[b]{0.32\textwidth}
        \centering
        \includegraphics[width=\linewidth]{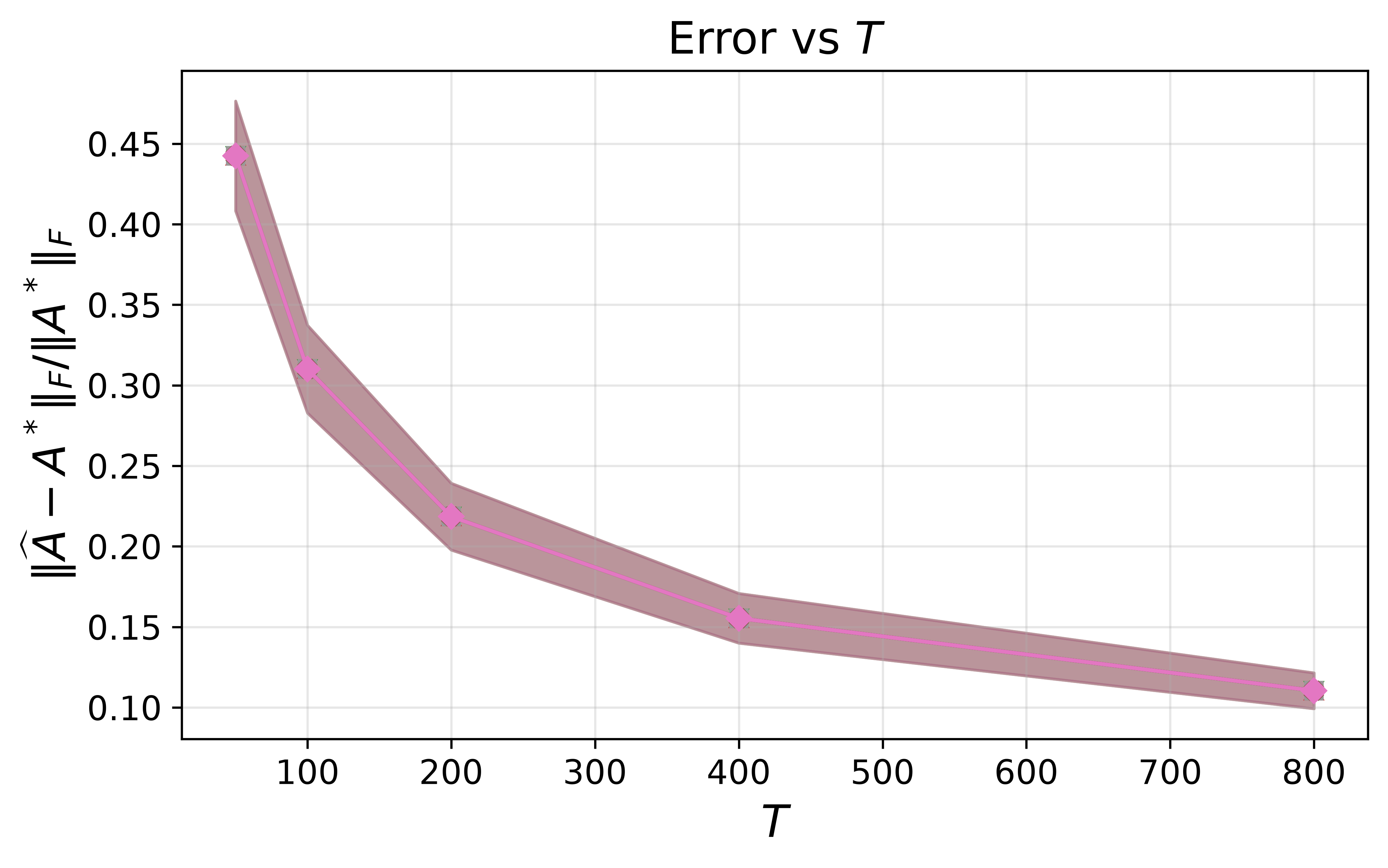}
        \caption{$\epsilon = 0$}
        \label{fig:T_leverage_targeted_0}
    \end{subfigure}
    \hfill
    \begin{subfigure}[b]{0.32\textwidth}
        \centering
        \includegraphics[width=\linewidth]{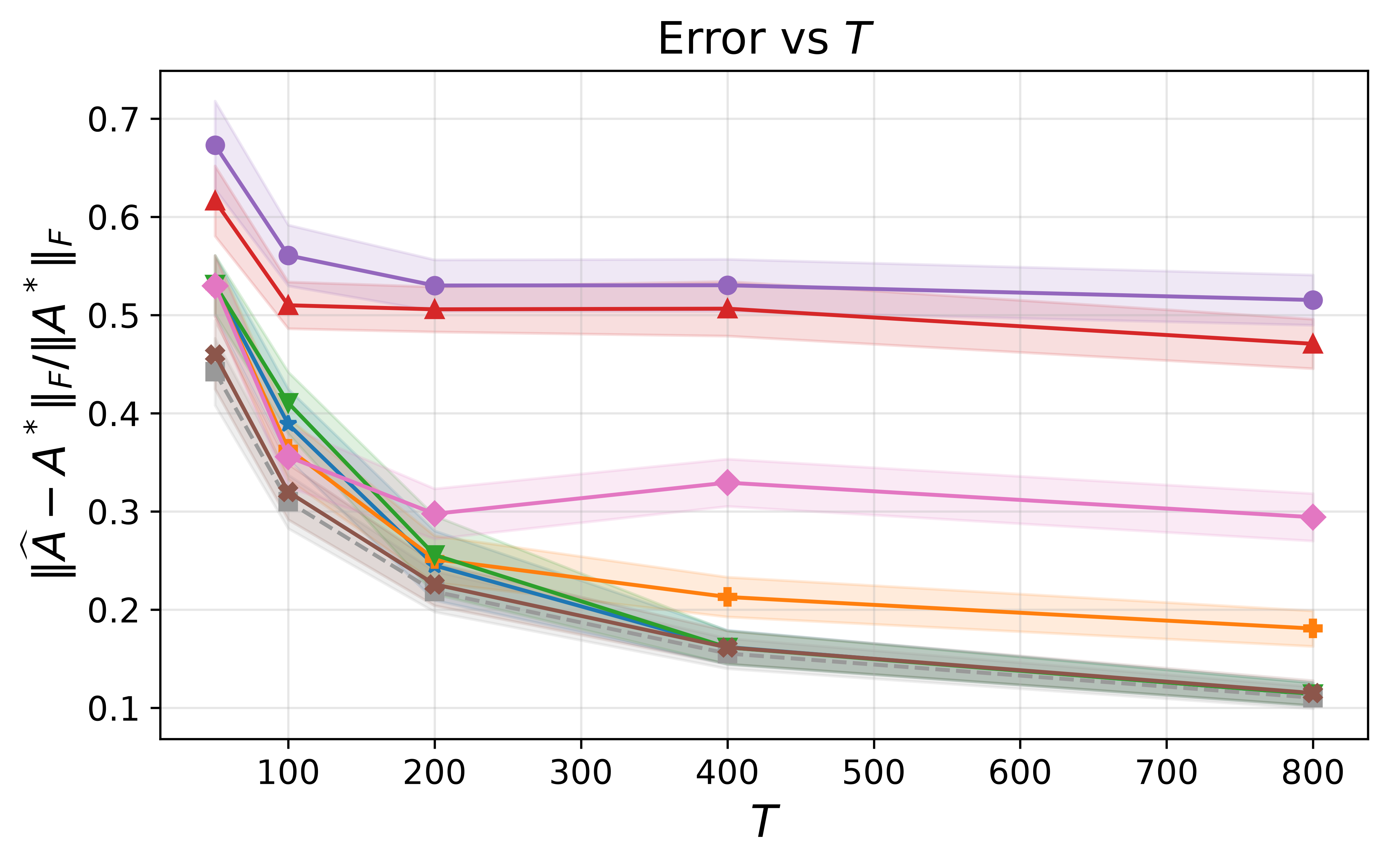}
        \caption{$\epsilon = 0.02$}
        \label{fig:T_leverage_targeted_002}
    \end{subfigure}
    \hfill
    \begin{subfigure}[b]{0.32\textwidth}
        \centering
        \includegraphics[width=\linewidth]{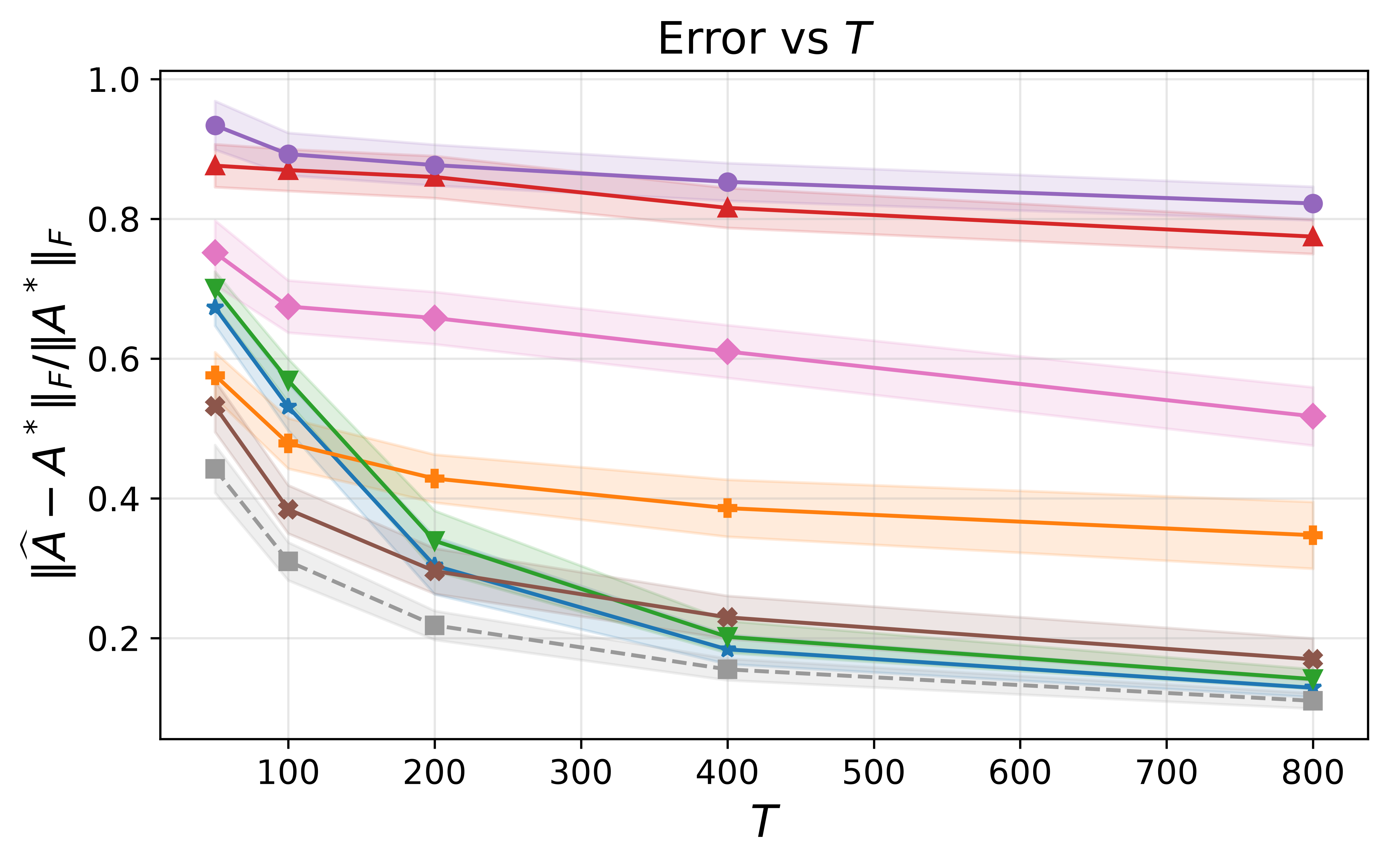}
        \caption{$\epsilon = 0.1$}
        \label{fig:T_leverage_targeted_01}
    \end{subfigure}
    
    \vspace{1em} 
    
    \begin{subfigure}[b]{0.6\textwidth}
        \centering
        \includegraphics[width=\linewidth]{exp_results_synth/legend}
        \label{fig:T_leverage_targeted_legend}
    \end{subfigure}
    \makeatletter
    \renewcommand{\@makecaption}[2]{%
      \vspace{\abovecaptionskip}%
      \parbox{\textwidth}{{\bf #1:} #2}%
      \vspace{\belowcaptionskip}%
    }
    \makeatother

    \caption{
    Frobenius norm error, $\norm{\widehat{A}-A^*}_F/{\norm{A^*}_F}$, for $T \in \{50,100,200,400,800\}$ with (a) $\epsilon=0$, (b) $\epsilon=0.02$, and (c) $\epsilon=0.1$, under the \emph{high-leverage contamination} model.}
    \label{fig:T_leverage_targeted}
\end{figure}

\paragraph{Frobenius norm error against $T$.} Figure~\ref{fig:T_leverage_targeted} shows Frobenius error, ${\norm{\widehat{A}-A^*}_F}/{\norm{A^*}_F}$ against $T \in \set{50,100,200,400,800}$. When $\epsilon=0$, all estimators achieve the same error as they coincide with \texttt{OLS}. As $\epsilon$ increases, the error of every estimator still decays with $T$, although \texttt{OLS} and \texttt{LS-BSP} decay considerably more slowly than the others. \texttt{SDP-AM}, \texttt{Biconvex-AM}, and \texttt{LS-BSHC} perform well, with \texttt{Biconvex-AM} achieving a recovery error close to that of $\calA$-constrained least squares on uncontaminated data. Since this contamination model simply adds zero-mean Gaussian noise, we would expect all estimators to have decaying error (albeit different rates) which aligns with the observations. The behavior of \texttt{LS-BSP} can be attributed to the choice of $\lambda$, which we set match the $\lambda$ in the theoretical guarantees. This $\lambda$ is chosen to dominate the Gaussian innovations $\eta_t$, so the additional zero-mean Gaussian corruption could also be absorbed by $\lambda$. As a result, \texttt{LS-BSP} behaves comparably to \texttt{OLS}, which also shows decaying error in this setting. 

\begin{figure}[htbp]
    \centering
    \begin{subfigure}[b]{0.32\textwidth}
        \centering
        \includegraphics[width=\linewidth]{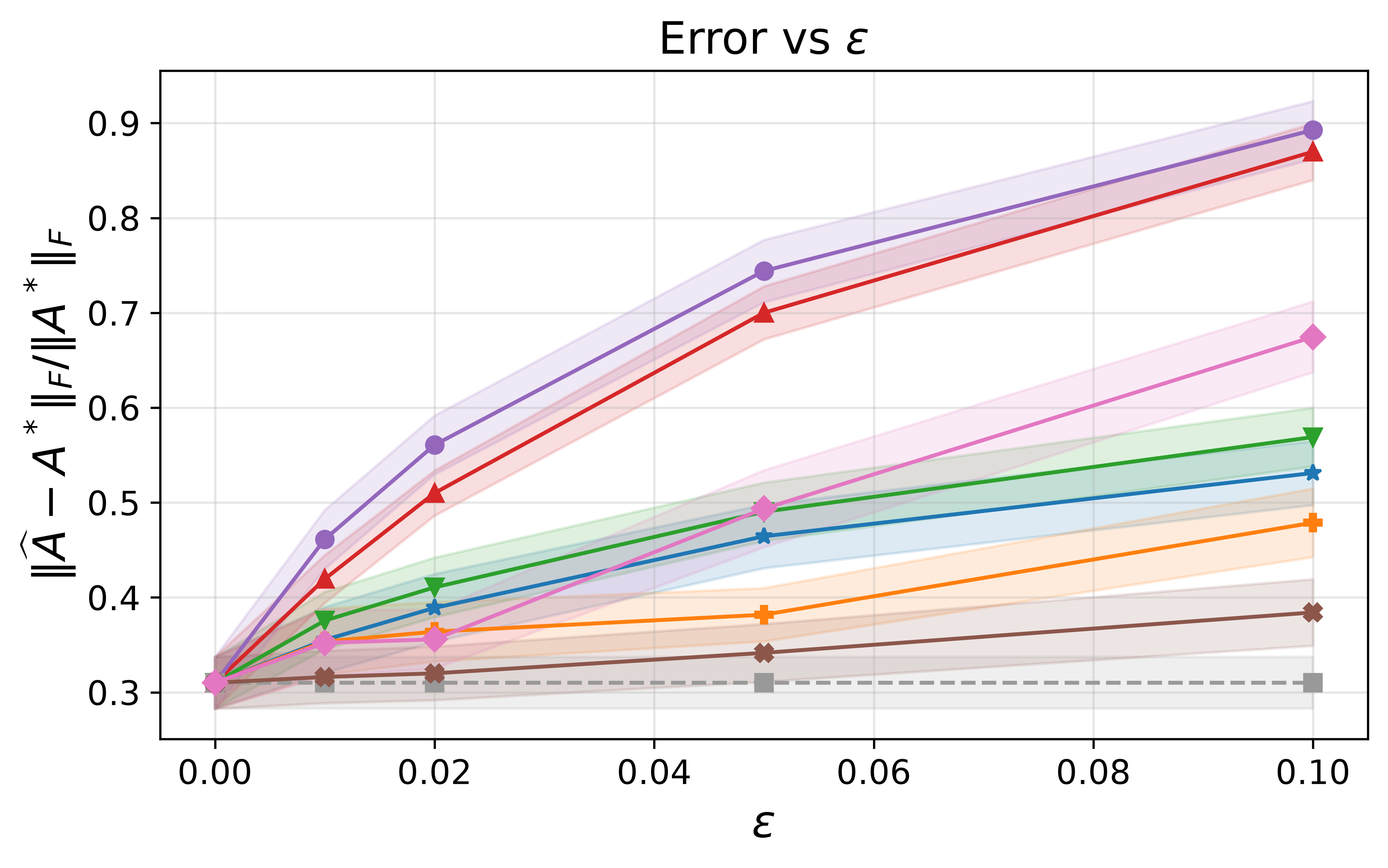}
        \caption{$T=100$}
        \label{fig:eps_leverage_targeted_100}
    \end{subfigure}
    \hfill
    \begin{subfigure}[b]{0.32\textwidth}
        \centering
        \includegraphics[width=\linewidth]{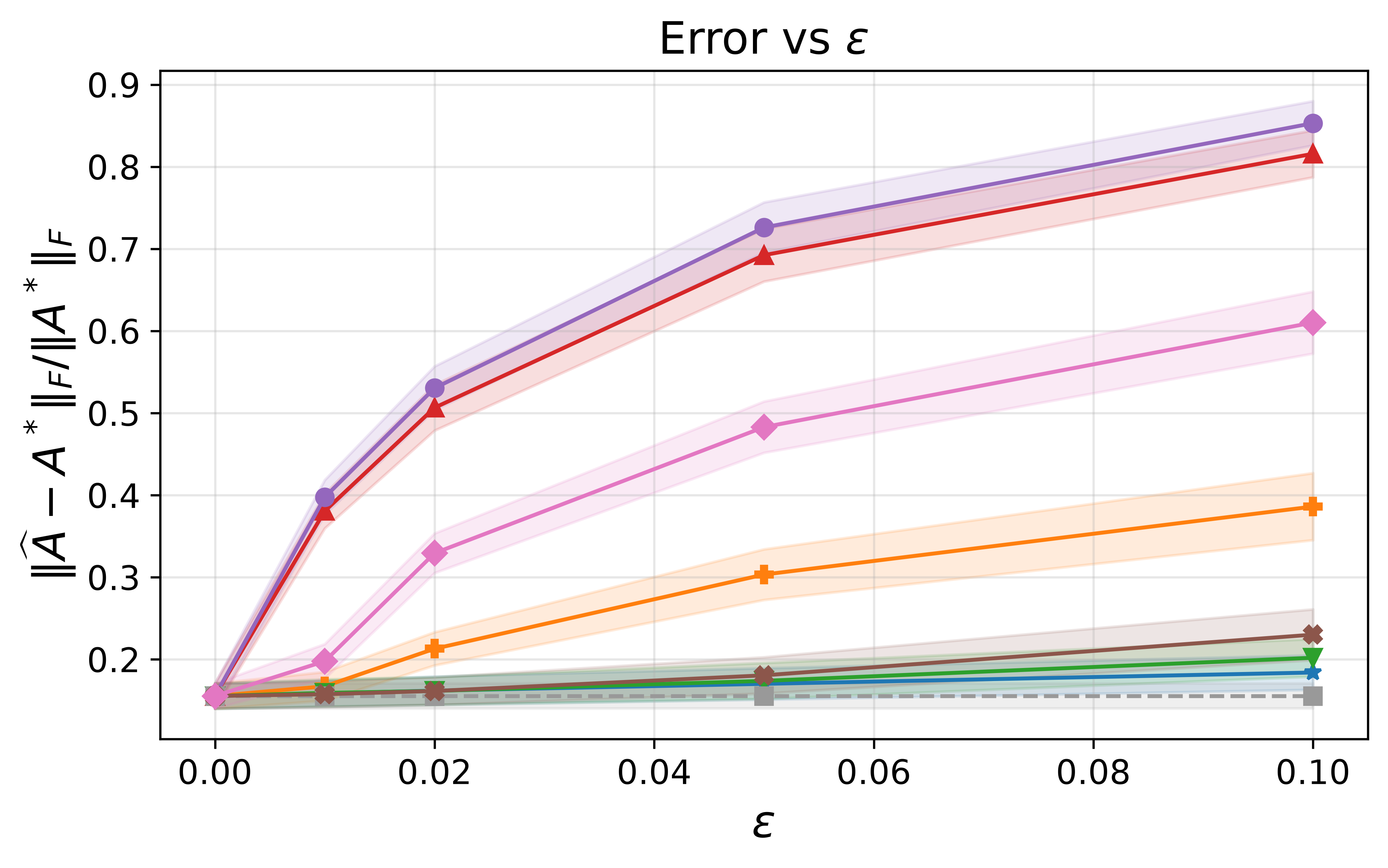}
        \caption{$T=400$}
        \label{fig:eps_leverage_targeted_400}
    \end{subfigure}
    \hfill
    \begin{subfigure}[b]{0.32\textwidth}
        \centering
        \includegraphics[width=\linewidth]{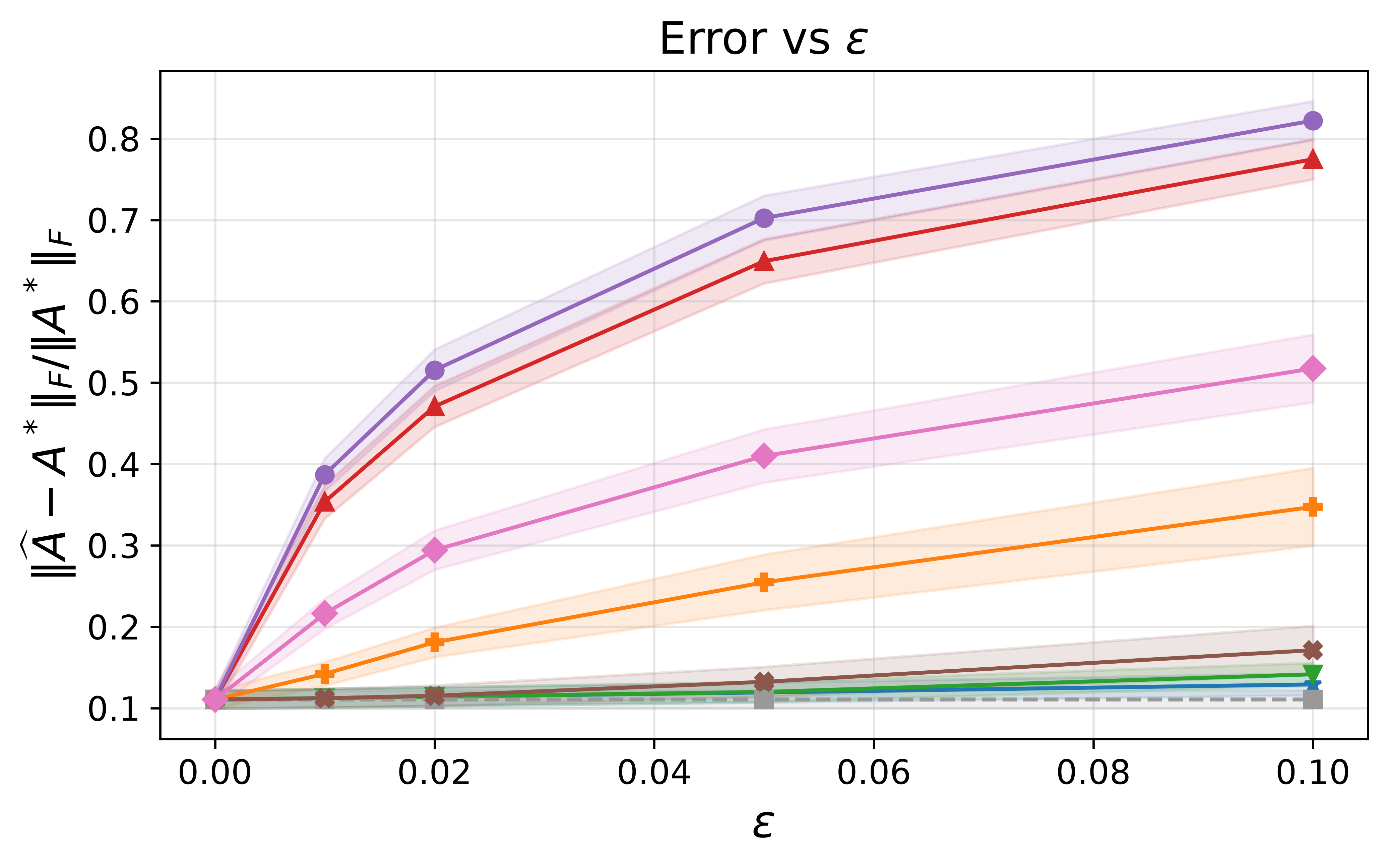}
        \caption{$T=800$}
        \label{fig:eps_leverage_targeted_800}
    \end{subfigure}
    
    \vspace{1em} 
    
    \begin{subfigure}[b]{0.6\textwidth}
        \centering
        \includegraphics[width=\linewidth]{exp_results_synth/legend}
        \label{fig:eps_leverage_targeted_legend}
    \end{subfigure}
    \makeatletter
    \renewcommand{\@makecaption}[2]{%
      \vspace{\abovecaptionskip}%
      \parbox{\textwidth}{{\bf #1:} #2}%
      \vspace{\belowcaptionskip}%
    }
    \makeatother
    \caption{
    Frobenius norm error, $\norm{\widehat{A}-A^*}_F/{\norm{A^*}_F}$, for $\epsilon = \{0,0.01,0.02,0.05,0.1\}$ with (a) $T=100$, (b) $T=400$, and (c) $T=800$, under the \emph{high-leverage contamination} model.}
    \label{fig:eps_leverage_targeted}
\end{figure}

\paragraph{Frobenius norm error against $\epsilon$.} Figure~\ref{fig:eps_leverage_targeted} shows the Frobenius error, ${\norm{\widehat{A}-A^*}_F}/{\norm{A^*}_F}$ against $\epsilon\in\set{0,0.01,0.02,0.05,0.1}$. We observe similar behavior across all $T \in \set{100,400,800}$, with \texttt{OLS} and \texttt{LS-BSP} showing the highest error. As expected, the error increases with $\epsilon$. Note that \texttt{Biconvex-AM},\texttt{SDP-AM}, and \texttt{LS-BSHC} still achieve strong recovery.

\end{document}